\pdfoutput=1

\documentclass[10pt]{article}

\usepackage{Sty/mcr}
\usepackage{Sty/pkg}
\usepackage{Sty/thmE}
\usepackage{subcaption}

\usepackage{algorithm} 
\usepackage{algorithmic}

\graphicspath{ {./img/} }

\usepackage[margin=1in]{Sty/geometry}
\usepackage{natbib}

\DeclareMathOperator*{\argmin}{arg\,min}

\title{More Powerful Conditional Selective Inference for \\ Generalized Lasso by Parametric Programming}

\date{\today}

\author{
    Vo Nguyen Le Duy \\
    Nagoya Institute of Technology and RIKEN\\
    duy.mllab.nit@gmail.com
  \and
    Ichiro Takeuchi\\
    Nagoya Institute of Technology and RIKEN\\
    takeuchi.ichiro@nitech.ac.jp
}

\begin{document}

\maketitle

\begin{abstract}
Conditional selective inference (SI) has been studied intensively as a new statistical inference framework for data-driven hypotheses.
The basic concept of conditional SI is to make the inference conditional on the selection event, which enables an exact and valid statistical inference to be conducted even when the hypothesis is selected based on the data.
Conditional SI has mainly been studied in the context of model selection, such as vanilla lasso or generalized lasso.
The main limitation of existing approaches is the low statistical power owing to over-conditioning, which is required for computational tractability.
In this study, we propose a more powerful and general conditional SI method for a class of problems that can be converted into quadratic parametric programming, which includes generalized lasso.
The key concept is to compute the continuum path of the optimal solution in the direction of the selected test statistic and to identify the subset of the data space that corresponds to the model selection event by following the solution path. 
The proposed parametric programming-based method not only avoids the aforementioned major drawback of over-conditioning, but also improves the performance and practicality of SI in various respects. 
We conducted several experiments to demonstrate the effectiveness and efficiency of our proposed method.
\end{abstract}

\clearpage

% --------------- Main Text --------------------

\section{Introduction}

As machine learning (ML) is applied to solve numerous practical problems, the quantification of the reliability of data-driven knowledge obtained by ML algorithms is becoming increasingly important.
Among the various potential approaches for reliable ML, \emph{conditional selective inference (SI)} has been recognized as a new promising method for assessing the statistical reliability of data-driven hypotheses that are selected by ML algorithms.
The main concept of conditional SI is to make the inference for a data-driven hypothesis \emph{conditional on the selection event} that the hypothesis is selected, which enables an \emph{exact} and a \emph{valid} inference to be conducted on the selected hypothesis.
In the conditional SI framework, the statistical significance and reliability of data-driven selected hypotheses are quantified by the so-called \emph{selective $p$-values} and \emph{selective confidence intervals}, which have proper false positive rates and coverage guarantees, respectively. 

\citet{lee2016exact} first introduced conditional SI as a statistical inference tool for the features selected by lasso \citep{tibshirani1996regression}.
Subsequently, \citet{hyun2018exact} studied SI for inference on the selected model using generalized lasso \citep{tibshirani2011solution}.
Their basic concept was to characterize the selection event using a polytope (a set of linear inequalities) in the sample space. 
In general, we refer to such methods as \emph{polytope-based SI} approaches.
The practical computational methods developed by these authors can be used when the selection event can be characterized by a single polytope.

However, the application scope of such polytope-based SI is limited because it can only be used when the characterization of all relevant selection events is represented by a polytope. 
Therefore, in most existing polytope-based SI studies, additional conditioning is required for the selection event to be characterized as a polytope.
For example, in the case of lasso \citep{lee2016exact}, the set of selected features as well as their signs require conditioning.
Similarly, in the case of generalized lasso \citep{hyun2018exact}, additional conditioning on the signs as well as on the history (sequential order) whereby the selected elements enter the selected model is required.
Such additional conditioning results in low statistical power, which is widely recognized as a major drawback of polytope-based SI studies \citep{fithian2014optimal}.

\paragraph{Contributions.} The contributions of this study are as follows:
\begin{itemize}
%$\bullet$ 
	\item We go beyond the scope of polytope-based SI and propose a new SI approach based on \emph{parametric programming (PP)}. 
We name the proposed method \emph{PP-based SI}.
The basic concept of PP-based SI is to compute the continuum path of the optimal solutions in the direction of the selected test statistic using PP, which is subsequently used to identify the exact sampling distribution of the test statistic with the minimum amount of conditioning.
Therefore, PP-based SI can fundamentally resolve the over-conditioning problem, which is a major concern in polytope-based SI, thereby achieving high statistical power.

%$\bullet$
	\item We derive the proposed PP-based SI for a generic class of conditional SI problems that can be represented as parametric quadratic programs (QPs). 
	We demonstrate that the conditional SI formulations for many practical problems, including generalized lasso, elastic net, non-negative least squares, and Huber regression with the $\ell_1$ penalty, belong to this class, which means that the proposed PP-based SI can be used extensively. 

	\item Furthermore, we discuss how the advantages of PP can be exploited for the effective performance of conditional SI in various data analysis tasks. 
	As an example, using PP, we demonstrate that conditional SI can be conducted with minimal conditioning for regularization parameter selection by cross-validation (CV),
the selection event of which is too complicated to be characterized as a single polytope, but can be fully characterized using PP\footnote{
     \citet{loftus2015selective}
     considered conditional SI for CV-based regularization parameter selection, but it was highly over-conditioned with additional events.
     We demonstrate that our PP-based SI is more powerful than this approach.
     }.

%$\bullet$ 
	\item We conducted intensive experiments on both synthetic and real-world datasets, by means of which we presented evidence that our proposed method can successfully control the false positive rate, has higher statistical power than polytope-based SI, and provides superior results in practical applications. 
%
%We provide our implementation in the supplementary document, which will be released when this paper is published.

\end{itemize}

A preliminary short version of this work was presented at the AI \& Statistics (AISTATS2021) conference \citep{le2021parametric}. In the conference version, we only studied a specific case of vanilla lasso. 
In this study, we extended the basic concept of PP-based SI to a more general class of problems that can be formulated as parametric QPs, which includes vanilla lasso as a special case.
Moreover, we extended the proposed method to various aspects and conducted intensive additional experiments to demonstrate the applicability of the generalized PP-based SI to a wider class of problems and settings.

\paragraph{Related works.}
In traditional statistical inference, it is assumed that the hypothesis is fixed in advance.
That is, the hypothesis on which we wish to conduct inference is determined prior to observing the dataset.
Therefore, if traditional statistical inference methods are applied to data-driven hypotheses, the inferential results will no longer be valid.
This problem has been discussed extensively in the context of testing the significance of the features selected by a feature selection method, such as lasso or stepwise feature selection.
Several approaches have been proposed to address this problem \citep{benjamini2005false, leeb2005model, leeb2006can, benjamini2009selective, potscher2010confidence, berk2013valid, lockhart2014significance, taylor2014post}.

In recent years, \citet{lee2016exact} proposed a practical SI framework to perform exact (non-asymptotic) inference for a set of features selected by lasso.
In their work, the authors revealed that the selection event can be characterized as a polytope by conditioning on a set of selected features and their signs.
Furthermore, \citet{hyun2018exact} demonstrated that polytope-based SI is applicable to generalized lasso by performing additional conditioning on the signs as well as the history (sequential order) whereby the selected elements entered the selected model.
Following the seminal work of \citep{lee2016exact}, conditional inference-based SI has been actively studied and applied to various problems \citep{fithian2015selective, choi2017selecting, tian2018selective, chen2019valid, hyun2018post, loftus2014significance, loftus2015selective, panigrahi2016bayesian, tibshirani2016exact, yang2016selective, suzumura2017selective, tanizaki2020computing, duy2020computing, duy2020quantifying, sugiyama2020more, tsukurimichi2021conditional}.

It is desirable to conduct more powerful inference by conditioning on as little information as possible in conditional SI \citep{fithian2014optimal}.
However, in polytope-based SI, an excessive amount of over-conditioning is required to represent the selection event using a single polytope.
The authors of \citet{lee2016exact} already mentioned the problem of over-conditioning and discussed the solution for removing the additional conditioning by enumerating all possible combinations of signs and taking the union over the resulting polyhedra.
Unfortunately, such an enumeration for an exponentially increasing number of sign combinations is only feasible when the number of selected features is small.
\citet{loftus2014significance} extended polytope-based SI to cases in which the selection event is characterized by quadratic inequalities.
Although we do not discuss quadratic inequality-based SI further, as it suffers from a similar over-conditioning issue, our proposed PP-based SI can also be applied to resolve the issue for this class of conditional SIs.

Our work was motivated by 
\citet{liu2018more}, in which the authors proposed solutions to overcome the over-conditioning issue of polytope-based SI for vanilla lasso in certain special settings. 
In one of the settings, inference on the full model parameters in which conditional SI can be performed with minimal conditioning was studied, because a full model parameter is not dependent on other parameters.
Moreover, our work was motivated by a discussion in the paper where the authors noted that multiple lasso fitting at a sequence of grid points may aid in alleviating over-conditioning. 
However, the authors did not suggest any practical computational methods to realize this concept. 
Furthermore, the conditional sampling distribution that is evaluated at a finite number of grid points only provides an approximation of the exact distribution, which means that the theoretical validity of the conditional SI is no longer guaranteed. 
Our proposed PP-based SI can be interpreted as a means of solving lasso at \emph{infinitely} many grid points, which completely resolves these challenging problems. 
As another direction to resolve over-conditioning, \citet{tian2018selective} and \citet{terada2019selective} proposed methods using randomization. 
A drawback of these randomization-based approaches (including the simple data-splitting approach) is that further randomness is added in both the feature selection and inference stages.

PP has long been studied in the optimization field to solve a family of optimization problems that are parameterized by a scalar parameter~\citep{Ritter84, Allgower93, Gal95, Best96}. Moreover, PP has been used in the context of the \emph{regularization path} in ML~\citep{osborne2000new, Efron04a, HasRosTibZhu04, RosZhu07, BacHecHor06, RosZhu07, Tsuda07, Lee07, Takeuchi09a, takeuchi2011target, Karasuyama11, hocking11a, Karasuyama12a, ogawa2013infinitesimal, takeuchi2013parametric}. 
The regularization path is a method of tracking the manner in which the optimal solution changes when the regularization parameter changes, which is useful for efficient model selection. Our main idea was to employ PP to track how the optimal solution and selected features change when the training dataset changes in the direction of the selected test statistic, which enables the exact sampling distribution of the test statistic that is conditional on the selection event to be identified. The power of the conditional SI introduced in the pioneering work of \citet{lee2016exact} can be optimized using the proposed approach, which is applicable to conditional SI for a wide class of problems including generalized lasso.
\section{Problem Statement} \label{sec:problem_statement}

To formulate the problem, we consider a random response vector 
\begin{equation} \label{eq:random_vector}
{\bm Y} = (Y_1, ..., Y_n)^\top \sim \NN({\bm \mu}, \Sigma),
\end{equation}
where $n$ is the number of instances, ${\bm \mu}$ is an unknown vector, and $\Sigma \in \RR^{n \times n}$ is a covariance matrix that is known or estimable from independent data.
The goal is \emph{statistical} quantification of the significance of the data-driven hypotheses that are obtained by applying the generalized lasso estimator to the response vector. 

\paragraph{Generalized lasso and its selection event.} 
We consider a linear regression model with $p$ features $\bm x_1, \ldots, \bm x_p \in \RR^n$ and denote the feature matrix as $X \in \RR^{n \times p}$, in which the features are considered as non-random.
We do not make any assumption about the relationship between the $\bm \mu$ and $p$ features $\bm x_1, \ldots, \bm x_p \in \RR^n$, but consider the case in which a linear model with generalized lasso regularization is employed to model the relationship between $X$ and a random response vector $\bm Y$.
Given an observed response 
vector ${\bm y}^{\rm obs} \in \RR^n$ that is sampled from model (\ref{eq:random_vector}), the generalized lasso optimization problem is expressed as
\begin{equation} \label{eq:generalized_lasso}
	\hat{{\bm \beta}} = \argmin \limits_{{\bm \beta} \in \RR^p} \frac{1}{2} \|{\bm y}^{\rm obs} - X {\bm \beta}\|^2_2 + \lambda \|D {\bm \beta}\|_1,
\end{equation}
where 
%$X \in \RR^{n \times p}$ is a feature matrix, 
$D \in \RR^{m \times p}$ is a penalty matrix and $\lambda \geq 0$ is a regularization parameter. 
The matrix $D$ and its number of rows are predetermined by the user to produce the desired structures in the solution $\hat{\bm \beta}$ in \eq{eq:generalized_lasso}.
Examples of matrix $D$ are presented in Examples \ref{example_1}, \ref{example_2}, and \ref{example_3}.

As the optimization in \eq{eq:generalized_lasso} produces the sparsity of $D \hat{\bm \beta}$, we define a set of non-zero components (the active set) as follows: 
\begin{align*}
	\cM_{\rm obs} = \cA({\bm y}^{\rm obs}) = \{ j : (D \hat{\bm \beta})_j \neq 0\}, \quad j \in [m],
\end{align*}
where $\cA: \bm Y \mapsto \cM$ indicates the algorithm that maps a response vector $\bm Y$ to a set of non-zero components $\cM$. 
Thereafter, we define the selection event in which the active set for a random response vector $\bm Y$ is the same as the observed response vector $\bm y^{\rm obs}$:
\begin{align} \label{eq:selection_event}
	\left \{ \cA(\bm Y) = \cA({\bm y}^{\rm obs}) \right \}.
\end{align}

\paragraph{Statistical inference.}
Let $\bm \eta^\top_j \bm Y$ be a linear contrast that indicates the test statistic that we wish to consider, where $\bm \eta_j$ is defined depending on the problem and the $j^{\rm th}$ selected component in the observed active set.

\begin{example}\label{example_1}
{
\normalfont
In the case of testing the features selected by vanilla lasso \citep{tibshirani1996regression}, $D = I_p \in \RR^{p \times p}$, which is the identity matrix.
The test statistic 
$\bm \eta^\top_j \bm Y = \hat{\beta}_j$ represents the coefficient of the $j^{\rm th}$ selected feature \citep{lee2016exact}, where $\bm \eta_j$ is defined as
\begin{align} \label{eq:eta_lasso}
\bm{\eta}_j = X_{\cM_{\rm obs}} \left( X^\top_{\cM_{\rm obs}} X_{\cM_{\rm obs}}\right)^{-1} \bm{e}_j,
\end{align}
%where $\bm{e}_j \in \RR^{|\cM_{\rm obs}|}$ is a basis vector with a $1$ at position $j^{\rm th}$.
in which $\bm{e}_j \in \RR^{|\cM_{\rm obs}|}$ is a basis vector with $1$ at the $j^{\rm th}$ position.
This form of test statistic can also be applied to test the features that are selected by other regression methods, such as elastic net \citep{zou2005regularization}, Huber regression \citep{huber1992robust}, or non-negative least squares.
}
\end{example}

\begin{example}\label{example_2}
{\normalfont
In the context of \emph{changepoint (CP)} detection using fused lasso, the matrix $D \in \RR^{(p - 1) \times p}$ is expressed as 
\begin{align*}
	D = 
	\begin{pmatrix}
		-1 & 1 & 0 & \cdots & 0 & 0 \\ 
		0 & -1 & 1 & \cdots & 0 & 0 \\ 
		& & & \cdots & & \\ 
		0 & 0 & 0 & \cdots & -1 & 1 
	\end{pmatrix}.
\end{align*}
The test statistic $\bm \eta^\top_j \bm Y$ represents the difference in the sample mean between the left and right segments of the CP at the $j^{\rm th}$ position, which was also used in \citet{hyun2018exact}.
In this case, $\bm \eta_j$ is defined as
\begin{align}\label{eq:eta_cp}
\bm \eta_j = \frac{1}{j - j_{\rm prev}} \one^n_{j_{\rm prev}+1:j} - \frac{1}{j_{\rm next} - j} \one^n_{j+1:j_{\rm next}},
\end{align}
where $j_{\rm prev} \in \cM_{\rm obs}$ and $j_{\rm next} \in \cM_{\rm obs}$ are the CP positions before and after the selected CP at position $j$, respectively, and $\one^n_{s:e} \in \RR^n$ is a vector in which the elements from positions $s$ to $e$ are set to 1, and 0 otherwise.
}
\end{example}

\begin{example}\label{example_3}
{
\normalfont
In trend filtering, the aim is to test whether a change occurs in the trend at position $j \in \cM_{\rm obs}$. 
We define $\bm \eta_j$ as follows:
\begin{align*}
	\bm \eta_j = \bm e_{j - 1} - 2 \bm e_j + \bm e_{j + 1}, 
\end{align*}
where $\bm{e}_j \in \RR^n$.
The test statistic $\bm \eta_j^\top \bm Y$ indicates that we wish to test whether the points at positions $j$, $j-1$, and $j + 1$ lie on the same line statistically.
In this case, the matrix $D \in \RR^{(p - 2) \times p}$ is expressed as 
\begin{align*}
	D = 
	\begin{pmatrix}
		-1 & 2 & -1 & \cdots & 0 & 0 & 0\\ 
		0 & -1 & 2 & \cdots & 0 & 0 & 0\\ 
		& & & \cdots & & & \\ 
		0 & 0 & 0 & \cdots & -1 & 2 & -1 
	\end{pmatrix}.
\end{align*}
}
\end{example}

For the inference, we consider the following \emph{null hypothesis} and \emph{alternative hypothesis}:
\begin{align}\label{eq:hypotheses}
 {\rm H}_{0, j}: \bm \eta_j^\top \bm \mu  = 0 \quad \text{vs.} \quad {\rm H}_{1, j}: \bm \eta_j^\top \bm \mu \neq 0.
\end{align}
%
%If the $p$-value is smaller than a given significance level $\alpha$ (e.g., 0.05), we can reject the null hypothesis and conclude that the result is statistical significance.

\paragraph{Conditional SI.}
Suppose that the hypotheses in \eq{eq:hypotheses} are fixed; that is, non-random. 
Thus, the \emph{naive} (two-sided) $p$-value 
in the classical $z$-test 
is obtained by 
\begin{align}
\label{eq:naive_p}
 \hspace{-2mm} P^{\rm naive}_j
 =
 \PP_{{\rm H}_{0, j}}
 \left(
 |\bm \eta_j^\top \bm Y| \ge |\bm \eta_j^\top \bm y^{\rm obs}|
 \right).
\end{align}
However,
as the hypotheses in \eq{eq:hypotheses} are not fixed in advance,
the naive $p$-value is not \emph{valid} 
in the sense that,
if we reject
${\rm H}_{0, j}$
with a
significance level
$\alpha$
(e.g., $\alpha=0.05$),
the false positive rate (type-I error) cannot be controlled at the level $\alpha$.
This is because the hypotheses in \eq{eq:hypotheses} 
are \emph{selected} by the data
and \emph{selection bias} exists.

It is necessary to remove the information that has been used for the initial hypothesis generation process to correct the selection bias. 
This can be achieved by considering the sampling distribution of the test statistic $\bm{\eta}^\top_j \bm Y$ that is conditional on the selection event; that is,
\begin{equation}\label{eq:condition_model}
	\bm{\eta}^\top_j \bm Y \mid \left \{ \cA(\bm Y) = \cA({\bm y}^{\rm obs}), {\bm q}({\bm Y}) = {\bm q} ({\bm y}^{\rm obs})\right \},
\end{equation}
where ${\bm q}({\bm Y}) = (I_n - {\bm c} {\bm \eta}^\top_j) {\bm Y}$ with $\bm c = \Sigma {\bm \eta}_j ({\bm \eta}^\top_j \Sigma {\bm \eta}_j)^{-1}$ is the nuisance component that is independent of the test statistic $\bm \eta_j^\top \bm Y$.
The second condition ${\bm q}({\bm Y}) = {\bm q} ({\bm y}^{\rm obs})$ 
indicates that the nuisance component for a random vector $\bm Y$ is the same as that for $\bm y^{\rm obs}$
\footnote{
${\bm q}(\bm Y)$ corresponds to the component $\bm z$ in the seminal paper (see \cite{lee2016exact}, Section 5, Eq. 5.2 and Theorem 5.2).}.

Once the selection event has been identified, the pivotal quantity can be computed:
\begin{align}\label{eq:pivotal_quantity}
	F^{\cZ}_{\bm{\eta}^\top_j \bm \mu, {\bm \eta}^\top_j \Sigma {\bm \eta}_j} (\bm{\eta}^\top_j \bm Y). 
	\mid \left \{ \cA(\bm Y) = \cA({\bm y}^{\rm obs}), {\bm q}({\bm Y}) = {\bm q} ({\bm y}^{\rm obs}) \right \},
\end{align}
which is the c.d.f. of the truncated normal distribution with a mean $\bm{\eta}^\top_j \bm \mu$, variance ${\bm \eta}^\top_j \Sigma {\bm \eta}_j$, and truncation region $\cZ$, which is calculated based on the selection event.
%
%The pivotal quantity is crucial for calculating selective $p$-value and confidence interval.
%
Based on the pivotal quantity, the \emph{selective type-I error} or \emph{selective $p$-value} \citep{fithian2014optimal} can be considered in the following form: 
 \begin{align} \label{eq:selective_p_value}
  P^{\rm selective}_j = 2\ \min\{\pi_j, 1 - \pi_j\},
\end{align}
where $\pi_j = 1 - F^{\cZ}_{0, {\bm \eta}^\top_j \Sigma {\bm \eta}_j} (\bm{\eta}^\top_j {\bm Y})$,
which is \emph{valid} in the sense that 
\begin{align*}
	{\rm Prob}_{{\rm H}_{0, j}} \left(P^{\rm selective}_j < \alpha \right) = \alpha, \forall \alpha \in [0, 1].
\end{align*}
Furthermore,
to obtain a confidence level of $1 - \alpha$ for any $\alpha \in [0, 1]$, by inverting the pivotal quantity in Equation (\ref{eq:pivotal_quantity}), we can determine the smallest and largest values of $\bm{\eta}^\top_j \bm \mu$ such that the value of the pivotal quantity remains in the interval $\left[ \frac{\alpha}{2}, 1- \frac{\alpha}{2} \right]$ \citep{lee2016exact}.

\paragraph{Challenge in conditional data space characterization.}
The main difficulty in the above conditional SI is that the characterization of the minimal conditional data space 
\[\left \{ \cA(\bm Y) = \cA({\bm y}^{\rm obs}), {\bm q}({\bm Y}) = {\bm q} ({\bm y}^{\rm obs})\right \}\] in Equation (\ref{eq:condition_model}) is intractable. 
To overcome this issue, \citet{hyun2018exact} considered the inference to be conditional not only on the active set, but also on the signs and history (order) whereby the elements of $D \hat{\bm \beta} $ entered the active set.
Unfortunately, such additional conditioning on the signs and history leads to low statistical power owing to over-conditioning \footnote{This over-conditioning corresponds to the additional conditioning on the signs of the selected features in the seminal conditional SI study of vanilla lasso \citep{lee2016exact}.}.

In the following section, we introduce a method for identifying the minimum amount of conditioning 
$\left \{ \cA(\bm Y) = \cA({\bm y}^{\rm obs}), {\bm q}({\bm Y}) = {\bm q} ({\bm y}^{\rm obs})\right \}$ that results in higher statistical power.
The main concept is to compute the path of the generalized lasso solutions in the direction of interest $\bm \eta_j$.
By focusing on the line along $\bm \eta_j$, the majority of irrelevant regions that do not affect the truncated normal sampling distribution can be skipped because they do not intersect with this line.
Thus, we can skip the majority of combinations of signs and history that never appear when applying generalized lasso to the data on the line.
\section{Proposed Method} \label{sec:proposed_method}

\begin{figure*}[!t]
\centering
\includegraphics[width=.85\linewidth]{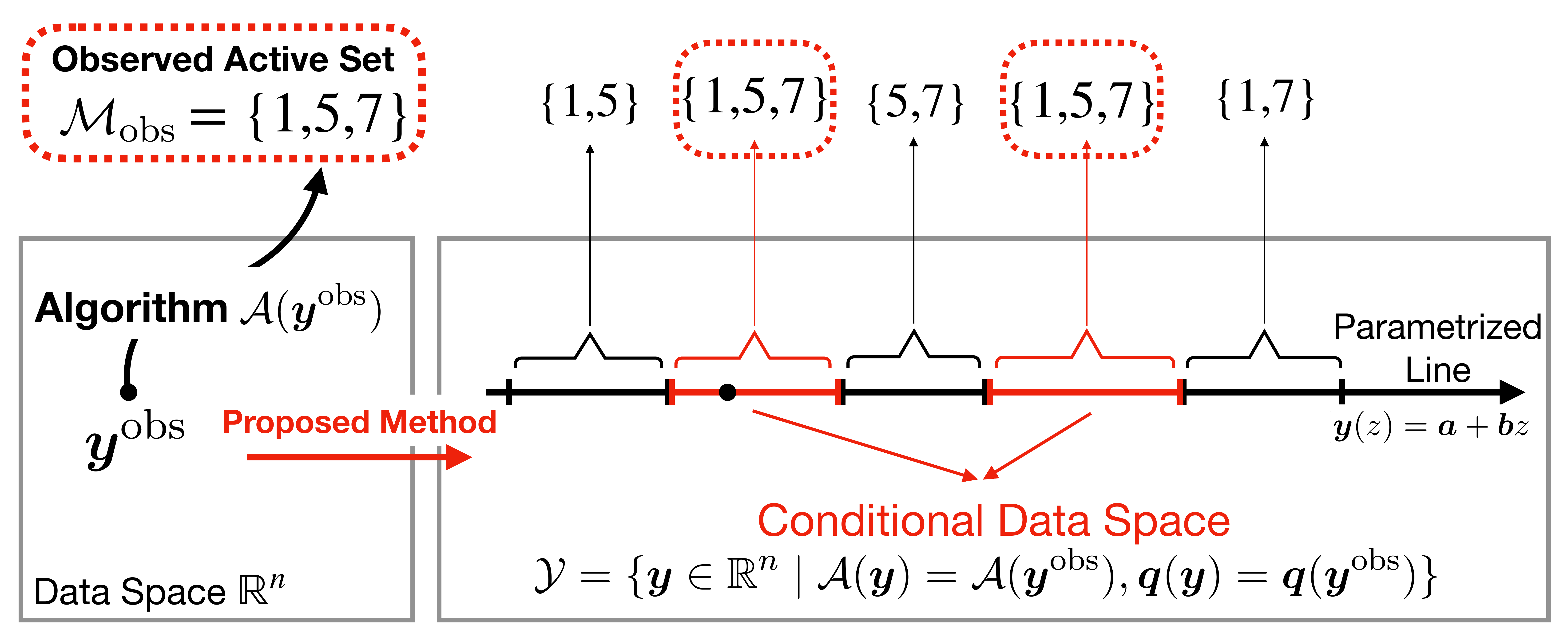}  
\caption{
Schematic of proposed method.
We obtain the observed active set $\cM_{\rm obs}$ by applying generalized lasso to the observed data $\bm y^{\rm obs}$.
The statistical inference for each selected element in the observed active set is conducted conditional on the subspace $\cY$, the data of which have the same active set as $\bm y^{\rm obs}$.
We introduce a PP method for characterizing the conditional data space $\cY$ by searching the parameterized line.
}
\label{fig:fig_illustration}
%\vspace{-10pt}
\end{figure*}

We present the technical details of the proposed method in this section.
A schematic of the method is provided in Figure \ref{fig:fig_illustration}. 
We first introduce a QP reformulation for the generalized lasso problem in \S \ref{subsec:qp_genlasso}.
Thereafter, the characterization of the conditional data space is presented in \S \ref{subsec:conditional_data_space_characterization}.
Subsequently, we propose a PP approach for identifying the conditional data space in \S \ref{subsec:piecewise}.
Finally, the detailed algorithm is presented in \S \ref{subsec:algorithm}.

% === Quadratic Programming for Generalized Lasso ===

\subsection{QP for Generalized Lasso} \label{subsec:qp_genlasso}

We demonstrate that the generalized lasso problem can be reformulated as a QP problem.

\begin{lemm} \label{lemm:qp_for_fused_lasso}
We denote $\bm \xi = D \bm \beta$, and the generalized lasso in \eq{eq:generalized_lasso} can be rewritten as
\begin{align}\label{eq:fused_lasso_trans}
	\left(\hat{\bm \beta}, \hat{\bm \xi} \right) = \argmin \limits_{\bm \beta \in \RR^p, \bm \xi \in \RR^m} \frac{1}{2} ||\bm y - X \bm \beta||_2^2 + \lambda ||\bm \xi||_1 \quad \text{subject to} \quad \bm \xi = D \bm \beta.
\end{align}
By decomposing $\bm \xi = \bm \xi^{+} - \bm \xi^{-}$, $\bm \xi^{+}, \bm \xi^{-}  \geq \bm 0$, the generalized lasso problem can be formulated as the following QP problem:
\begin{align}\label{eq:qp_fused_lasso}
\begin{aligned}
	\left(\hat{\bm \beta}, \hat{\bm \xi}^{+}, \hat{\bm \xi}^{-}\right) = \argmin \limits_{\bm \beta, \bm \xi^{+}, \bm \xi^{-}} & ~ \frac{1}{2}   
	\begin{pmatrix}
		\bm \beta \\  \bm \xi^{+} \\  \bm \xi^{-}
	\end{pmatrix}^\top
	\begin{pmatrix}
		X^\top X &  0 & 0 \\ 
		0 &  0 & 0 \\ 
		0 &  0 & 0
	\end{pmatrix}
	\begin{pmatrix}
		\bm \beta \\  \bm \xi^{+} \\  \bm \xi^{-}
	\end{pmatrix}
	+ 
	\left ( \lambda
	\begin{pmatrix}
		\bm 0_p  \\  {\bm 1}_m \\  {\bm 1}_m
	\end{pmatrix}
		-
	\begin{pmatrix}
		X^\top \bm y \\  \bm 0_m \\ \bm 0_m
	\end{pmatrix}
	\right )
	^\top
	\begin{pmatrix}
		\bm \beta \\  \bm \xi^{+} \\  \bm \xi^{-}
	\end{pmatrix} \\ 
	\textrm{s.t} ~  
	&
	\begin{pmatrix}
		0 & I_m & - I_m
	\end{pmatrix}
	\begin{pmatrix}
		\bm \beta \\  \bm \xi^{+} \\  \bm \xi^{-}
	\end{pmatrix} 
	=
	\begin{pmatrix}
		D & 0 & 0
	\end{pmatrix}
	\begin{pmatrix}
		\bm \beta \\  \bm \xi^{+} \\  \bm \xi^{-}
	\end{pmatrix},
	~\bm \xi^{+} \geq 0, ~\bm \xi^{-} \geq 0,
\end{aligned}
\end{align}
where $\bm 1_m \in \RR^m$ and $\bm 0_m \in \RR^m$ are vectors in which all elements are set to 1 and 0, respectively, and $I_m \in \RR^{m \times m}$ is an identity matrix.
\end{lemm}

\begin{proof}
The optimization problem in \eq{eq:fused_lasso_trans} can be rewritten as follows:
\begin{align}
%	& \hat{\bm \beta}, \hat{\bm \xi} = \argmin \limits_{\bm \beta, \bm \xi} \frac{1}{2} ||\bm y - X \bm \beta||_2^2 + \lambda ||\bm \xi||_1 \quad \textrm{s.t} \quad \bm \xi = D \bm \beta \nonumber \\ 
%	\Leftrightarrow ~ & \hat{\bm \beta}, \hat{\bm \xi} = \argmin \limits_{\bm \beta, \bm \xi} \frac{1}{2} 
%	\left ( \bm y^\top \bm y  - \bm \beta^\top X^\top \bm y - \bm y^\top X \bm \beta  + \bm \beta^\top X^\top X \bm \beta  \right ) + \lambda ||\bm \xi||_1 \quad \textrm{s.t} \quad \bm \xi = D \bm \beta  \label{eq:eq_a} \\
	\left(\hat{\bm \beta}, \hat{\bm \xi} \right)= \argmin \limits_{\bm \beta, \bm \xi} \frac{1}{2}  \bm \beta^\top X^\top X \bm \beta
	  - (X^\top \bm y)^\top \bm \beta  + \lambda ||\bm \xi||_1 \quad \textrm{s.t} \quad \bm \xi = D \bm \beta \label{eq:eq_b}.
\end{align}
We obtain \eq{eq:qp_fused_lasso} by decomposing $\bm \xi = \bm \xi^{+} - \bm \xi^{-}$ with $\bm \xi^{+}, \bm \xi^{-} \geq \bm 0$.
In this case, the key point is that the component $||\bm \xi||_1$ in \eq{eq:eq_b} can be written as $||\bm \xi||_1 = \sum_{j \in [m]} (\xi^{+}_j + \xi^{-}_j)$ because at least one of $\xi^{+}_j$ and $\xi^{-}_j$, $j \in [m]$, must be set to zero in the optimal solution.
That is, $\hat{\xi}^{+}_j > 0 \Rightarrow \hat{\xi}^{-}_j = 0$, and vice versa. 
The concept of decomposing $\bm \xi = \bm \xi^{+} - \bm \xi^{-}$ is often employed to reformulate the $\ell_1$-norm of $\bm \xi$.
\end{proof}
%

% === Conditional Data Space Characterization ===

\subsection{Conditional Data Space Characterization} \label{subsec:conditional_data_space_characterization}
We define the set of $\bm y \in \RR^n$ that satisfies the conditions in Equation (\ref{eq:condition_model}): 
\begin{equation} \label{eq:conditional_data_space}
	\hspace{-0.05mm} \cY = \{ {\bm y} \in \RR^{n} \mid \cA({\bm y}) = \cA({\bm y}^{\rm obs}), {\bm q} (\bm y) = {\bm q} ({\bm y}^{\rm obs})\}.
\end{equation}
According to the second condition, the data in $\cY$ are restricted to a line (see Section 6 in \cite{liu2018more} and \cite{fithian2014optimal}).
Therefore, the set $\cY$ can be rewritten using the scalar parameter $z \in \RR$, as follows:
\begin{equation} \label{eq:parametrized_data_space}
	\cY = \left \{ {\bm y}(z) = {\bm a} + {\bm b} z \mid z \in \cZ \right \},
\end{equation}
where 
${\bm a} = {\bm q}(\bm y^{\rm obs})$, 
${\bm b} = \Sigma {\bm \eta}_j ({\bm \eta}^\top_j \Sigma {\bm \eta}_j)^{-1} $,
and 
\begin{equation} \label{eq:truncation_region_z}
	\cZ = \left \{ z \in \RR \mid \cA({\bm y}(z)) = \cA({\bm y}^{\rm obs}) \right \}.
\end{equation}
Next, we consider a random variable $Z \in \RR$ and its observation $z^{\rm obs} \in \RR$, which satisfies ${\bm Y} = {\bm a} + {\bm b} Z$ and ${\bm y}^{\rm obs} =  {\bm a} + {\bm b} z^{\rm obs}$. 
The conditional inference in (\ref{eq:condition_model}) is rewritten as the problem of characterizing the sampling distribution of 
\begin{equation} \label{eq:condition_parametric}
	Z \mid \left \{ Z \in \cZ \right \}.
\end{equation}
As $Z \sim \NN(0, {\bm \eta}^\top_j \Sigma {\bm \eta}_j)$ under the null hypothesis, $Z \mid Z \in \cZ$ follows a truncated normal distribution.
Once the truncation region $\cZ$ has been identified, the pivotal quantity in Equation (\ref{eq:pivotal_quantity}) is equal to $F^{\cZ}_{0, {\bm \eta}^\top_j \Sigma {\bm \eta}_j} (Z)$, and it can be obtained easily.
Thus, the remaining task is the characterization of $\cZ$.
\paragraph{Characterization of truncation region $\cZ$.}
We introduce the optimization problem \eq{eq:qp_fused_lasso} with the parameterized response vectors ${\bm y}(z) = \bm a + \bm b z$ ($\bm a$, $\bm b$ that are defined in \eq{eq:parametrized_data_space}) for $z \in \RR$ as follows:
\begin{align} \label{eq:qp}
\begin{aligned}
\hat{\bm r}(z) = \argmin_{\bm r} \quad & \frac{1}{2} \bm r^\top P  \bm r + (\bm q^0 + \bm q^1 z)^\top \bm r\\
\textrm{s.t.} \quad & G \bm r \leq \bm h^0 + \bm h^1 z,
\end{aligned}
\end{align}
where 
$\bm r = \left (\bm \beta, \bm \xi^{+}, \bm \xi^{-} \right )^\top \in \RR^{p + 2m}$, 
$\bm q^0 = \left ( -X^\top \bm a, \lambda \bm 1_m,  \lambda \bm 1_m \right )^\top \in \RR^{p + 2m}$, 
$\bm q^1 = \left ( -X^\top \bm b, \bm 0, \bm 0 \right ) \in \RR^{p + 2m} $, $P \in \RR^{(p + 2m) \times (p + 2m)}$ and $G \in \RR^{4m \times (p + 2m) }$,
$\bm h^0 = \bm h^1 = \bm 0_{4m}$,
\begin{align*}
	P = 
	\begin{pmatrix}
		X^\top X &  0 & 0 \\ 
		0 &  0 & 0 \\ 
		0 &  0 & 0
	\end{pmatrix} ,
	\quad 
	G = 
	\begin{pmatrix}
		-D & D  & 0 & 0 \\ 
		 I_m & -I_m  & -I_m & 0 \\ 
		 -I_m & I_m & 0 & -I_m \\ 
	\end{pmatrix} ^\top.
\end{align*}

Let $\hat{\bm u}(z)$ be the vector of optimal Lagrange multipliers and ${\rm row} (G)$ be the number of rows in matrix $G$. The KKT conditions of \eq{eq:qp} are written as 
\begin{equation}\label{eq:kkt_1}
\begin{aligned}
	P \hat{\bm r}(z) + \bm q^0 + \bm q^1 z + G^\top \hat{\bm u}(z) &= 0, \\ 
	G \hat{\bm r}(z) - \bm h^0 - \bm h^1 z & \leq 0, \\ 
	\hat{u}_i(z) (G \hat{\bm r}(z) - \bm h^0 - \bm h^1 z)_i &= 0, \quad \forall i \in [{\rm row}(G)],  \\  
	\hat{u}_i(z) & \geq 0 , \quad \forall i \in [{\rm row}(G)].
\end{aligned}
\end{equation}

To construct the truncation region $\cZ$ in Equation (\ref{eq:truncation_region_z}), we must 1) compute the entire path of $\hat{\bm \xi}(z) = \hat{\bm \xi}^{+}(z) - \hat{\bm \xi}^{-}(z)$ in \eq{eq:qp}, 
and 2) identify the set of intervals of $z$ on which $\cA({\bm y}(z)) = \cA({\bm y}^{\rm obs})$.
However, it is difficult to compute $\hat{\bm \xi}^{+}(z)$ and $\hat{\bm \xi}^{-}(z)$ for infinitely many values of $z \in \RR$.
We demonstrate that the paths of $\hat{\bm \xi}^{+}(z)$ and $\hat{\bm \xi}^{-}(z)$ can be computed within \emph{finite} operations by introducing parametric quadratic programming.

% === A Piecewise Linear Homotopy ===

\subsection{Piecewise Linear Homotopy} \label{subsec:piecewise}
In this section, we demonstrate that $\hat{\bm r}(z)$ in \eq{eq:qp} is a piecewise linear function of $z$, which also indicates that $\hat{\bm \xi}^{+}(z)$ and $\hat{\bm \xi}^{-}(z)$ are piecewise linear functions of $z$.

\begin{lemm} \label{lemm:piecewise_linear}
We denote $\cI_z = \{ i \in [{\rm row} (G)] : \hat{u}_i(z) > 0\} $, $\cI_z^c = [{\rm row} (G)] \setminus \cI_z$, and $G_{\cI_z}$ as the rows of matrix $G$ in a set $\cI_z$.
Consider two real values $z$ and $z^\prime$ ($z < z^\prime$). 
If $\cI_z = \cI_{z^\prime}$, we obtain 
\begin{align}
	\hat{\bm r}(z^\prime) -  \hat{\bm r}(z) &= \bm \psi (z) \times (z^\prime - z), \label{eq:psi}\\ 
	\hat{\bm u}_{\cI_z}(z^\prime) - \hat{\bm u}_{\cI_z} (z) &= \bm \gamma (z) \times (z^\prime - z) \label{eq:gamma},
\end{align}
where 
$\bm \psi (z) \in \RR^{{\rm row}(P)}, \bm \gamma (z) \in \RR^{|\cI_z|}$, 
$
\begin{bmatrix}
	\bm \psi (z)\\
	\bm \gamma (z)
\end{bmatrix}
=  
\begin{bmatrix}
	P & G_{\cI_z}^\top \\
	G_{\cI_z} & 0 \\ 
\end{bmatrix}^{-1}
\begin{bmatrix}
	- \bm q^1 \\
	\bm h^1_{\cI_z}
\end{bmatrix}
$.
\end{lemm}

\begin{proof}
From the KKT conditions in \eq{eq:kkt_1}, we obtain 
\begin{equation}\label{eq:kkt_2}
\begin{aligned}
	P \hat{\bm r} (z) + \bm q^0 + \bm q^1 z + G^\top \hat{\bm u} (z) &= 0, \\ 
	(G \hat{\bm r}(z) - \bm h^0 - \bm h^1 z)_i  &= 0, \quad \forall i \in \cI_z, \\ 
	(G \hat{\bm r}(z) - \bm h^0 - \bm h^1 z)_i  &\leq 0, \quad \forall i \in \cI_z^c.
\end{aligned}
\end{equation}
According to \eq{eq:kkt_2}, we obtain the following linear system: 
\begin{align}
	&\begin{bmatrix}
		P & G_{\cI_z}^\top \\
		G_{\cI_z} & 0
	\end{bmatrix}
	\begin{bmatrix}
		\hat{\bm r}(z) \\
		\hat{\bm u}_{\cI_z}(z) 
	\end{bmatrix}
	= 
	\begin{bmatrix}
		- \bm q^0 \\
		\bm h^0_{\cI_z} 
	\end{bmatrix}
	+ 
	\begin{bmatrix}
		- \bm q^1 \\
		\bm h^1_{\cI_z}
	\end{bmatrix} z
	\nonumber \\ 
	\Leftrightarrow ~
	&\begin{bmatrix}
		\hat{\bm r}(z) \\
		\hat{\bm u}_{\cI_z}(z)
	\end{bmatrix}
	= 
	\begin{bmatrix}
		P & G_{\cI_z}^\top \\
		G_{\cI_z} & 0
	\end{bmatrix}^{-1}
	\begin{bmatrix}
		- \bm q^0 \\
		\bm h^0_{\cI_z}
	\end{bmatrix}
	+ 
	\begin{bmatrix}
		P & G_{\cI_z}^\top \\
		G_{\cI_z} & 0
	\end{bmatrix}^{-1}
	\begin{bmatrix}
		- \bm q^1 \\
		\bm h^1_{\cI_z}
	\end{bmatrix} z. \label{eq:piecewise_z}
\end{align}
Similarly, for $z^\prime$, we obtain 
\begin{align} \label{eq:piecewise_z_prime}
	\begin{bmatrix}
		\hat{\bm r}(z^\prime) \\
		\hat{\bm u}_{\cI_{z^\prime}}(z^\prime)
	\end{bmatrix}
	= 
	\begin{bmatrix}
		P & G_{\cI_{z^\prime}}^\top \\
		G_{\cI_{z^\prime}} & 0
	\end{bmatrix}^{-1}
	\begin{bmatrix}
		- \bm q^0 \\
		\bm h^0_{\cI_{z^\prime}}
	\end{bmatrix}
	+ 
	\begin{bmatrix}
		P & G_{\cI_{z^\prime}}^\top \\
		G_{\cI_{z^\prime}} & 0
	\end{bmatrix}^{-1}
	\begin{bmatrix}
		- \bm q^1 \\
		\bm h^1_{\cI_{z^\prime}}
	\end{bmatrix} z^\prime.
\end{align}
By subtracting \eq{eq:piecewise_z} from \eq{eq:piecewise_z_prime} and $\cI_z = \cI_{z^\prime}$, we can express the following:
\begin{align} \label{eq:piecewise_proof}
	\begin{bmatrix}
		\hat{\bm r}(z^\prime) \\
		\hat{\bm u}_{\cI_z}(z^\prime)
	\end{bmatrix}
	-
	\begin{bmatrix}
		\hat{\bm r}(z) \\
		\hat{\bm u}_{\cI_z}(z)
	\end{bmatrix}
	= 	
	\begin{bmatrix}
		P & G_{\cI_z}^\top \\
		G_{\cI_z} & 0
	\end{bmatrix}^{-1}
	\begin{bmatrix}
		- \bm q^1 \\
		\bm h^1_{\cI_z}
	\end{bmatrix} 
	\times (z^\prime - z).
\end{align}
We denote 
$
\begin{bmatrix}
	\bm \psi (z) \\ 
	\bm \gamma (z)
\end{bmatrix}
= 
\begin{bmatrix}
	P & G_{\cI_z}^\top \\
	G_{\cI_z} & 0 
\end{bmatrix}^{-1}
\begin{bmatrix}
	- \bm q^1 \\
	\bm h^1_{\cI_z}
\end{bmatrix}
$
with $\bm \psi (z) \in \RR^{{\rm row}(P)}$ and $\bm \gamma (z) \in \RR^{|\cI_z|}$,
and we subsequently achieve the results in Lemma \ref{lemm:piecewise_linear}.
\end{proof}

For simplicity, we assume that the generalized lasso solution is unique for any $\bm y(z), z \in \RR$. The uniqueness in the generalized lasso problem has been studied in \citet{ali2019generalized}. In this case, it can be guaranteed that the matrix inverse in Equation \eq{eq:piecewise_proof} always exists. 
If this is not the case, we can use parametric quadratic programming for the degenerate cases in \citet{Best96}.

\paragraph{Computation of breakpoint.} According to Lemma \ref{lemm:piecewise_linear}, the solution $\hat{\bm r}(z)$ is a linear function of $z$ until $z$ reaches a breakpoint, at which one component of $\hat{\bm u}(z)$ enters or leaves the set $\cI_z$.
At this point, we discuss the identification of the breakpoint.
\begin{lemm} \label{eq:compute_breakpoint}
Consider a real value $z$. 
Subsequently, $\cI_z = \cI_{z^\prime}$ for any real value $z^\prime$ in the interval $[z, z + t_z)$, where $z + t_z$ is the value of the breakpoint:
\begin{align}
	t_z = \min \{t_z^1, t_z^2\} &, \\ 
	t_z^1 = \min \limits_{j \in \cI_z^c} \left ( - \frac{ ( G_{\cI^c_z} \hat{\bm r}(z) - \bm h^0_{\cI_z^c} - \bm h^1_{\cI_z^c}z )_j}{(G_{\cI^c_z} \bm \psi (z) - \bm h^1_{\cI_z^c})_j} \right )_{++} 
	\quad &\text{ and } \quad 
	t_z^2 = \min \limits_{j \in \cI_z} \left ( - \frac{\hat{u}_j(z)}{\gamma_j(z)} \right )_{++}.
\end{align}
In this case, for any $a \in \RR$, $(a)_{++} = a$ if $a > 0$, and $(a)_{++} = \infty$ otherwise.
\end{lemm}

\begin{proof}
We first illustrate how to derive $t_z^1$. According to \eq{eq:psi}, we obtain 
\begin{align*}
	\hat{\bm r}(z^\prime) =   \hat{\bm r}(z) + \bm \psi (z) \times (z^\prime - z). 
\end{align*}
Thereafter, we need to guarantee 
\begin{align}
	G_{\cI^c_z} \hat{\bm r}(z^\prime) - \bm h^0_{\cI_z^c} - \bm h^1_{\cI_z^c} z^\prime &\leq 0 \nonumber\\ 
	\Leftrightarrow ~ G_{\cI^c_z} (\hat{\bm r}(z) + \bm \psi (z) \times (z^\prime - z)) - \bm h^0_{\cI_z^c} - \bm h^1_{\cI_z^c} \times (z^\prime - z) - \bm h^1_{\cI_z^c}z  &\leq 0 \nonumber\\ 
	\Leftrightarrow ~  \left ( G_{\cI^c_z} \bm \psi (z) - \bm h^1_{\cI_z^c} \right ) \times (z^\prime - z) & \leq - ( G_{\cI^c_z} \hat{\bm r}(z) - \bm h^0_{\cI_z^c} - \bm h^1_{\cI_z^c} z)  \label{eq:first_condition}.
\end{align}
The right-hand side of \eq{eq:first_condition} is positive because $G_{\cI^c_z} \hat{\bm r}(z) - \bm h^0_{\cI_z^c} - \bm h^1_{\cI_z^c} z \leq 0$.
Therefore, to satisfy Equation \eq{eq:first_condition}, 
\begin{align*}
	z^\prime - z \leq \min \limits_{j \in \cI_z^c} \left ( - \frac{ ( G_{\cI^c_z} \hat{\bm r}(z) - \bm h^0_{\cI_z^c} - \bm h^1_{\cI_z^c} z)_j}{(G_{\cI^c_z} \bm \psi (z) - \bm h^1_{\cI_z^c})_j} \right )_{++} = t_z^1.
\end{align*}
Next, we explain how to derive $t_z^2$.
According to \eq{eq:gamma}, we obtain
\begin{align*}
	\hat{\bm u}_{\cI_z}(z^\prime) = \hat{\bm u}_{\cI_z} (z) + \bm \gamma (z) \times (z^\prime - z).
\end{align*} 
We need to guarantee 
\begin{align} \label{eq:second_condition}
	\hat{\bm u}_{\cI_z}(z^\prime) > 0 \Leftrightarrow \hat{\bm u}_{\cI_z} (z) + \bm \gamma (z) \times (z^\prime - z) > 0.
\end{align}
Therefore, to satisfy Equation \eq{eq:second_condition}, 
\begin{align*}
	z^\prime - z < \min \limits_{j \in \cI_z} \left ( - \frac{\hat{u}_j(z)}{\gamma_j(z)} \right )_{++} = t_z^2.
\end{align*}
Finally, using $t_z = \min \{t_z^1, t_z^2\}$, we obtain the interval in which $\cI_{z^\prime} = \cI_z$ for any $z^\prime \in [z, z + t_z)$.
\end{proof}

% ================ Algorithm ================ 

\subsection{Algorithm} \label{subsec:algorithm}

\begin{algorithm}[!t]
\renewcommand{\algorithmicrequire}{\textbf{Input:}}
\renewcommand{\algorithmicensure}{\textbf{Output:}}
\begin{footnotesize}
 \begin{algorithmic}[1]
  \REQUIRE $X, {\bm y}^{\rm obs}, \lambda, D, [z_{\rm min}, z_{\rm max}]$
	\vspace{2pt}
	\STATE Obtain observed active set $\cM_{\rm obs} = \cA(\bm y^{\rm obs})$ for data $(X, {\bm y}^{\rm obs})$
	\vspace{2pt}
	\FOR {each selected $j \in \cM_{\rm obs}$}
		\vspace{2pt}
		\STATE Compute $\bm \eta_j$, and subsequently calculate $\bm a$ and $\bm b$ based on $\bm \eta_j$ and $\bm y^{\rm obs}$ $\leftarrow$ Equation (\ref{eq:parametrized_data_space})
		\vspace{4pt}
		\STATE $\cA(\bm y(z)) \leftarrow {\tt compute\_solution\_path}$ ($X$, $\lambda$,  $D$, $\bm a$, $\bm b$, $[z_{\rm min}, z_{\rm max}]$)
		\vspace{4pt}
		\STATE Truncation region $\cZ \leftarrow \{z: \cA(\bm y(z)) = \cM_{\rm obs}\}$
		\vspace{4pt}
		\STATE $P^{\rm selective}_j \leftarrow $ Equation (\ref{eq:selective_p_value}) (and$/$or selective confidence interval)
		\vspace{2pt}
	\ENDFOR
	\vspace{2pt}
  \ENSURE $\{P^{\rm selective}_j\}_{j \in \cM_{\rm obs}}$ (and$/$or selective confidence intervals)
 \end{algorithmic}
\end{footnotesize}
\caption{{\tt parametric\_SI}}
\label{alg:parametric_SI}
\end{algorithm}

\begin{algorithm}[t]
\renewcommand{\algorithmicrequire}{\textbf{Input:}}
\renewcommand{\algorithmicensure}{\textbf{Output:}}
\begin{footnotesize}
 \begin{algorithmic}[1]
  \REQUIRE $X, \lambda, D, \bm a, \bm b, [z_{\rm min}, z_{\rm max}]$
	\vspace{2pt}
	\STATE Initialization: $k = 1$, $z_k=z_{\rm min}$, $\cT = {z_k}$
	\vspace{2pt}
	\WHILE {$z_k < z_{\rm max}$}
%		\vspace{2pt}
%		\STATE $\bm y(z_k) = \bm a + \bm b z_k$
		\vspace{4pt}
		\STATE $t_{z_k}, \cM_{z_k} \leftarrow {\tt compute\_step\_size}(X, z_k, \bm a, \bm b, \lambda, D)$
		\vspace{4pt}
		\STATE $z_{k+1} = z_k + t_{z_k}$, $\cT = \cT \cup \{z_{k+1}\}$ ($z_{k+1}$ is the value of the next breakpoint)
		\vspace{2pt}
		\STATE $k = k + 1$
		\vspace{2pt}
	\ENDWHILE
	\vspace{2pt}
  %\ENSURE $\{(\hat{\bm \beta}(z_k), z_k)\}_{z_k \in \cT}$
  \ENSURE $\{\cM_{z_k}\}_{k \in [|\cT| - 1 ]}$
 \end{algorithmic}
\end{footnotesize}
\caption{{\tt compute\_solution\_path}}
\label{alg:solution_path}
\end{algorithm}

\begin{algorithm}[t]
\renewcommand{\algorithmicrequire}{\textbf{Input:}}
\renewcommand{\algorithmicensure}{\textbf{Output:}}
\begin{footnotesize}
 \begin{algorithmic}[1]
  \REQUIRE $X, z, \bm a, \bm b, \lambda, D$
  	\vspace{2pt}
	\STATE $\bm y(z) = \bm a + \bm b z$
	\STATE Compute $\hat{\bm r}(z)$ for data $(X, \bm y(z))$ and calculate $\hat{\bm \xi} (z) = \hat{\bm \xi}^{+} (z) - \hat{\bm \xi}^{-} (z)$ based on $\hat{\bm r}(z)$ 
	\vspace{2pt}
	\STATE Obtain $\cM_z = \cA(\bm y(z)) = \{j : \hat{\xi}_j(z) \neq 0\}$
	\vspace{4pt}
	\STATE Compute $t_z$ $\leftarrow$ Lemma \ref{eq:compute_breakpoint}
	\vspace{2pt}
  \ENSURE $t_z, \cM_{z}$
 \end{algorithmic}
\end{footnotesize}
\caption{{\tt compute\_step\_size}}
\label{alg:compute_step_size}
\end{algorithm}

In this section, we present the detailed algorithm of the proposed PP-based SI method. In Algorithm \ref{alg:parametric_SI}, to obtain the active set, we simply apply generalized lasso to the data $(X, \bm y^{\rm obs})$, and we obtain $\cM_{\rm obs}$.
Thereafter, we conduct SI for each observed active set.
For every $j \in \cM_{\rm obs},$ we first obtain the direction of interest $\bm \eta_j$. 
The main task is to compute the solution path of $\hat{\bm r}(z) $ in Equation (\ref{eq:qp}) for the parameterized response vector $\bm y (z)$,
where the parameterized solution $\hat{\bm r}(z)$ varies for different $j \in \cM_{\rm obs}$ because the direction of interest $\bm \eta_j$ is dependent on $j$.
This task can be achieved by Algorithm \ref{alg:solution_path}.
Finally, after obtaining the path, we can easily determine the truncation region $\cZ$, which is used to compute the selective $p$-value or selective confidence interval.

In Algorithm \ref{alg:solution_path}, a sequence of breakpoints is computed individually.
The algorithm is initialized at $z_k = z_{\rm min}, k = 1$.
At each $z_k$, the task is to determine the next breakpoint $z_{k+1}$.
This task can be performed by computing the step size in Algorithm \ref{alg:compute_step_size}.
This step is repeated until $z_k > z_{\rm max}$. 
%
%The algorithm returns the path of  $\hat{\bm r}(z)$ and the sequence of models corresponding to the breakpoints.
%
By identifying all of the breakpoints $\{ z_t \}_{t \in [|\cT|]}$, the entire path of $\cM_z$ for $z \in \RR$ is expressed as
\begin{align*}
 \cM_z = \cA(\bm y(z))
 =
 \mycase{
 \cA(\bm y(z_1)) & \text{ if } z \in [z_1, z_2], \\
 \cA(\bm y(z_2)) & \text{ if } z \in [z_2, z_3], \\
 ~~~~~ \vdots & \\
 \cA(\bm y(z_{|\cT| - 1}))& \text{ if } z \in [z_{|\cT| - 1}, z_{|\cT|}].
 }
\end{align*}

\paragraph{Selection of $[z_{\rm min}, z_{\rm max}]$.} 
Under normality, very positive and negative values of $z$ do not affect the inference. Therefore, it is reasonable to consider a range of values, such as $[ - 20 \sigma, 20 \sigma]$ \citep{liu2018more} or $\left [ - |\bm \eta_j^\top \bm y^{\rm obs}| - 20 \sigma, |\bm \eta_j^\top \bm y^{\rm obs}| + 20 \sigma \right ]$ \citep{sugiyama2020more}, where $\sigma$ is the standard deviation of the sampling distribution of the test statistic.

In Line 2 of Algorithm \ref{alg:compute_step_size}, $\hat{\bm r} (z)$ can be computed based on the KKT conditions.
However, it is well known that numerical issues will arise. 
Therefore, in our experiments, we modify the algorithm slightly to overcome these numerical problems.
In particular, we first replace Line 1 in Algorithm \ref{alg:compute_step_size} with $\bm y (z) = \bm a + \bm b (z + \Delta z)$, where $\Delta z$ is a small value such that $z_k  + \Delta z < z_{k + 1}$ for all $k \in [|\cT| - 1]$.
Subsequently, at Line 2 of Algorithm \ref{alg:compute_step_size}, we simply obtain $\hat{\bm r} (z)$ by applying the QP solver to $\bm y(z)$.
We can confirm whether $\Delta z$ is sufficiently small by verifying whether exactly one component in the vector of the optimal Lagrange multipliers $\hat{\bm u}(z)$ has already entered or left the set $\cI_z$.
This condition is satisfied in all of the experiments by simply setting $\Delta z = 0.0001$.

The complexity of Algorithm \ref{alg:parametric_SI} is dependent on the number of breakpoints. 
In the literature on PP, the worst-case complexity increases exponentially with the problem size~\citep{Ritter84, Allgower93, Gal95, Best96, mairal2012complexity}. 
However, in practice, it has been reported that the number of breakpoints is approximately linear with the problem size and it does not actually increase as much as in the theoretical worst case. 
This has also been noted in regularization path studies~\citep{osborne2000new, Efron04a, HasRosTibZhu04, mairal2012complexity}. 
In fact, in all of the experiments in \S6, the number of breakpoints is within a reasonable size and the computational cost is not a major problem of the proposed method.
\section{Generality of Proposed Method}

Although we only focused on generalized lasso in \S3, the parametric QP formulation in \eq{eq:qp} is more general and the forms of matrices $P$ and $G$ as well as vectors $\bm q^0, \bm q^1, \bm h^0$, and $\bm h^1$ can be changed depending on the problem.
That is, the method proposed method in \S3 is flexible and can be applied to any problem that can be converted into a parametric QP in the form of \eq{eq:qp}.
In this section, we demonstrate the extensions of the proposed method for testing the statistical significance of the features that are selected by various feature selection algorithms.

When applying a feature selection algorithm $\cA$ to the observed response vector $\bm y^{\rm obs}$, the observed active set can be defined as follows: 
\begin{align*}
	\cM_{\rm obs} = \cA(\bm y^{\rm obs}) = \{j : \hat{\beta}_j \neq 0\}.
\end{align*}
%
%Here, $\cM_{\rm obs}$ can be also called the observed \emph{active set}.
%
In this setting, $D = I_p$.
To test the selected features in $\cM_{\rm obs}$, the conditional inference is the same as that defined in \eq{eq:condition_model}
and the characterization of the truncation region $\cZ$ is the same as \eq{eq:truncation_region_z}.
To identify $\cZ$, the remaining task is to compute the solution path $\hat{\bm \beta}(z)$ for $z \in \RR$ and to identify the intervals of $z$ in which we obtain the same active set as $\bm y^{\rm obs}$.
In the following sections, we present the parametric QP formulations for vanilla lasso, elastic net, non-negative least squares, and Huber regression.
As all of these regression problems can be converted into the form of \eq{eq:qp}, the path of $\hat{\bm \beta}(z)$ can be computed within finite operations by using PP, as demonstrated in \S \ref{subsec:piecewise}.

\paragraph{Parametric QP for vanilla lasso.} Vanilla lasso with a parameterized response vector $\bm y (z)$ for $z \in \RR$ is defined as 
\begin{align}\label{eq:vanilla_lasso}
	 \hat{\bm \beta}_{\rm lasso}(z)&= \argmin \limits_{\bm \beta \in \RR^p} \frac{1}{2} ||\bm y (z) - X \bm \beta||_2^2 + \lambda ||\bm \beta||_1.
\end{align}
\begin{lemm}\label{lemma:lasso_qp}
By decomposing $\bm \beta = \bm \beta^{+} - \bm \beta^{-}$, $\bm \beta^{+}, \bm \beta^{-} \geq \bm 0_p$, the lasso problem in \eq{eq:vanilla_lasso} can be solved by the following parametric QP:
\begin{align*}
 	\left ( \hat{\bm \beta}^{+}_{\rm lasso} (z), \hat{\bm \beta}^{-}_{\rm lasso} (z) \right ) = \argmin \limits_{\bm \beta^{+} , \bm \beta^{-} \in \RR ^p}
	& \quad  \frac{1}{2} 
	\begin{pmatrix}
		\bm \beta^{+} \\ \bm \beta^{-}
	\end{pmatrix}^\top 
	\begin{pmatrix}
		X^\top X & - X^\top X \\ 
		- X^\top X & X^\top X
	\end{pmatrix}  
	\begin{pmatrix}
		\bm \beta^{+} \\ \bm \beta^{-}
	\end{pmatrix} \\ 
	& \quad + \left (\lambda \bm 1_{2p} -
	\begin{pmatrix}
		X^\top \bm y(z) \\ -X^\top \bm y(z)
	\end{pmatrix}
	\right)^\top 
	\begin{pmatrix}
		\bm \beta^{+} \\ \bm \beta^{-}
	\end{pmatrix}
	\\
\textrm{s.t.}  & \quad  \bm \beta^{+} \geq \bm 0_p, \bm \beta^{-} \geq \bm 0_p.
\end{align*}
\end{lemm}

\begin{proof}
According to \eq{eq:vanilla_lasso}, we obtain 
\begin{align}
	\hat{\bm \beta}_{\rm lasso}(z) &= \argmin \limits_{\bm \beta \in \RR^p} \frac{1}{2} ||\bm y (z) - X \bm \beta||_2^2 + \lambda ||\bm \beta||_1 \nonumber \\ 
	& = \argmin \limits_{\bm \beta \in \RR^p}  \frac{1}{2} 
	\left ( \bm y(z)^\top \bm y(z)  - \bm \beta^\top X^\top \bm y(z) - \bm y(z)^\top X \bm \beta  + \bm \beta^\top X^\top X \bm \beta  \right ) + \lambda ||\bm \beta||_1 \label{eq:lasso_1}.
\end{align}
Similar to the proof of Lemma \ref{lemm:qp_for_fused_lasso}, as the component $\frac{1}{2} \bm y(z)^\top \bm y(z)$ does not affect the optimal solution $\hat{\bm \beta}_{\rm lasso}(z)$, we can rewrite \eq{eq:lasso_1} as 
\begin{align}
	\hat{\bm \beta}_{\rm lasso}(z) = \argmin \limits_{\bm \beta \in \RR^p}  \frac{1}{2} \bm \beta^\top X^\top X \bm \beta
	   - (X^\top \bm y(z))^\top \bm \beta  + \lambda ||\bm \beta||_1 \label{eq:lasso_2}.
\end{align}
Finally, by decomposing $\bm \beta = \bm \beta^{+} - \bm \beta^{-}$, $\bm \beta^{+}, \bm \beta^{-} \geq \bm 0_p$ and $||\bm \beta||_1 = \sum_{j \in [p]} (\beta^{+}_j + \beta^{-}_j)$, we obtain the result in Lemma \ref{lemma:lasso_qp}.
\end{proof}

In our preliminary conference paper \citep{le2021parametric}, we initially presented the idea of introducing PP for vanilla lasso in a slightly different (but essentially the same) manner.

\paragraph{Parametric QP for elastic net.} The elastic net with a parameterized response vector $\bm y (z)$ for $z \in \RR$ is defined as 
\begin{align} \label{eq:elastic_net}
	\hat{\bm \beta}_{\rm elastic} (z) = \argmin \limits_{\bm \beta \in \RR^p} \frac{1}{2n} ||\bm y (z) - X \bm \beta||_2^2 + \lambda ||\bm \beta||_1 + \frac{1}{2} \zeta ||\bm \beta||_2^2.
\end{align}
\begin{lemm}\label{lemma:elastic_net_qp}
By decomposing $\bm \beta = \bm \beta^{+} - \bm \beta^{-}$, $\bm \beta^{+}, \bm \beta^{-} \geq \bm 0_p$, the elastic net problem in \eq{eq:elastic_net} can be solved using the following parametric QP:
\begin{align*}
 	\left ( \hat{\bm \beta}^{+}_{\rm elastic} (z), \hat{\bm \beta}^{-}_{\rm elastic} (z) \right )= \argmin \limits_{\bm \beta^{+} , \bm \beta^{-} \in \RR ^p}
	& \quad  \frac{1}{2n} 
	\begin{pmatrix}
		\bm \beta^{+} \\ \bm \beta^{-}
	\end{pmatrix}^\top 
	\begin{pmatrix}
		X^\top X & - X^\top X \\ 
		- X^\top X & X^\top X
	\end{pmatrix}  
	\begin{pmatrix}
		\bm \beta^{+} \\ \bm \beta^{-}
	\end{pmatrix} \\ 
	& \quad +
	\zeta \cdot \frac{1}{2}
	\begin{pmatrix}
		\bm \beta^{+} \\ \bm \beta^{-}
	\end{pmatrix}^\top 
	\begin{pmatrix}
		I_p & - I_p \\ 
		- I_p & I_p
	\end{pmatrix}  
	\begin{pmatrix}
		\bm \beta^{+} \\ \bm \beta^{-}
	\end{pmatrix}\
	\\ 
	& \quad + \left (\lambda \bm 1_{2p} -
	\frac{1}{n}
	\begin{pmatrix}
		X^\top \bm y(z) \\ - X^\top \bm y(z)
	\end{pmatrix}
	\right)^\top 
	\begin{pmatrix}
		\bm \beta^{+} \\ \bm \beta^{-}
	\end{pmatrix}
	\\
\textrm{s.t.}  & \quad  \bm \beta^{+} \geq \bm 0_p, \bm \beta^{-} \geq \bm 0_p.
\end{align*}
\end{lemm}
\begin{proof}
The proof of Lemma \ref{lemma:elastic_net_qp} is similar to the proof of Lemma \ref{lemma:lasso_qp}.
\end{proof}

\paragraph{Parametric QP for non-negative least squares.} The non-negative least squares problem with a parametrized response vector $\bm y (z)$ for $z \in \RR$ is defined as 
\begin{align*}
	\hat{\bm \beta}_{\rm non-negative} (z) = \argmin  \limits_{\bm \beta \in \RR^p} \frac{1}{2} ||\bm y (z) - X \bm \beta||_2^2  \quad 
	\textrm{s.t}  \quad  \bm \beta \geq 0.
\end{align*}
The above problem can be formulated as the following parametric QP: 
\begin{align*}
	\hat{\bm \beta}_{\rm non-negative} (z) = \argmin  \limits_{\bm \beta \in \RR^p} \frac{1}{2} \bm \beta^\top X^\top X \bm \beta - \left ( X^\top \bm y (z) \right )^\top \bm \beta
	\quad  
	\textrm{s.t} \quad  \bm \beta \geq 0.
\end{align*}

\paragraph{Parametric QP for Huber regression with $\ell_1$ penalty.}
The Huber regression with the $\ell_1$ penalty for a parameterized response vector $\bm y(z)$ is formulated as
\begin{align}
 \label{eq:eq_huber}
 \hat{\bm \beta}_{\rm huber}(z)
 =
 \argmin_{\bm \beta \in \RR^p}
 \sum_{i=1}^n
 L_{\delta}(y_i(z) - \bm x_i^\top \bm \beta) 
 + 
 \lambda ||\bm \beta||_1,
\end{align}
with
\begin{align*}
 L_\delta(e)
 =
 \mycase{
 \frac{1}{2} e^2
 &
 \text{ if } |e| \le \delta,
 \\
 \delta (|e| - \frac{1}{2} \delta)
 &
 \text{ otherwise},
 }
\end{align*}
where $\delta > 0$ is also a predetermined tuning parameter.
Let $\bm e = \bm y(z) - X \bm \beta$, $\bm \phi \in \RR^n$, and $\bm \nu \in \RR^n$ be the vectors in which the $i^{\rm th}$ element $\phi_i$ and $\nu_i$ are respectively defined as 
\begin{align*}
	\phi_i = \min \{|e_i|, \delta \}, \quad \nu_i = \max \{|e_i| - \delta, 0\}, 
\end{align*}
where $e_i$ is the $i^{\rm th}$ element of $\bm e$.
Obviously, $|e_i| = \phi_i + \nu_i$.
Subsequently, by decomposing $\bm \beta = \bm \beta^{+} - \bm \beta^{-}$, $\bm \beta^{+}, \bm \beta^{-} \geq \bm 0_p$, the problem in \eq{eq:eq_huber} can be solved by the following parametric QP:
\begin{align*}
	\left ( \hat{\bm \phi}(z), \hat{\bm \nu} (z), \hat{\bm \beta}^{+}_{\rm huber}(z), \hat{\bm \beta}^{-}_{\rm huber}(z) \right ) = \argmin \limits_{\bm \phi, \bm \nu, \bm \beta^{+}, \bm \beta^{-}}  & \frac{1}{2} \bm \phi^\top \bm \phi + \delta \cdot \bm 1_n^\top  \bm \nu + \sum \limits_j \left( \beta_j^{+} + \beta_j^{-} \right ) \\ 
	\textrm{s.t }  & -\bm \phi - \bm \nu \leq \bm y(z) - X\left( \bm \beta^{+} - \bm \beta^{-} \right) \leq \bm \phi + \bm \nu, \\ 
	& \bm 0_n \leq \bm u \leq \delta \cdot \bm 1_n, ~\bm v \geq \bm 0_n, \\
	& \bm \beta^{+} \geq \bm 0_p, ~ \bm \beta^{-} \geq \bm 0_p.
\end{align*}

The lasso-like formulation of the Huber regression has been discussed in \citet{she2011outlier}. Conditional SI for outliers was studied in \citet{chen2019valid}, and we recently investigated its PP version in \citet{tsukurimichi2021conditional}.

\section{Characterization of CV-Based Tuning Parameter Selection Event}  \label{sec:cv_event} 

Various ML tasks involve careful tuning of a regularization parameter $\lambda$ that controls the balance between an empirical loss term and a regularization term; for example, this is commonly achieved by CV. However, the majority of the current SI studies have assumed a pre-specified $\lambda$ and have ignored the fact that $\lambda$ is selected based on the data, because it is difficult to characterize the CV selection event. \cite{loftus2015selective} and \citet{markovic2017unifying} proposed solutions to incorporate CV events. However, the former work required additional conditioning on all intermediate models, which led to a loss of power, whereas the latter considered a randomization version of CV instead of vanilla CV.

In this section, we introduce a new means of characterizing the \emph{minimal} selection event in which $\lambda$ is selected based on the data; for example, via CV \footnote{We note that the following discussion is only applicable when the number of features $p$ is independent of $n$.}.
For notational simplicity, we consider the case in which the data are divided into training and validation sets, and the latter is used for selecting $\lambda$. The following discussion can easily be extended to CV scenarios. We rewrite the observed data as follows: 
\begin{align*}
\{X, \bm{y}^{\rm obs}\} = 
\left \{
(X_{\rm train}\ X_{\rm val} )^\top \in \RR^{n \times p},
(\bm y^{\rm obs}_{\rm train}\ \bm y^{\rm obs}_{\rm val})^\top \in\RR^{n}
\right \}.
\end{align*}
Given a set of regularization parameter candidates $\Lambda$, the process of selecting $\lambda$ is as follows:
\begin{enumerate}
	\item  For each $\lambda \in \Lambda$, we first obtain $\hat{\bm \beta}_\lambda$ using the training data
	\begin{align*} %\label{eq:training_step}
		\hat{\bm \beta}_\lambda \in \argmin \limits_{\bm \beta} 		\frac{1}{2} \|\bm y_{\rm train}^{\rm obs} - X_{\rm train} \bm \beta\|^2_2 + \lambda \|D \bm \beta\|_1.
	\end{align*}
	Thereafter, the validation error is defined as 
	\begin{align*} %\label{eq:validation_step}
		E_\lambda = \frac{1}{2} \|\bm y_{\rm val}^{\rm obs} - X_{\rm val} \bm \hat{\bm \beta}_\lambda\|^2_2.
	\end{align*}
	\item We select $\lambda^{\rm obs} = \lambda \in \Lambda$, which has the corresponding smallest validation error $E_\lambda$.

\end{enumerate}
The selection event of the above validation process is defined as
\begin{align} \label{eq:cv_event}
	\{\cV(\bm Y) = \cV(\bm y^{\rm obs})\},
\end{align}
where $\cV(\bm y^{\rm obs}) =  \lambda^{\rm obs} \in \Lambda$ is the event that $\lambda^{\rm obs}$ is selected when validation is performed on $\bm y^{\rm obs}$.

After selecting $\lambda^{\rm obs}$, we can obtain the observed active set $\cM_{\rm obs}$ by applying generalized lasso on $\bm y^{\rm obs}$ with the selected $\lambda^{\rm obs}$, which can be defined as in \eq{eq:selection_event}.
However, to conduct a statistical test for each element in $\cM_{\rm obs}$, the conditional inference will be different from \eq{eq:condition_model}, because we must incorporate the selection event of the validation process.
Therefore, for each $j \in \cM_{\rm obs}$, we consider the following conditional inference: 
\begin{align}\label{eq:conditional_model_with_cv}
	\bm \eta_j^\top \bm Y \mid 
	 \{ 
		\cA(\bm Y) = \cA(\bm y^{\rm obs}),  
		\cV(\bm Y) = \cV(\bm y^{\rm obs}),
		\bm q(\bm Y) = \bm q(\bm y^{\rm obs})
	 \}.
\end{align}
According to the third condition, the data are restricted on the line, as discussed in \S \ref{subsec:conditional_data_space_characterization}. 
Therefore, the conditional data space in \eq{eq:conditional_model_with_cv} can be rewritten as: 
\begin{align*}
	\cY_{\rm CV} =\{\bm y(z) = \bm a + \bm bz \mid z \in \cZ_{\rm CV}\},
\end{align*}
where 
\begin{align*} %\label{eq:c_Z_cv}
	\cZ_{\rm CV} = \{z \in \RR \mid 
		\cA(\bm y(z)) = \cA(\bm y^{\rm obs}),
		\cV(\bm y(z)) = \cV(\bm y^{\rm obs})\}.
\end{align*}
The remaining task to conduct the inference is to identify $\cZ_{\rm CV}$.
We can decompose $\cZ_{\rm CV}$ into two separate sets 
\begin{align*}
	\cZ_{\rm CV} = \cZ_1 \cap \cZ_2,
\end{align*}
where 
$\cZ_1 = \{z \in \RR \mid \cA(\bm y(z)) = \cA(\bm y^{\rm obs})\}$
and 
$\cZ_2 = \{z \in \RR \mid \cV(\bm y(z)) = \cV(\bm y^{\rm obs})\}.$
The set $\cZ_1$ can easily be constructed using the method proposed in the previous sections. 
The remaining challenge is to identify $\cZ_2$.

To construct $\cZ_2$, it is necessary to identify the intervals of $z$ on which $\lambda^{\rm obs}$ has the smallest validation error.
That is, we can redefine 
\begin{align*}
	\cZ_2 = \{z \in \RR \mid E_{\lambda^{\rm obs}}(z) \leq E_\lambda(z) \text{ for any } \lambda \in \Lambda \},
\end{align*}
where 
\begin{align} 
	E_\lambda(z) = \frac{1}{2} &\|\bm y_{\rm val}(z) - X_{\rm val} \bm \hat{\bm \beta}_\lambda(z)\|^2_2, \label{eq:validation_step}, \\ 
	\hat{\bm \beta}_\lambda(z) \in \argmin \limits_{\bm \beta \in \RR^p} \frac{1}{2} &\|\bm y_{\rm train}(z) - X_{\rm train} \bm \beta\|^2_2 + \lambda \|D \bm \beta\|_1. \label{eq:training_step}
\end{align}
For each $\lambda \in \Lambda$, although it appears to be intractable to compute $E_{\lambda}(z)$ for infinitely many values of $z \in \RR$, the task can be completed within finite operations using PP.

\begin{algorithm}[!t]
\renewcommand{\algorithmicrequire}{\textbf{Input:}}
\renewcommand{\algorithmicensure}{\textbf{Output:}}
\begin{footnotesize}
 \begin{algorithmic}[1]
  \REQUIRE $X, \bm y^{\rm obs}, \Lambda, D, K, [z_{\rm min}, z_{\rm max}]$
  	\vspace{2pt}
	\STATE Conduct $K$-fold CV to select $\lambda^{\rm obs}$
	\vspace{2pt}
	\STATE Obtain $\cM_{\rm obs} = \cA(\bm y^{\rm obs})$ for data $(X, {\bm y}^{\rm obs})$ with the selected $\lambda^{\rm obs}$
	\vspace{2pt}
	\FOR {each $j \in \cM_{\rm obs}$}
		\vspace{2pt}
		\STATE Compute $\bm \eta_j$, and subsequently calculate $\bm a$ and $\bm b$ based on $\bm \eta_j$ and $\bm y^{\rm obs}$ $\leftarrow$ Equation (\ref{eq:parametrized_data_space})
		\vspace{4pt}
		\STATE Obtain $\cA(\bm y(z))$ using Algorithm \ref{alg:solution_path} $\leftarrow {\tt compute\_solution\_path}$ ($X$, $\lambda^{\rm obs}$,  $D$, $\bm a$, $\bm b$, $[z_{\rm min}, z_{\rm max}]$)
		\vspace{4pt}
		\STATE $\cZ_1 \leftarrow \{z: \cA(\bm y(z)) = \cM_{\rm obs}\}$ $~~$ (characterize model selection event)
		\vspace{4pt}
		\STATE $\cZ_2 \leftarrow$ {\tt compute\_$\cZ_2$}($X, \bm a, \bm b, \Lambda, D, K$) $~~$ (characterize CV selection event)
		\vspace{4pt}
		\STATE  $\cZ_{\rm CV} = \cZ_1 \cap \cZ_2$
		\vspace{4pt}
		\STATE Compute $P^{\rm selective}_j $ in Equation (\ref{eq:selective_p_value}) with truncation region $\cZ_{\rm CV}$
		\vspace{2pt}
	\ENDFOR
	\vspace{2pt}
  \ENSURE $\{P^{\rm selective}_j\}_{j \in \cM_{\rm obs}}$
 \end{algorithmic}
\end{footnotesize}
\caption{{\tt SI\_with\_$K$\_fold\_cross\_validation}}
\label{alg:cv_algorithm}
\end{algorithm}

As demonstrated in \S \ref{subsec:piecewise}, the optimal solution $\hat{\bm \beta}_\lambda(z)$ in \eq{eq:training_step} is a piecewise linear function of $z$.
That is, 
\begin{align} \label{eq:ev_1}
\hat{\bm \beta}_\lambda(z)
 =
 \mycase{
 \hat{\bm \beta}_\lambda(z_1) + \bm s_\lambda^1 (z - z_1) & \text{ if } z \in [z_1, z_2], \\
 \hat{\bm \beta}_\lambda(z_2) + \bm s_\lambda^2 (z - z_2) & \text{ if } z \in [z_2, z_3], \\
 ~~~~~~~~ \vdots & \\
 \hat{\bm \beta}_\lambda(z_{|\cT_\lambda| - 1}) + \bm s_\lambda^{|\cT_\lambda| - 1} (z - z_{|\cT_\lambda| - 1}) & \text{ if } z \in [z_{|\cT_\lambda| - 1}, z_{|\cT_\lambda|}],
 }
\end{align}
where $\bm s_\lambda^{k \in |\cT_\lambda|} $ is the slope vector, the elements of which are dependent on $\bm \psi (z_{k \in |\cT_\lambda|})$ in \eq{eq:psi}.
The subscript $\lambda$ in $\hat{\bm \beta}_\lambda(z), \bm s_\lambda^{k}, \cT_\lambda$ indicates that these components are dependent on $\lambda$.
Moreover, because we decompose $\bm y(z) = \left( \bm y_{\rm train}(z), \bm y_{\rm val}(z)\right)^\top$, the vectors $\bm a$ and $\bm b$ can be rewritten as follows: 
\begin{align*}
	\bm a = \left(\bm a_{\rm train} ~ \bm a_{\rm val}\right)^\top, \quad \bm b = \left(\bm b_{\rm train} ~ \bm b_{\rm val}\right)^\top.
\end{align*}
Therefore, we can write
\begin{align} \label{eq:ev_2}
	\bm y_{\rm val}(z) = \bm a_{\rm val} + \bm b_{\rm val} z.
\end{align}
According to the piecewise linearity of $\hat{\bm \beta}_\lambda (z)$ in \eq{eq:ev_1} and the linearity of $\bm y_{\rm val}(z)$ in \eq{eq:ev_2}, the validation error $E_\lambda(z)$ in \eq{eq:validation_step} is a \emph{piecewise quadratic} function of $z$, which can be expressed as follows:
\begin{align*} 
E_\lambda(z)
 =
 \mycase{
 (1/2)  || ( \bm a_{\rm val} - X_{\rm val} \hat{\bm \beta}_\lambda(z_1) + X_{\rm val} \bm s_\lambda^1 z_1 ) + ( \bm b_{\rm val} - X_{\rm val} \bm s_\lambda^1 ) z  ||_2^2 & \text{ if } z \in [z_1, z_2], \\
  (1/2)  || ( \bm a_{\rm val} - X_{\rm val} \hat{\bm \beta}_\lambda(z_2) + X_{\rm val} \bm s_\lambda^2 z_2 ) + ( \bm b_{\rm val} - X_{\rm val} \bm s_\lambda^2 ) z  ||_2^2 & \text{ if } z \in [z_2, z_3], \\
 ~~~~~~~~~~~~~~~~~~~~~~~~~~~~~ \vdots & \\
  (1/2)  || ( \bm a_{\rm val} - X_{\rm val} \hat{\bm \beta}_\lambda(z_{|\cT_\lambda| - 1}) + X_{\rm val} \bm s_\lambda^{|\cT_\lambda| - 1} z_{|\cT_\lambda| - 1}).  
  \\ ~~~~~~~~~~~~~~~~~~~~~~~~~~~~~~~~~~~~~~ + ( \bm b_{\rm val} - X_{\rm val} \bm s_\lambda^{|\cT_\lambda| - 1} ) z  ||_2^2 & \text{ if } z \in [z_{|\cT_\lambda| - 1}, z_{|\cT_\lambda|}].
 }
\end{align*}
At this point, for each $\lambda \in \Lambda$, we have a corresponding validation error $E_\lambda(z)$, which is a piecewise quadratic function of $z$. 
Finally, $\cZ_2$ can be identified by determining the intervals of $z$ in which the validation error $E_{\lambda^{\rm obs}}(z)$ corresponding to $\lambda^{\rm obs}$ is the minimum among a set of piecewise quadratic functions.
The procedure for the $K$-fold CV case is presented in Algorithm \ref{alg:cv_algorithm}.

% =================
\begin{algorithm}[!t]
\renewcommand{\algorithmicrequire}{\textbf{Input:}}
\renewcommand{\algorithmicensure}{\textbf{Output:}}
\begin{footnotesize}
 \begin{algorithmic}[1]
  \REQUIRE $X, \bm a, \bm b, \Lambda, D, K$
  	\vspace{2pt}
	\FOR {$\lambda \in \Lambda$}
	\vspace{2pt}
		\FOR {$k \leftarrow 1$ to $K$}
			\vspace{2pt}
			\STATE Compute $\hat{\bm \beta}_\lambda^k (z)$ as in \eq{eq:training_step} for fold $k$
			\vspace{2pt}
			\STATE Compute validation error $E_\lambda^k (z)$ as in \eq{eq:validation_step} for fold $k$
			\vspace{2pt}
		\ENDFOR
		\vspace{2pt}
		\STATE Compute mean error $\bar{E}_\lambda (z) = \frac{1}{K} \sum_{k=1}^K E_\lambda^k (z)$ among $K$ folds for current $\lambda$ candidate 
	\vspace{2pt}
	\ENDFOR
  	\vspace{2pt}
	\STATE $\cZ_2 = \{z \in \RR \mid \bar{E}_{\lambda^{\rm obs}}(z) \leq \bar{E}_\lambda(z) \text{ for any } \lambda \in \Lambda \}$
	\vspace{2pt}
  \ENSURE $\cZ_2$
 \end{algorithmic}
\end{footnotesize}
\caption{{\tt compute\_$\cZ_2$}}
\label{alg:compute_c_Z_2}
\end{algorithm}

\section{Experiments}
In this section, we discuss the performance evaluation of the proposed method.
First, the experimental setup is presented in \S \ref{subsec:exp_setup}.
Thereafter, the results on synthetic and real data are outlined in \S \ref{subsec:exp_numerical_results} and \S \ref{subsec:exp_real_data_results}, respectively.

\subsection{Experimental Setup} \label{subsec:exp_setup}

We conducted experiments on fused lasso, which is one of the most commonly studied cases of generalized lasso, as well as on feature selection methods including vanilla lasso, elastic net, non-negative least squares, and Huber regression + $\ell_1$ penalty.
We executed the code on an Intel(R) Xeon(R) CPU E5-2687W v4 @ 3.00 GHz.

\paragraph{Methods for comparison.} We present the false positive rate (FPR), true positive rate (TPR), and confidence interval (CI) for the following conditional inferences:

%$\bullet$ Proposed Method: conditional inference \emph{without} extra-conditioning as defined in \eq{eq:condition_model}, which is mainly focused in this paper.
%
%$\bullet$  Over-Conditioning (OC): conditional inference with extra-conditioning. 
%%
%Regarding the fused lasso, OC is the method proposed in \citet{hyun2018exact}.
%%
%In regard to vanilla lasso, elastic net, huber regression + $\ell_1$ penalty, OC is the polytope-based SI proposed in \citet{lee2016exact}.
%
%$\bullet$ Data Splitting (DS) \citep{cox1975note}: DS is the commonly used procedure for the purpose of selection bias correction. 
%%
%In this approach, the data is randomly divided in two halves — one half is used for model selection and the other half is used for inference. 

\begin{itemize}
%$\bullet$ 
\item Proposed method: Conditional inference \emph{without} extra conditioning, as defined in \eq{eq:condition_model}, which was the main focus of this study.

%$\bullet$ 
\item Over-conditioning (OC): Conditional inference with extra conditioning. 
Regarding fused lasso, OC is the method that was proposed in \citet{hyun2018exact}.
Regarding vanilla lasso, elastic net, and Huber regression + $\ell_1$ penalty, OC is the polytope-based SI that was proposed in \citet{lee2016exact}.

%$\bullet$ 
\item Data splitting (DS) \citep{cox1975note}: DS is the commonly used procedure for selection bias correction. 
In this approach, the data are randomly divided into two halves: one half is used for model selection and the other half is used for inference. 
\end{itemize}

\paragraph{Synthetic data generation for fused lasso.}
We set $X = I_n$, and the matrix $D \in \RR^{(n-1) \times n}$ was defined as 
\begin{align*}
\begin{pmatrix}
	-1 & 1 & 0 & \cdots & 0 & 0 \\ 
	0 & -1 & 1 & \cdots & 0 & 0 \\ 
	& & & \cdots & & \\ 
	0 & 0 & 0 & \cdots & -1 & 1
\end{pmatrix}.
\end{align*}
Regarding the FPR experiments, we generated 100 null data $\bm y = (y_1, ..., y_n)^\top$, where $y_{i \in [n]} \sim \NN(0, 1)$ for each $n \in \{70, 80, 90, 100\} $.
To test the TPR, we generated $\bm y = (y_1, ..., y_n)^\top$ with $n = 60$, in which 
	\begin{align*}
		y_i \sim \NN (\mu_i, 1), \quad \mu_i =
		\begin{cases}
			0, &\text{ if } 1 \leq i \leq 20 \text{ or } 41 \leq i \leq 60,\\ 
			0 + \Delta_{\mu}, &\text{ if } 21 \leq i \leq 40.
		\end{cases}
\end{align*}
The significance level $\alpha$ was set to 0.05.
We used Bonferroni correction to account for the multiplicity in all of the experiments.
If $v$ selected features (hypotheses) are tested simultaneously, Bonferroni correction tests each individual hypothesis at $\alpha^\ast = \alpha / v$.
We ran 100 trials for each $\Delta_{\mu} \in \{1, 2, 3, 4\}$ and we repeated the experiment 10 times.

\paragraph{Synthetic data generation for feature selection methods.} We generated $n$ outcomes as $y_i = \bm x_i^\top \bm \beta + \veps_i$,  $i = 1, ..., n$, where $\bm x_i \sim \NN(0, I_p)$, in which $p = 5$, and $\veps_i \sim \NN(0, 1)$.
We set the regularization parameter $\lambda = 1$ and significance level $\alpha = 0.05$.
Bonferroni correction was also applied.
For the FPR experiments, all of the elements of $\bm \beta$ were set to 0.
For the TPR experiments, the first two elements of $\bm \beta$ were set to 0.25. 
We ran 100 trials for each $n \in \{50, 100, 150, 200\}$ and we repeated this experiment 10 times.
For the CI experiments, we set $n=100, p=5$ and ran 100 trials.

\begin{figure}[!t]
\begin{subfigure}{.49\textwidth}
  \centering
  \includegraphics[width=.8\linewidth]{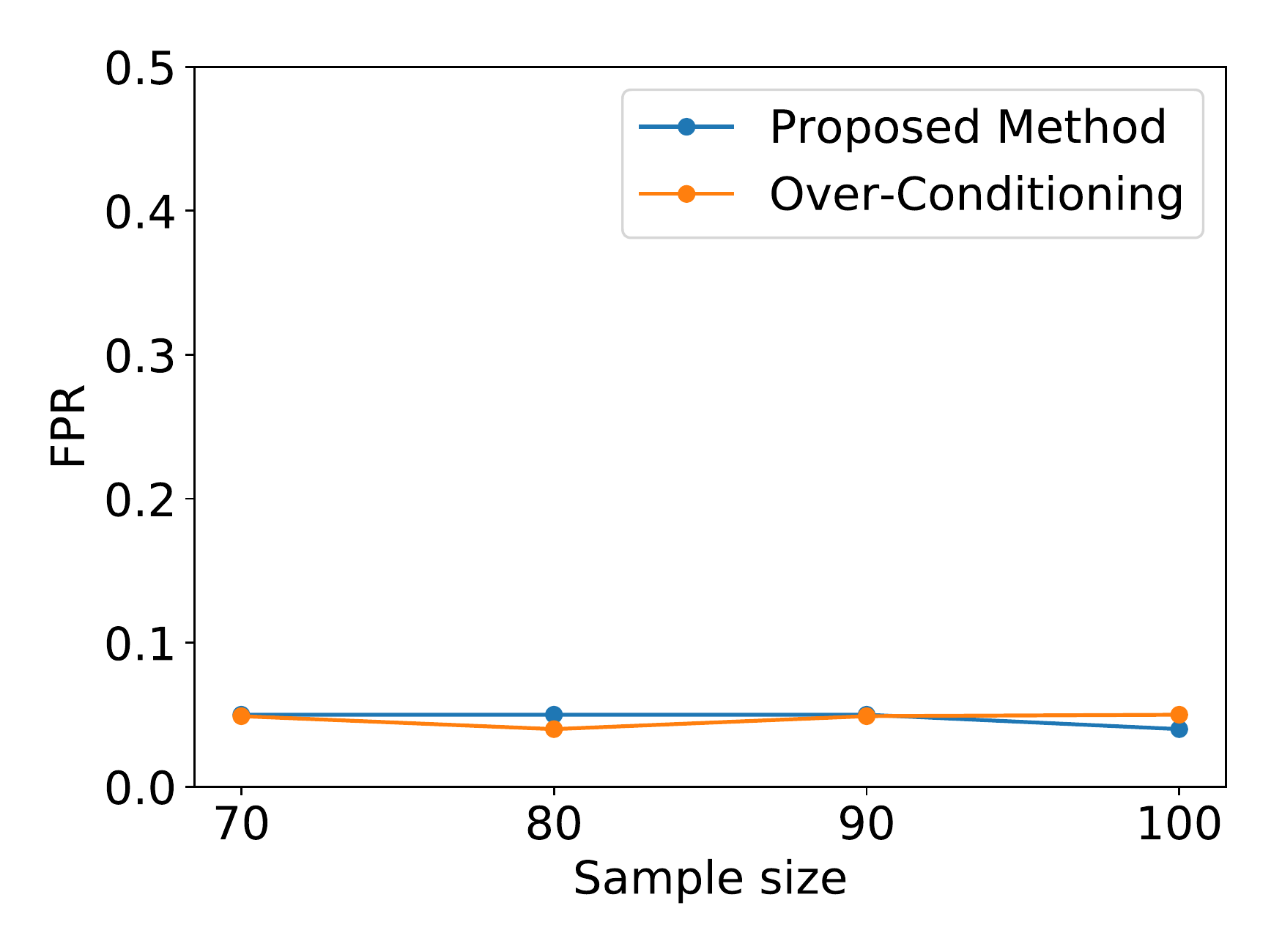}  
%  \vspace{-8pt}
  \caption{FPR}
\end{subfigure}
\begin{subfigure}{.49\textwidth}
  \centering
  \includegraphics[width=.8\linewidth]{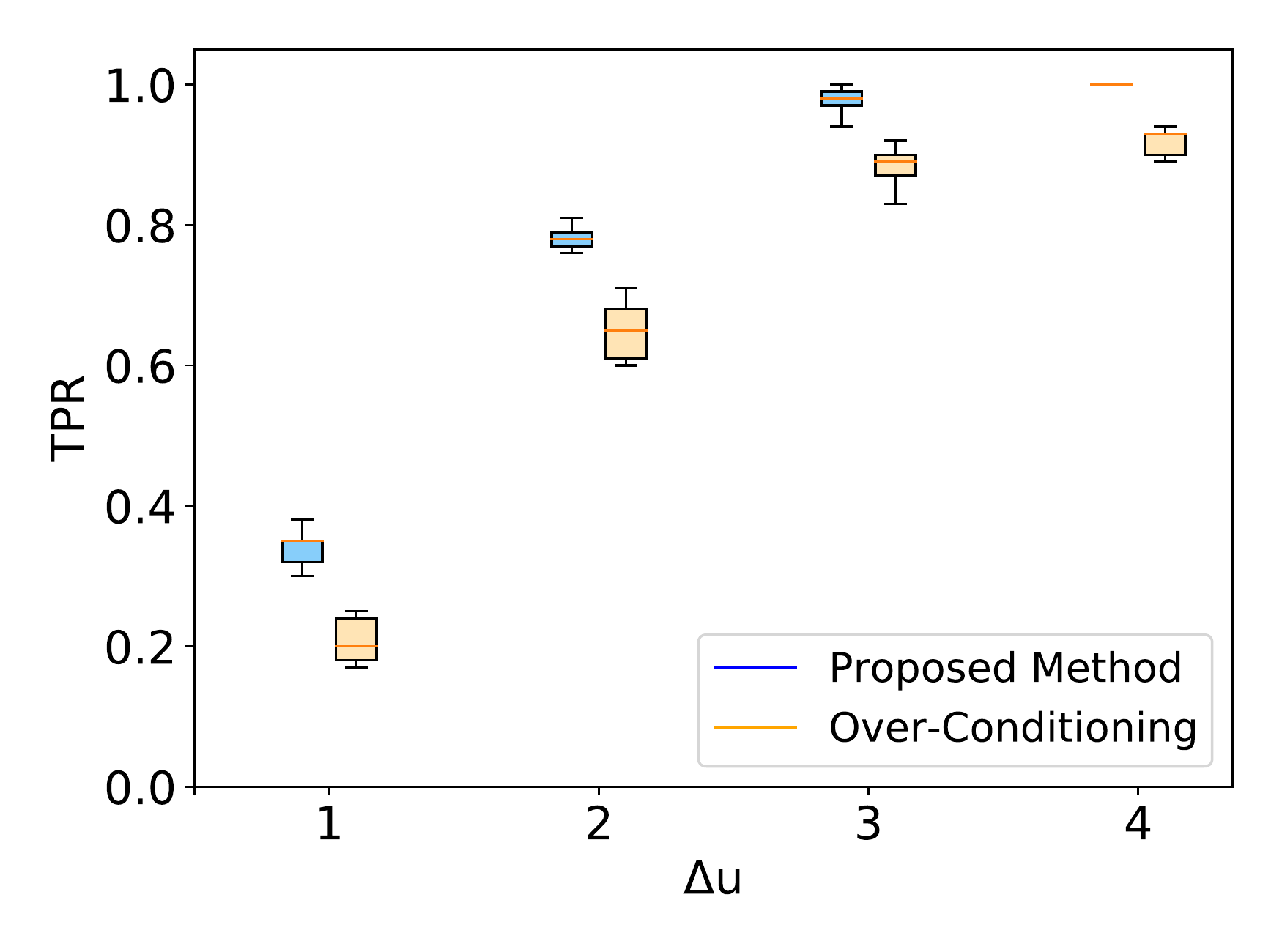}  
%  \vspace{-8pt}
  \caption{TPR}
\end{subfigure}
\caption{Results of FPR and TPR for fused lasso.}
\label{fig:fpr_tpr_fused_lasso}

\vspace*{\floatsep}% 

\begin{subfigure}{.325\textwidth}
  \centering
  \includegraphics[width=\linewidth]{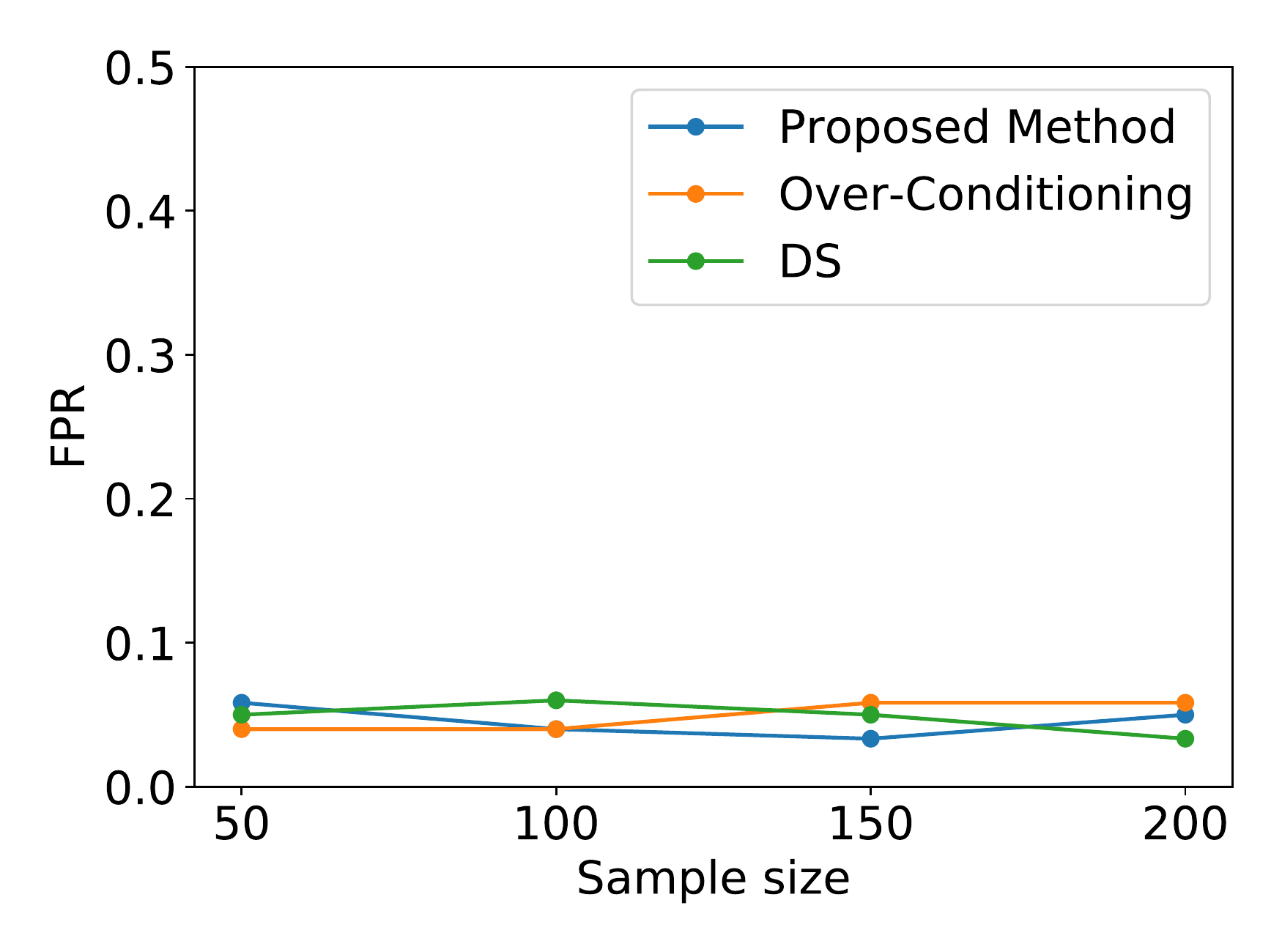}  
%  \vspace{-8pt}
  \caption{FPR}
\end{subfigure}
\begin{subfigure}{.325\textwidth}
  \centering
  \includegraphics[width=\linewidth]{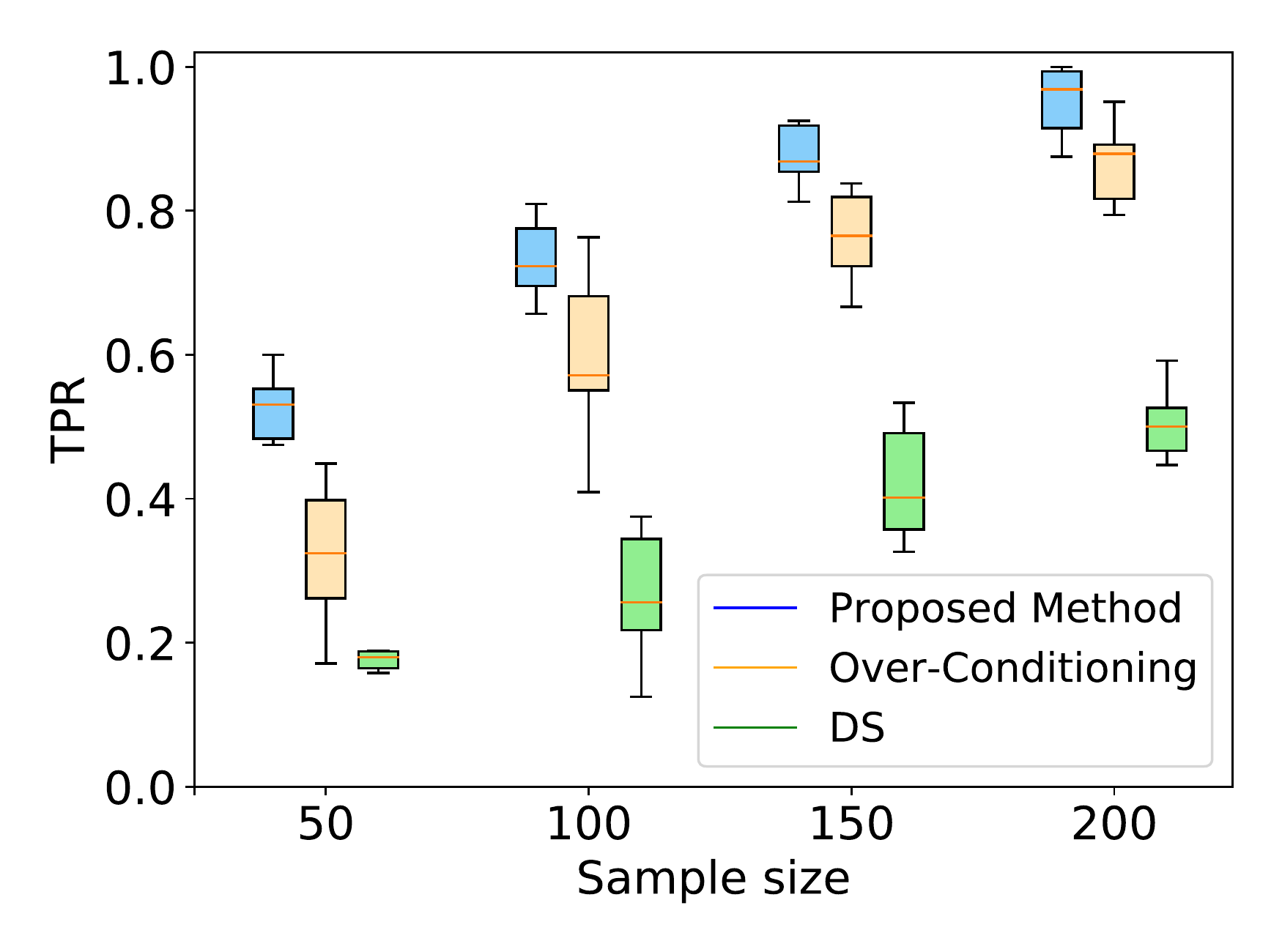}  
%  \vspace{-8pt}
  \caption{TPR}
\end{subfigure}
\begin{subfigure}{.325\textwidth}
  \centering
  \includegraphics[width=\linewidth]{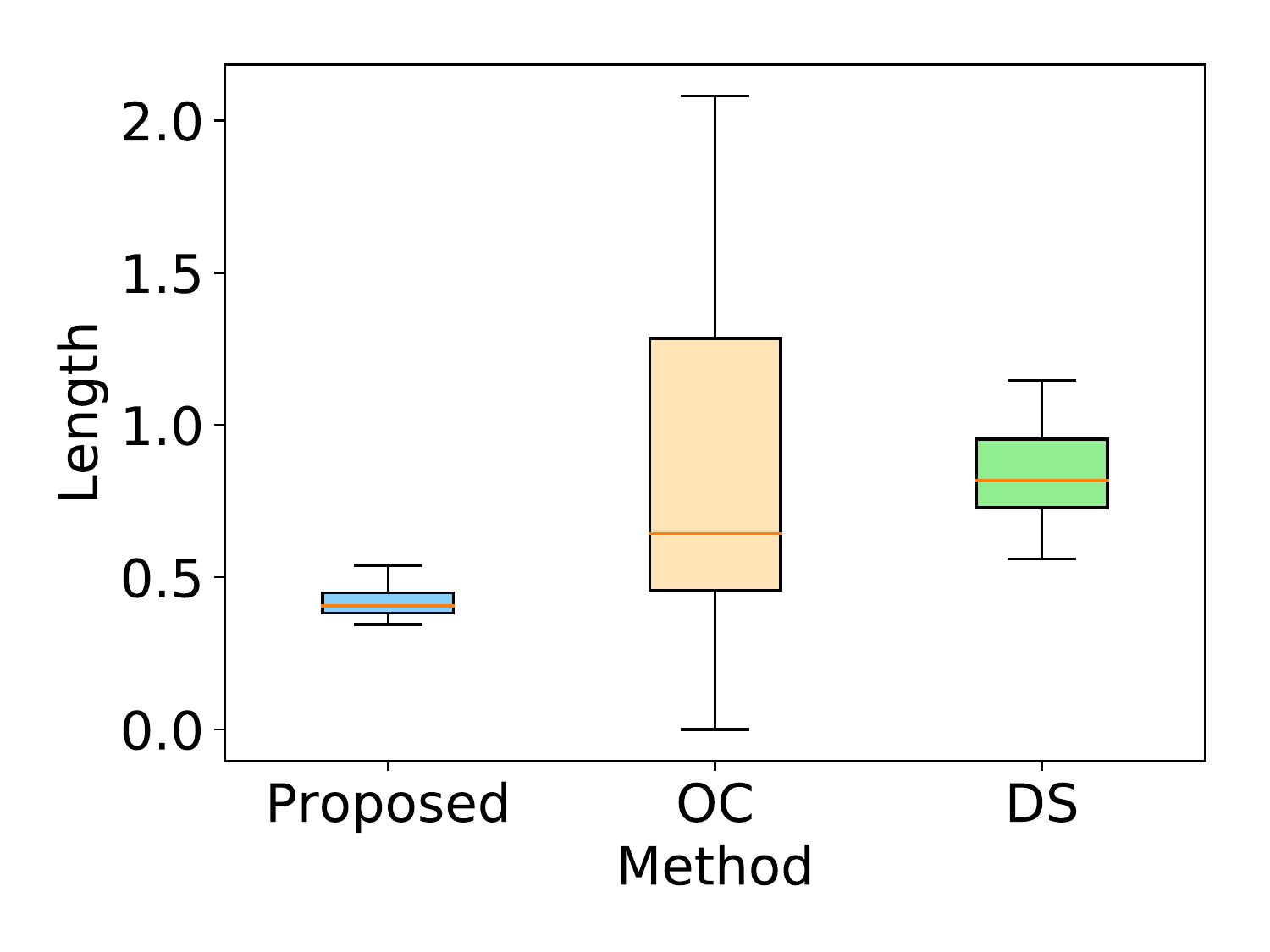}  
%  \vspace{-8pt}
  \caption{CI}
\end{subfigure}
\caption{Results of FPR, TPR, and CI for vanilla lasso.}
\label{fig:fpr_tpr_lasso}

\vspace*{\floatsep}% 

\begin{subfigure}{.325\textwidth}
  \centering
  \includegraphics[width=\linewidth]{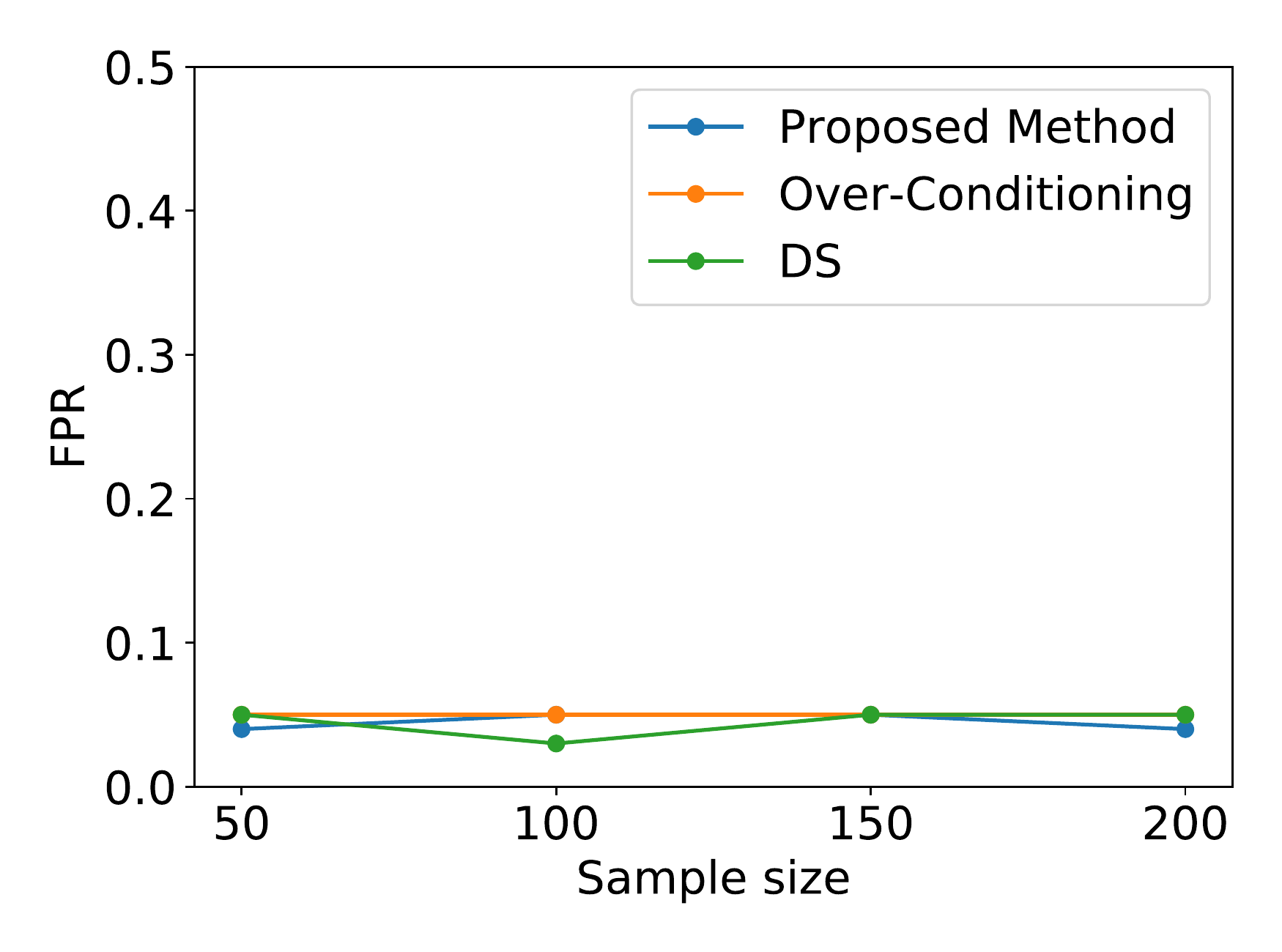}  
%  \vspace{-1cm}
  \caption{FPR}
\end{subfigure}
\begin{subfigure}{.325\textwidth}
  \centering
  \includegraphics[width=\linewidth]{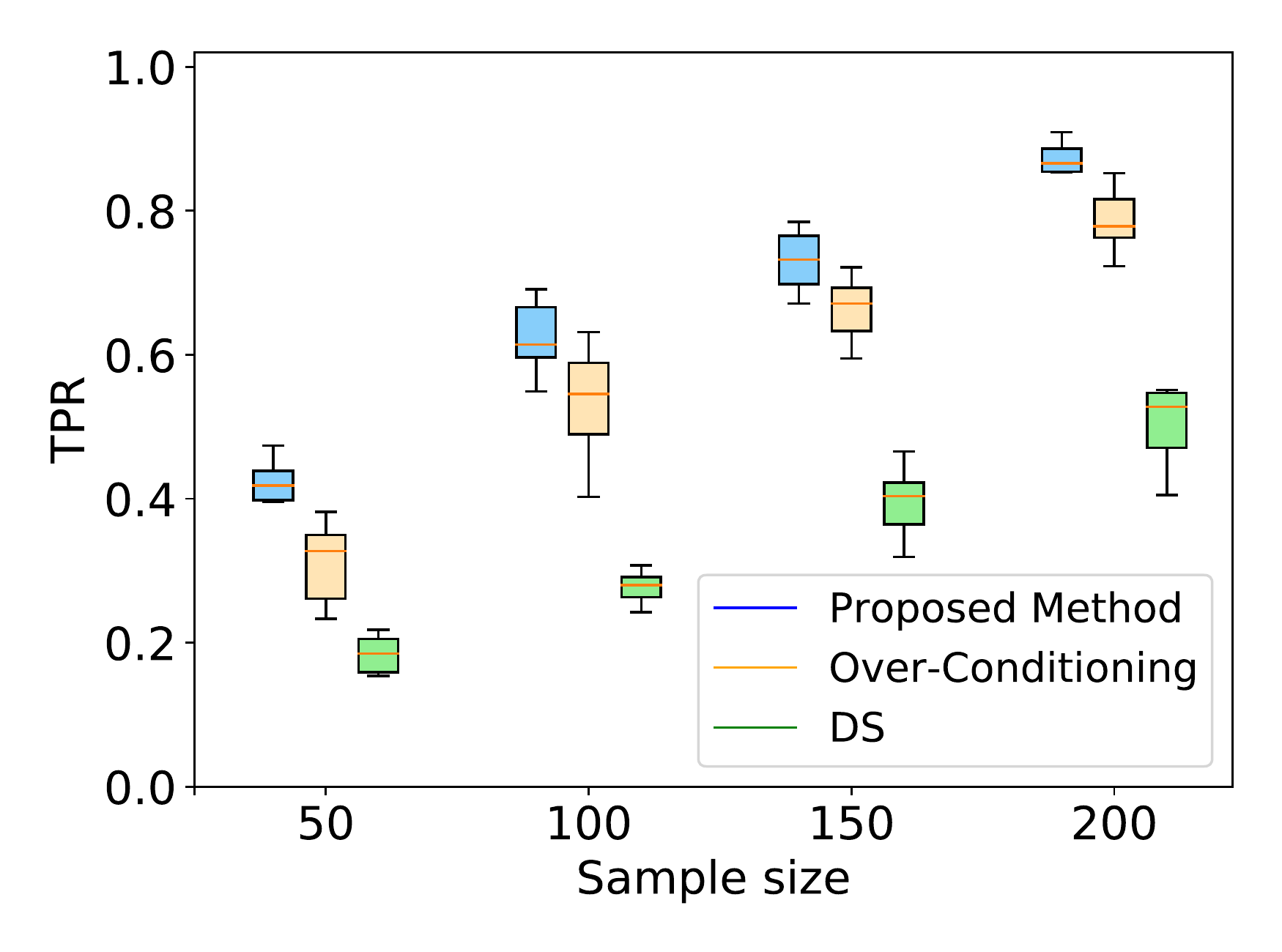}  
%  \vspace{-1cm}
  \caption{TPR}
\end{subfigure}
\begin{subfigure}{.325\textwidth}
  \centering
  \includegraphics[width=\linewidth]{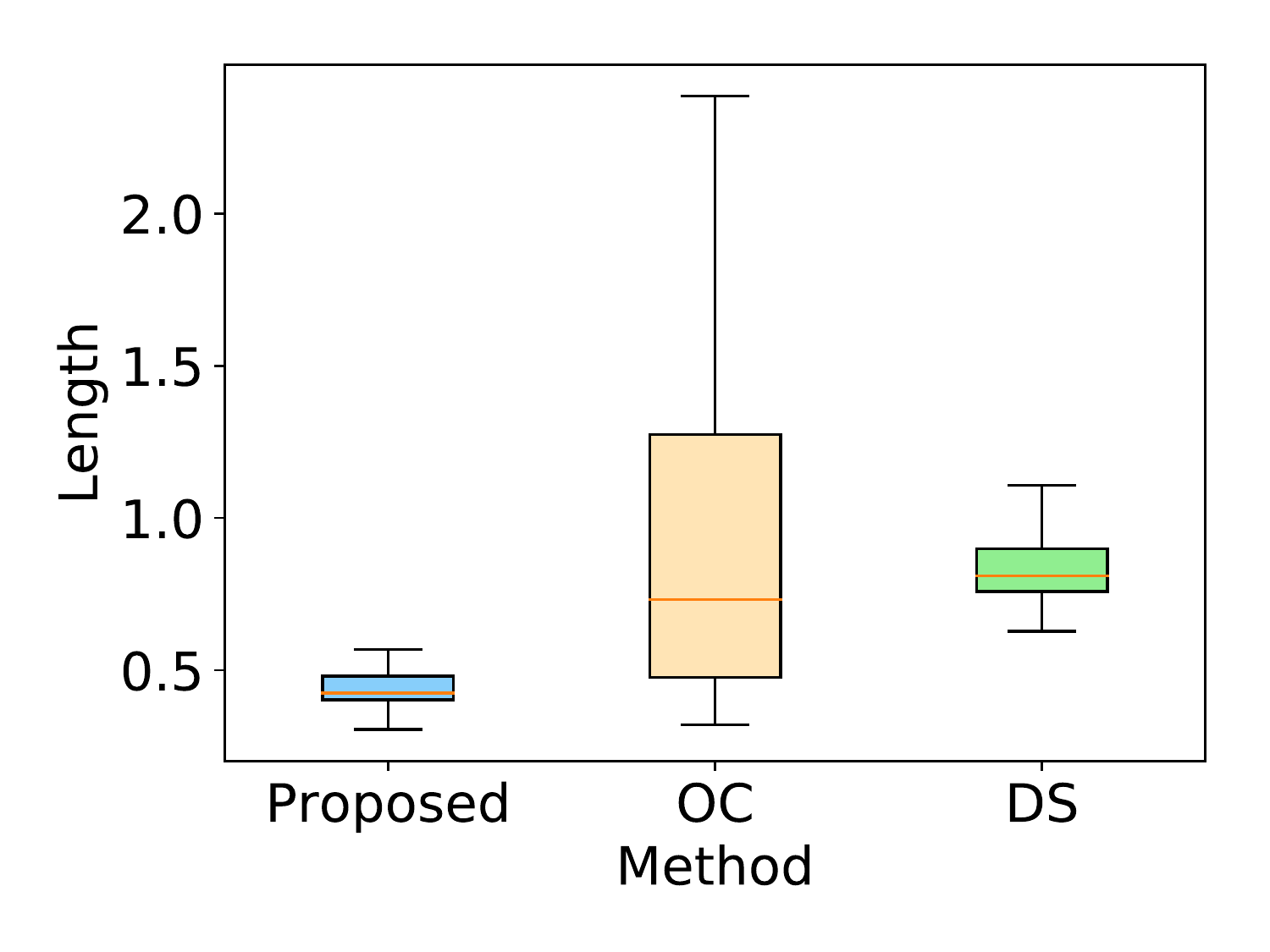}  
%  \vspace{-1cm}
  \caption{CI}
\end{subfigure}
\caption{Results of FPR, TPR, and CI for elastic net.}
\label{fig:fpr_tpr_elastic_net}

\end{figure}

\paragraph{Definition of TPR.} In SI, statistical testing is only conducted when at least one hypothesis is discovered by the algorithm.
Therefore, the definition of the TPR, which is also known as the \emph{conditional power}, is as follows:
\begin{equation*}
	{\rm TPR} = \frac{{\rm \#\ correctly~detected\ \&\ rejected}}{{\rm \#\ correctly~detected}}.
\end{equation*}
In the case of fused lasso, a detection is considered as correct if it is within $L = \pm 2$ of the true CP locations. 
This is because it is often difficult to identify the exact CPs accurately in the presence of noise.
Many existing studies considered a detection to be correct if it was within $L$ positions of the true CP locations \citep{truong2020selective}.
In the case of feature selection, ${\rm \#\ correctly~detected}$ indicates the number of truly positive features that are selected by the algorithm (e.g., lasso), whereas  ${\rm \#\ rejected}$ indicates the number of truly positive features for which the null hypothesis is rejected by the SI.

\subsection{Numerical Results} \label{subsec:exp_numerical_results}

\begin{figure}[!t]

%\begin{subfigure}{.325\textwidth}
%  \centering
%  \includegraphics[width=\linewidth]{fpr_elastic_net}  
%%  \vspace{-1cm}
%  \caption{FPR}
%\end{subfigure}
%\begin{subfigure}{.325\textwidth}
%  \centering
%  \includegraphics[width=\linewidth]{tpr_elastic_net}  
%%  \vspace{-1cm}
%  \caption{TPR}
%\end{subfigure}
%\begin{subfigure}{.325\textwidth}
%  \centering
%  \includegraphics[width=\linewidth]{ci_elastic_net}  
%%  \vspace{-1cm}
%  \caption{CI}
%\end{subfigure}
%\caption{Results of FPR, TPR and CI for elastic net.}
%\label{fig:fpr_tpr_elastic_net}
%
%\vspace*{\floatsep}% 

\begin{subfigure}{.325\textwidth}
  \centering
  \includegraphics[width=\linewidth]{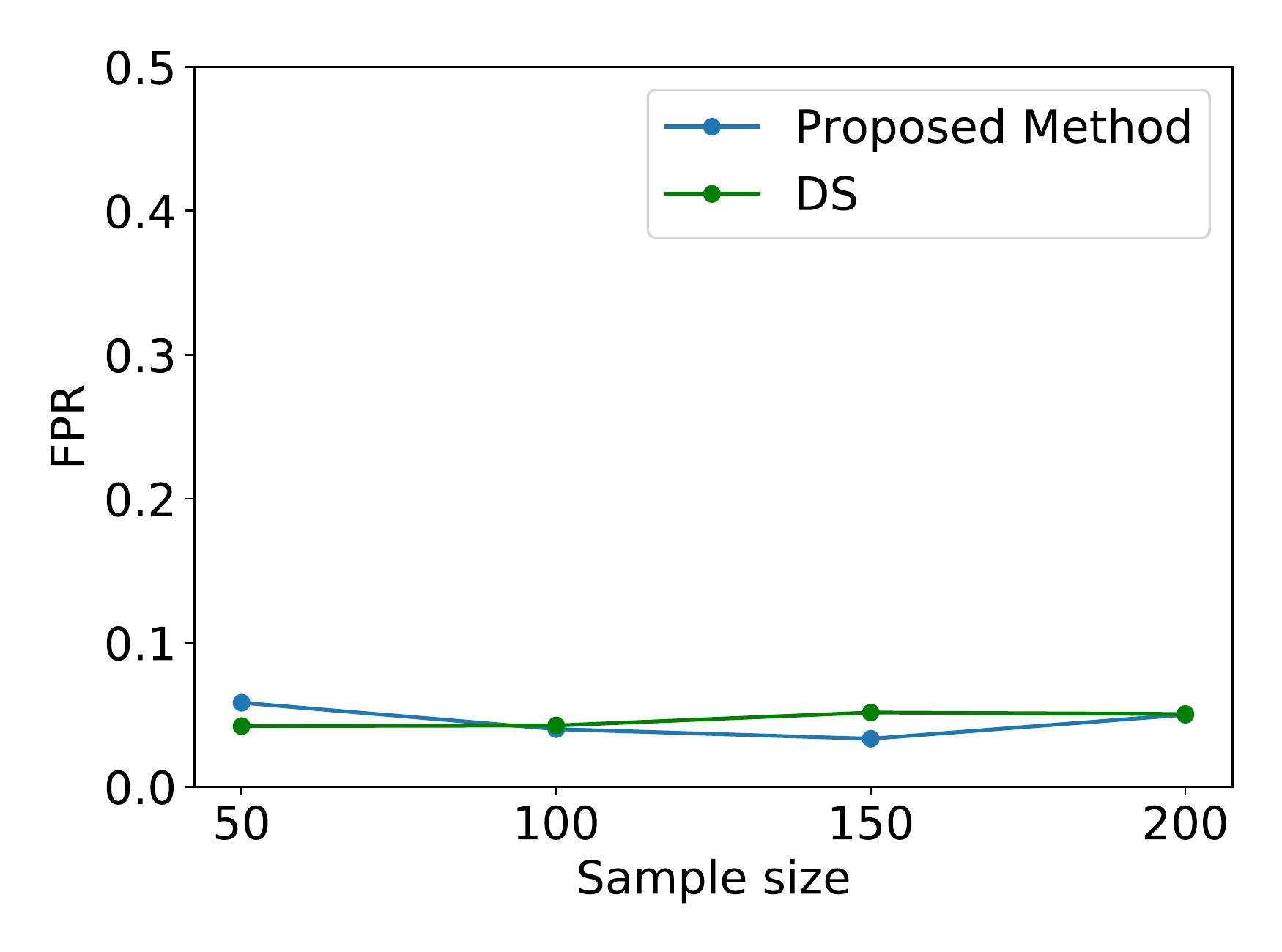}  
%  \vspace{-1cm}
  \caption{FPR}
\end{subfigure}
\begin{subfigure}{.325\textwidth}
  \centering
  \includegraphics[width=\linewidth]{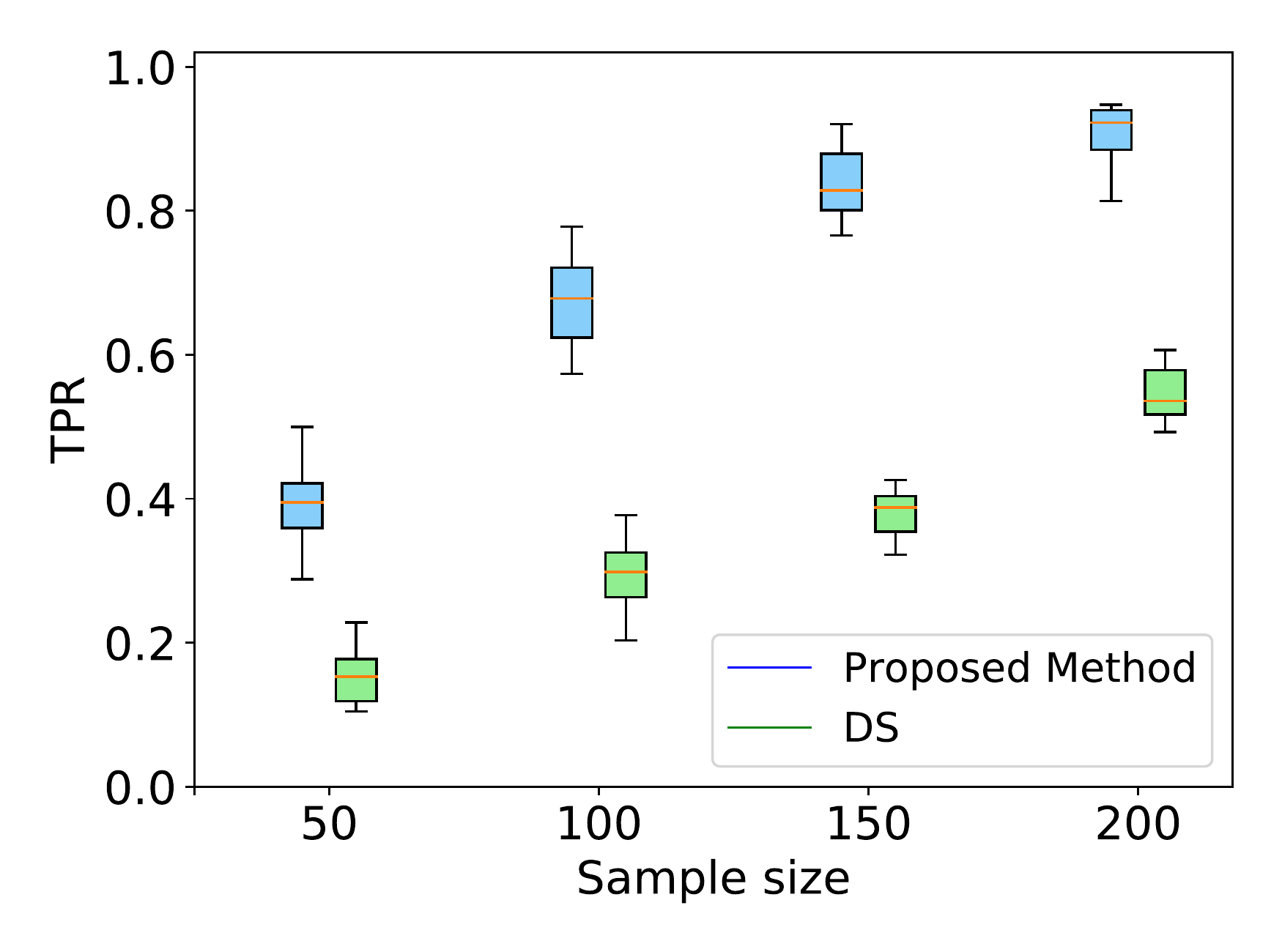}  
%  \vspace{-1cm}
  \caption{TPR}
\end{subfigure}
\begin{subfigure}{.325\textwidth}
  \centering
  \includegraphics[width=\linewidth]{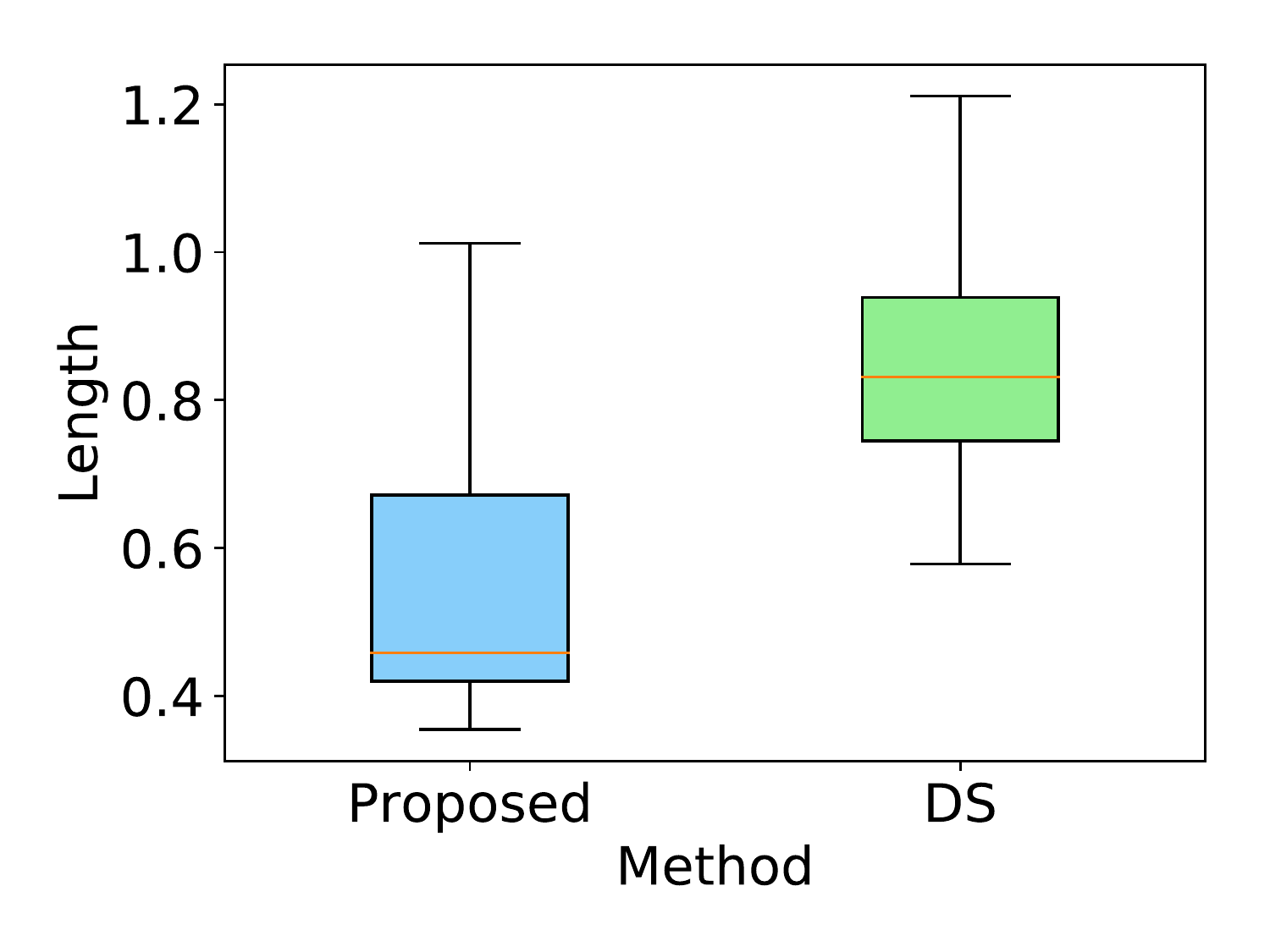}  
%   \vspace{-1cm}
  \caption{CI}
\end{subfigure}
\caption{Results of FPR, TPR, and CI for non-negative least squares.}
\label{fig:fpr_tpr_non_negative}

\end{figure}

\paragraph{FPR, TPR, and CI results.} 
The fused lasso results are presented in Figure \ref{fig:fpr_tpr_fused_lasso}. The results for the vanilla lasso, elastic net, non-negative least squares, and Huber regression + $\ell_1$ norm are depicted in Figures \ref{fig:fpr_tpr_lasso}, \ref{fig:fpr_tpr_elastic_net}, \ref{fig:fpr_tpr_non_negative}, and \ref{fig:fpr_tpr_huber_l1}, respectively.
No over-conditioning occurred in the case of the non-negative least squares as we had already restricted the coefficients to be positive.
In summary, although all of the methods could properly control the FPR at a significance level of $\alpha$, 
the proposed method had the highest power among the methods.
The CI results were also consistent with the TPR results.
That is, the shortest CI for the proposed method indicated that it exhibited the highest power.

\begin{figure}[t!]

\begin{subfigure}{.325\textwidth}
  \centering
  \includegraphics[width=\linewidth]{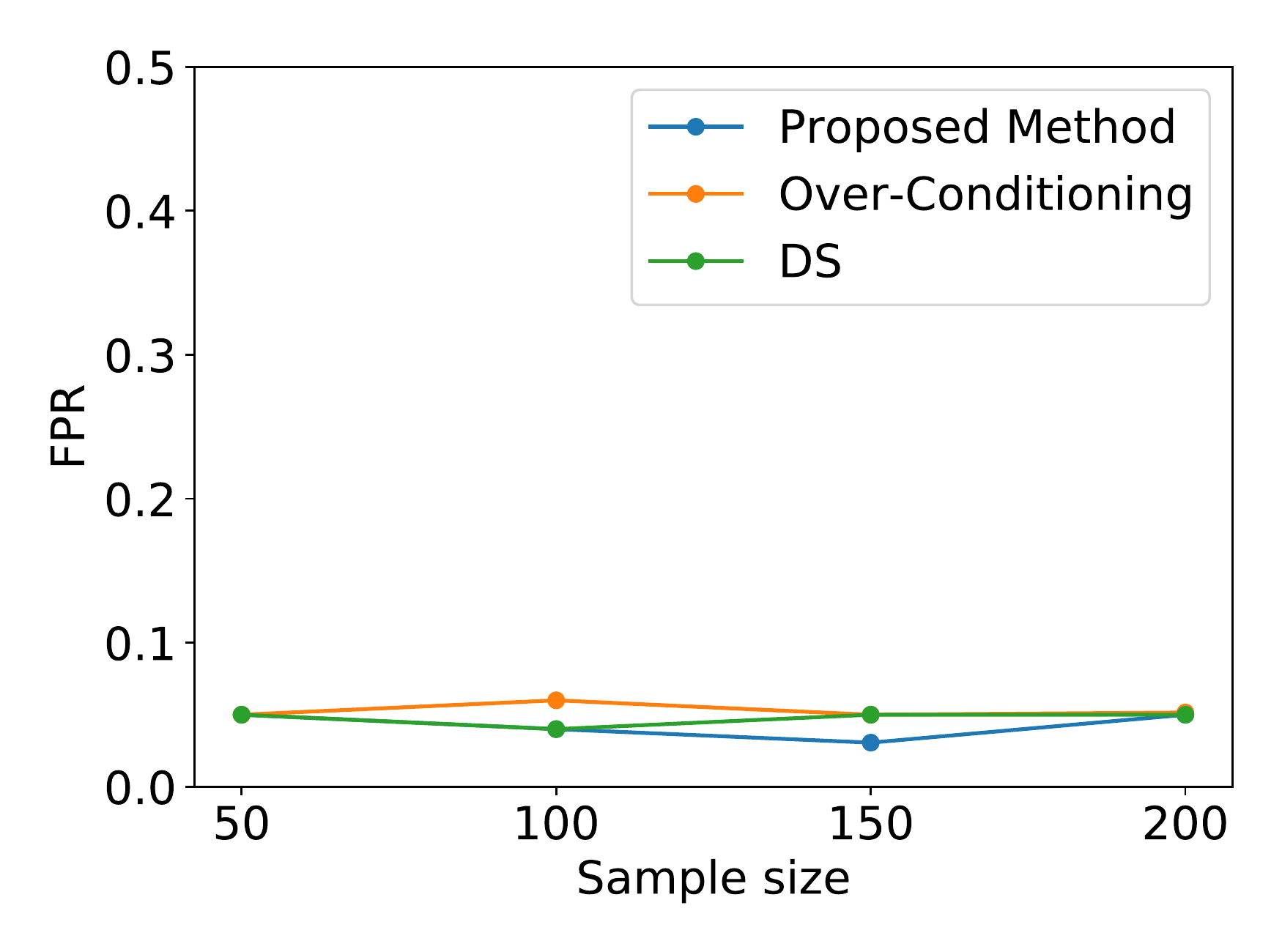}  
%  \vspace{-1cm}
  \caption{FPR}
\end{subfigure}
\begin{subfigure}{.325\textwidth}
  \centering
  \includegraphics[width=\linewidth]{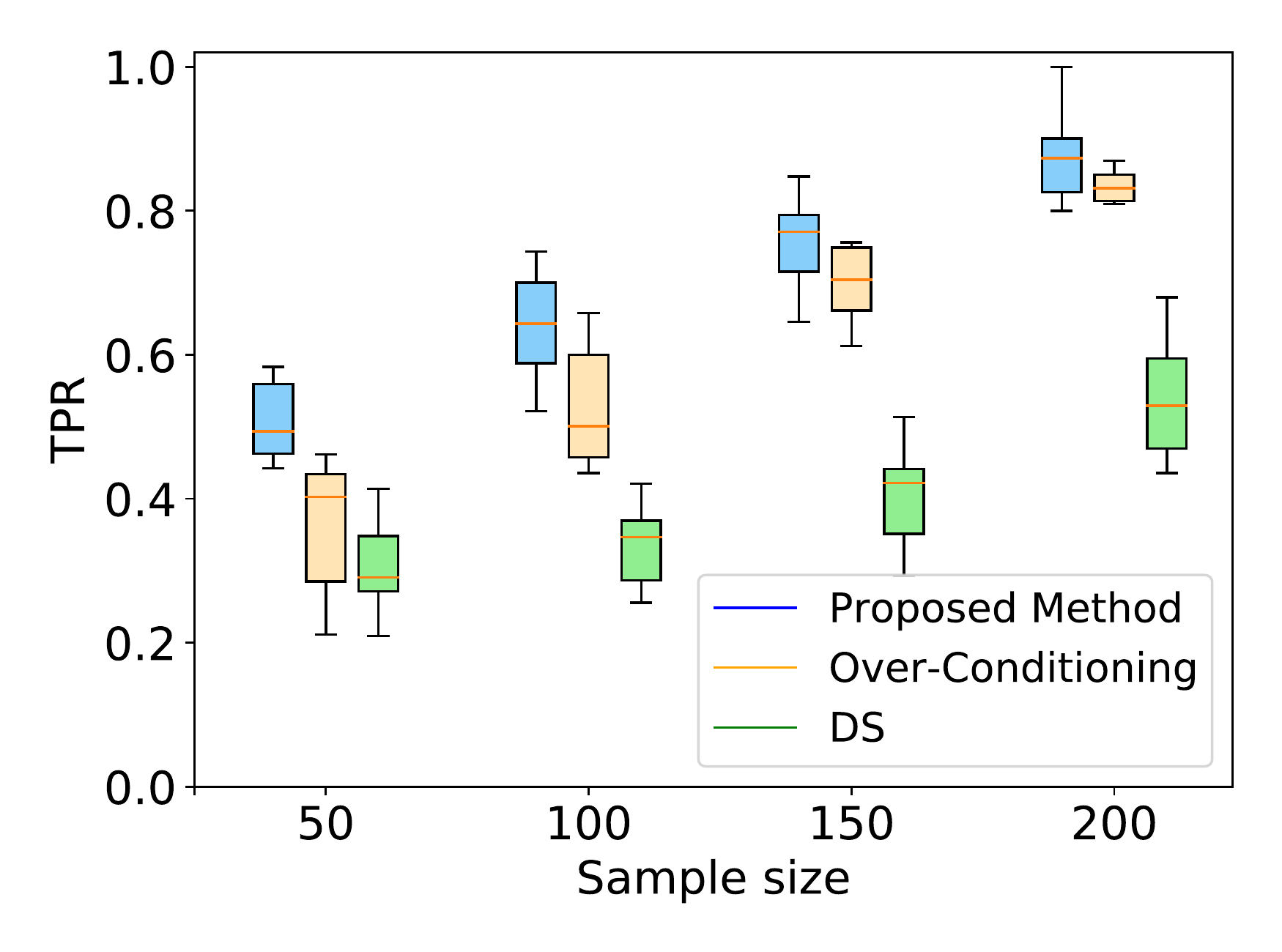}  
%  \vspace{-1cm}
  \caption{TPR}
\end{subfigure}
\begin{subfigure}{.325\textwidth}
  \centering
  \includegraphics[width=\linewidth]{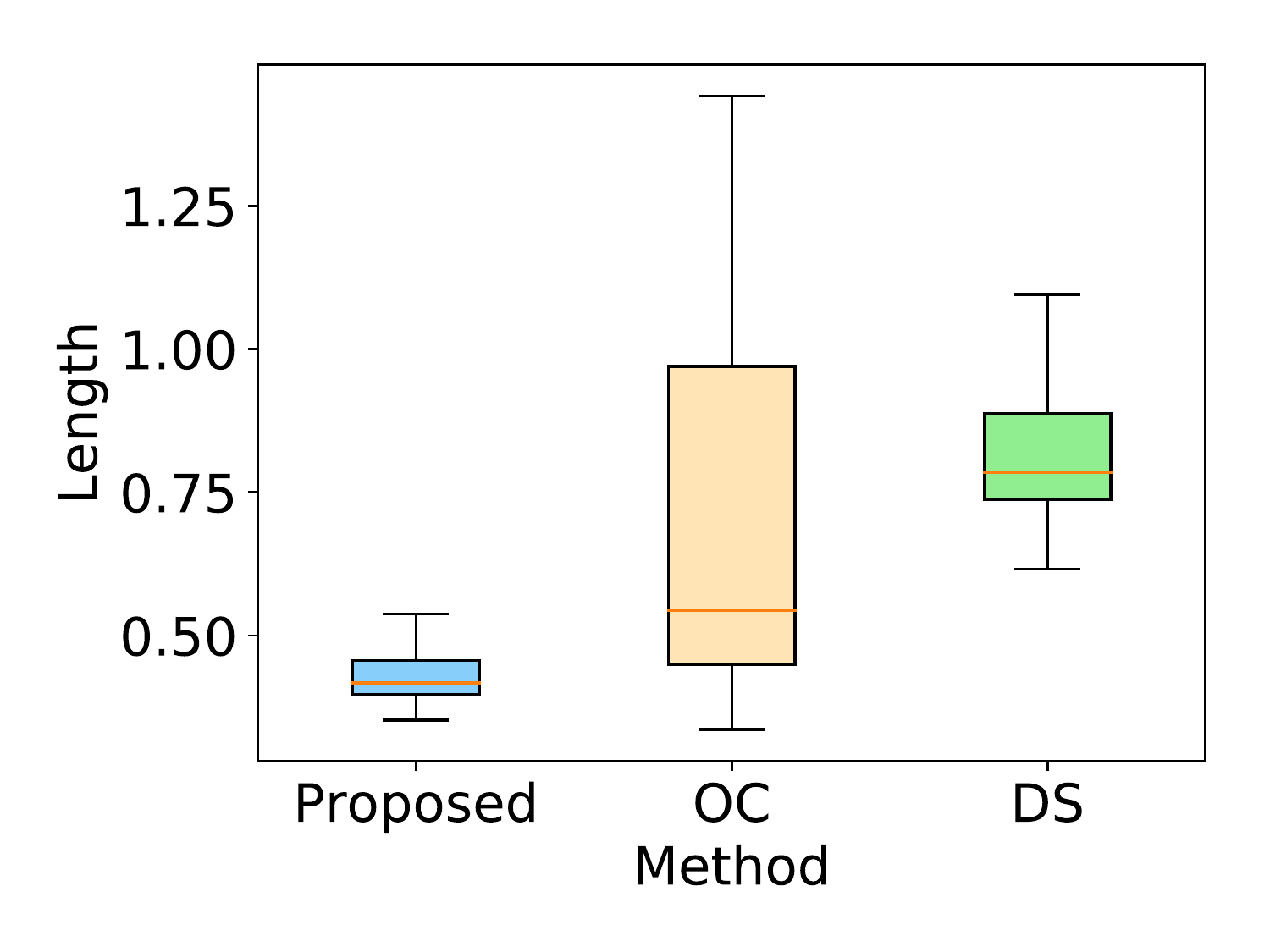}  
%  \vspace{-1cm}
  \caption{CI}
\end{subfigure}
\caption{ FPR, TPR, and CI results for Huber regression with $\ell_1$ penalty.}
\label{fig:fpr_tpr_huber_l1}

\end{figure}

%===

% ==============================

\begin{figure}[!t]
\begin{subfigure}{.245\textwidth}
  \centering
  \includegraphics[width=\linewidth]{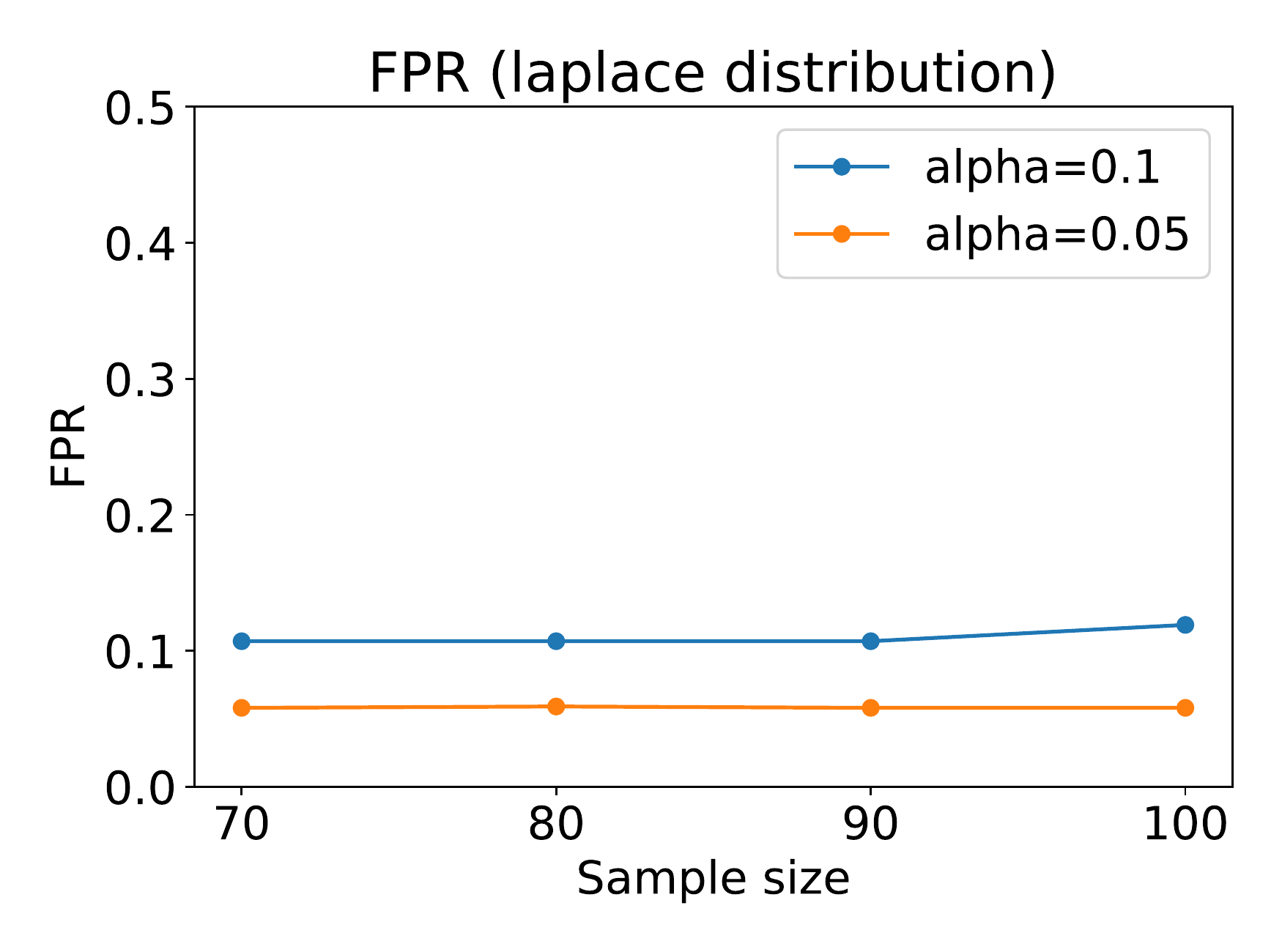}
\end{subfigure}
\begin{subfigure}{.245\textwidth}
  \centering
  \includegraphics[width=\linewidth]{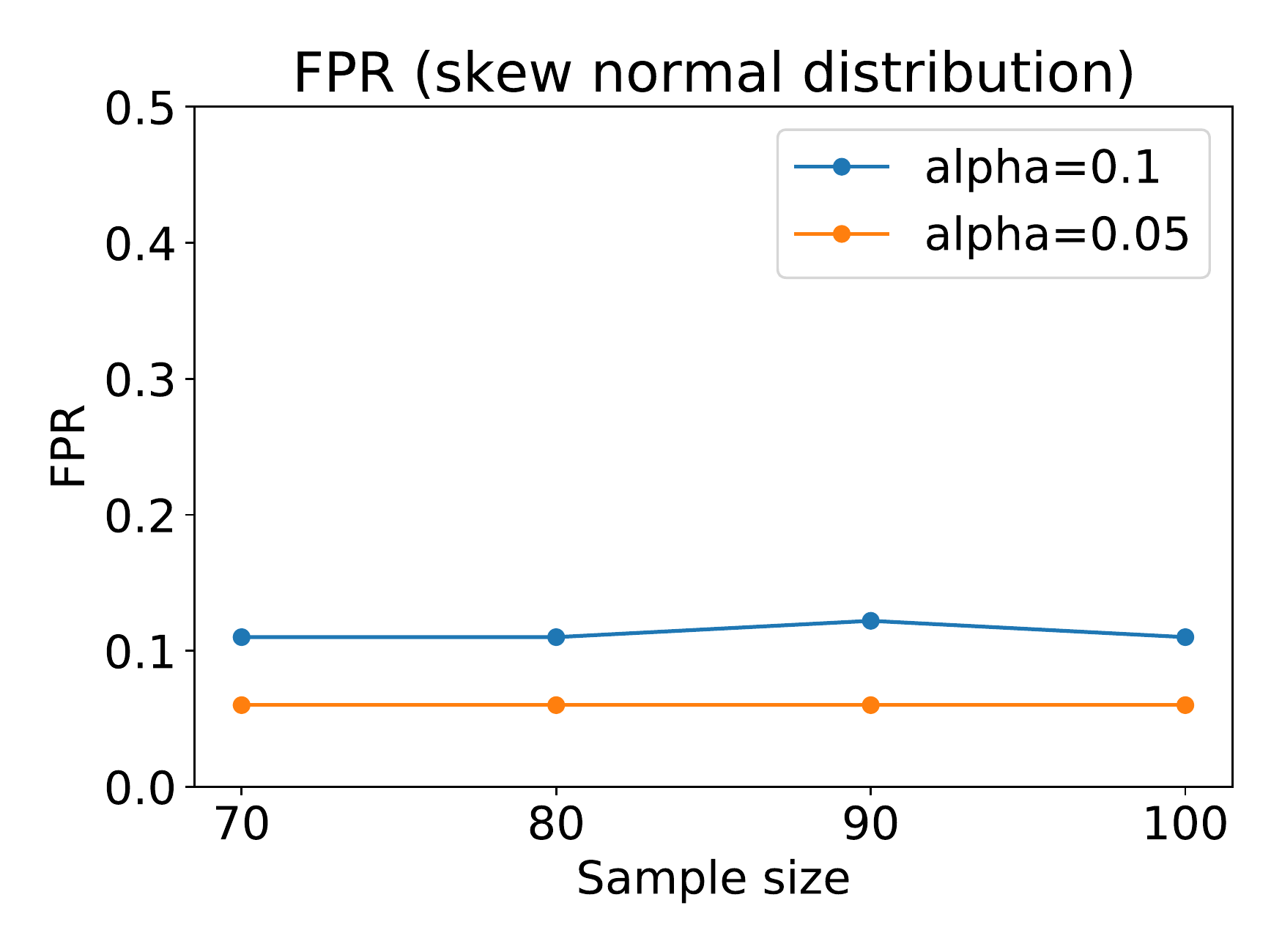}
\end{subfigure}
\begin{subfigure}{.245\textwidth}
  \centering
  \includegraphics[width=\linewidth]{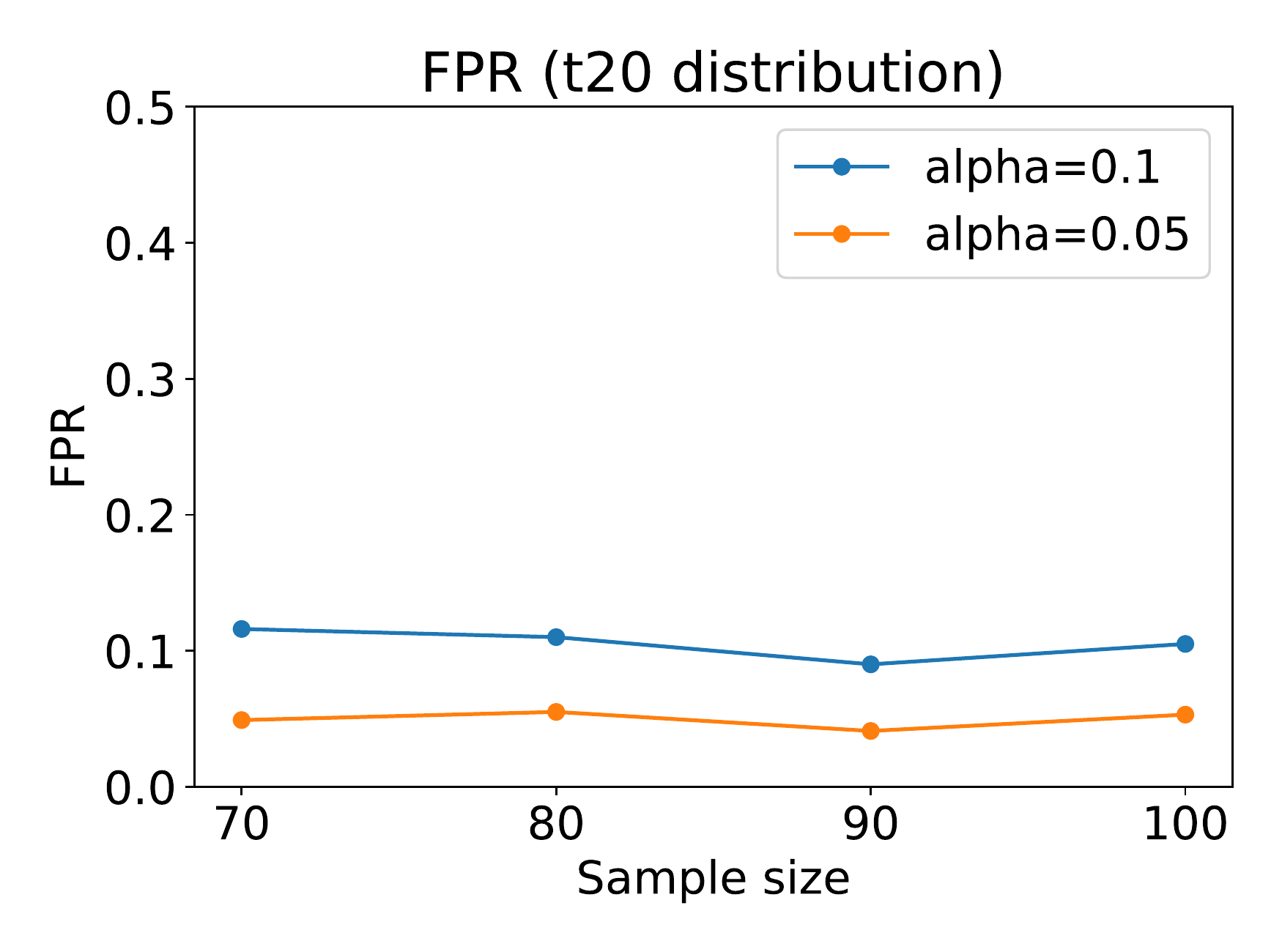}
\end{subfigure}
\begin{subfigure}{.245\textwidth}
  \centering
  \includegraphics[width=\linewidth]{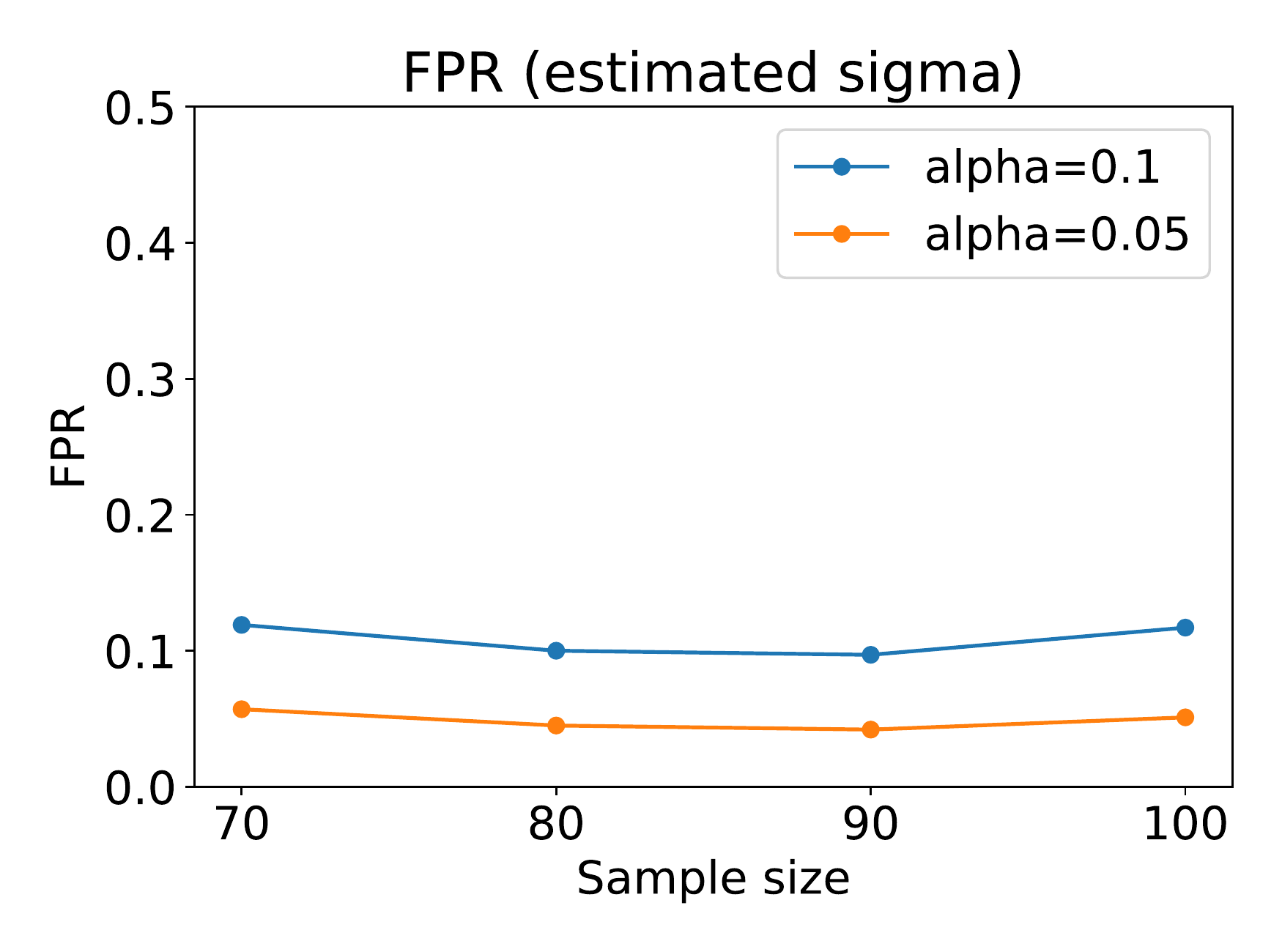}
\end{subfigure}
\caption{Robustness of proposed method for fused lasso.}
\label{fig:robust_fused_lasso}

\vspace*{\floatsep}% 

\begin{subfigure}{.245\textwidth}
  \centering
  \includegraphics[width=\linewidth]{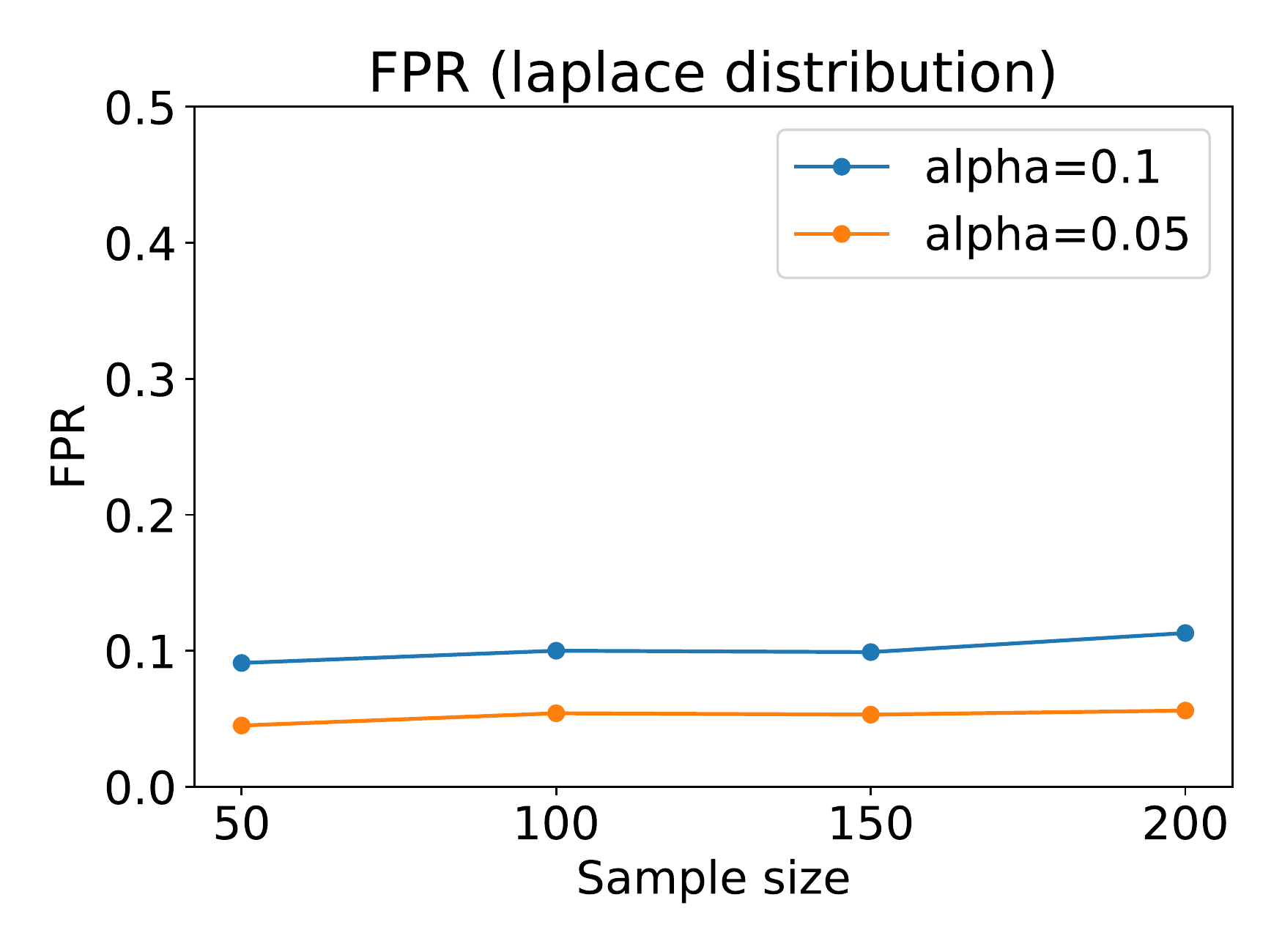}
\end{subfigure}
\begin{subfigure}{.245\textwidth}
  \centering
  \includegraphics[width=\linewidth]{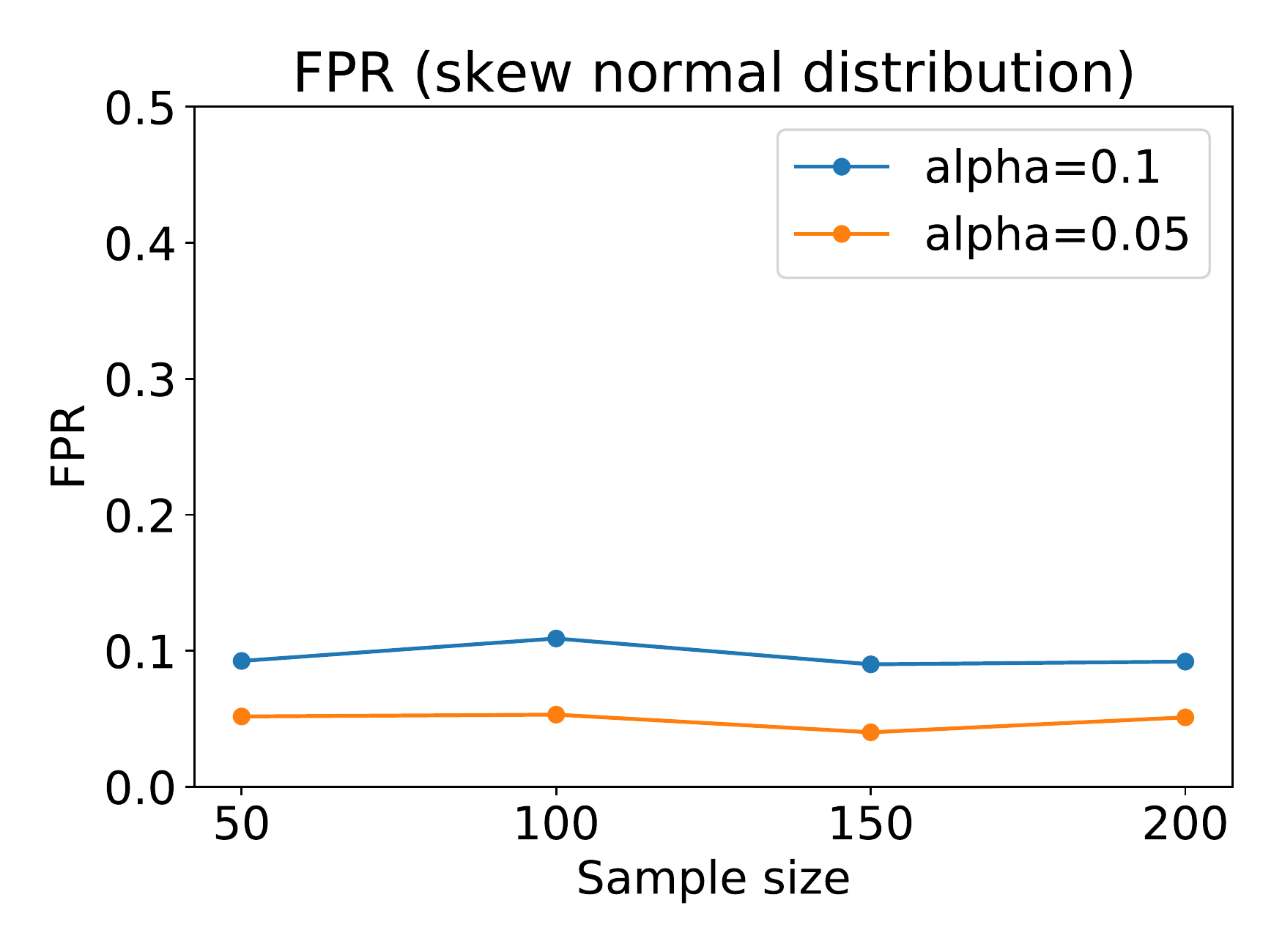}
\end{subfigure}
\begin{subfigure}{.245\textwidth}
  \centering
  \includegraphics[width=\linewidth]{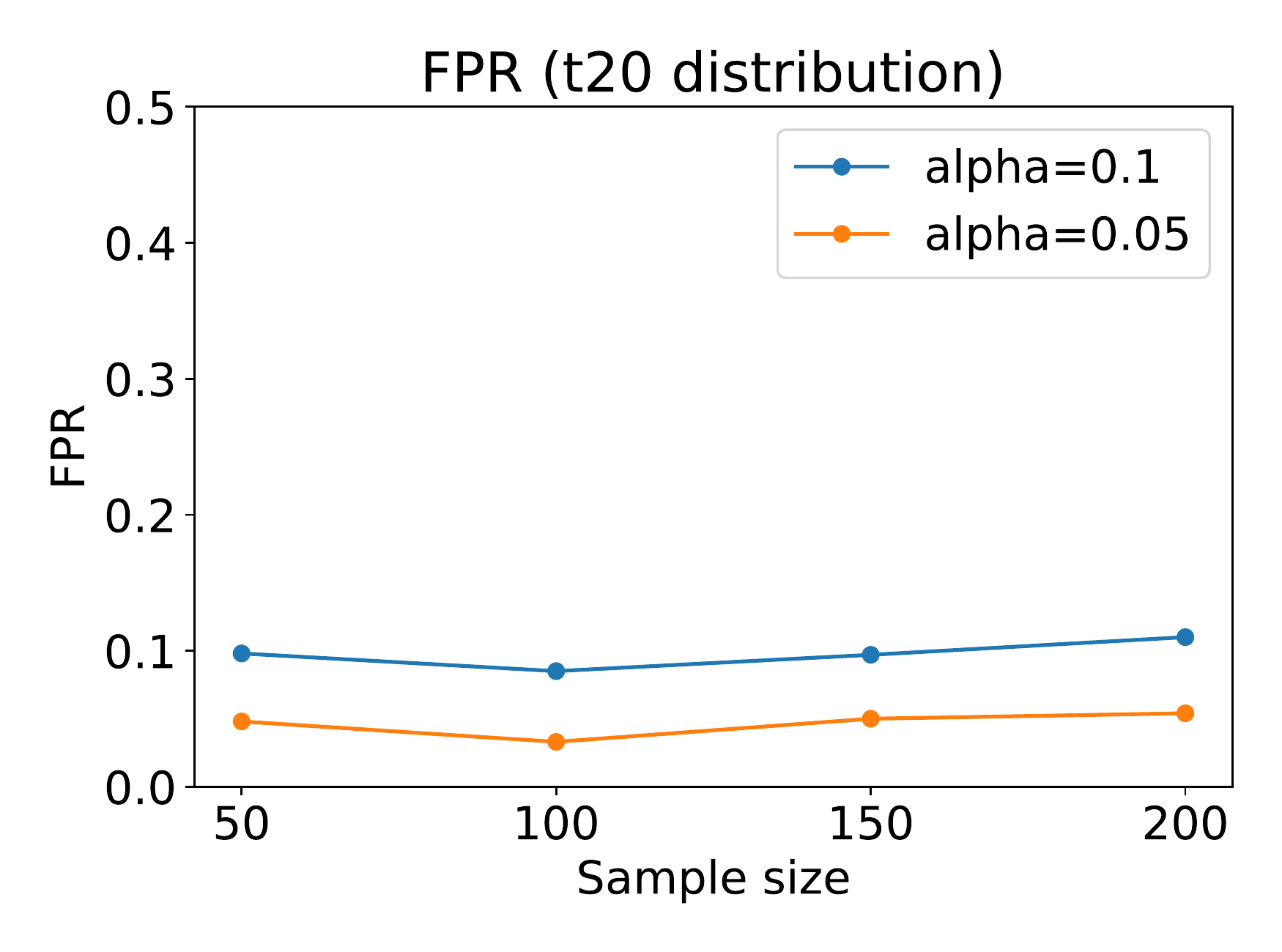}
\end{subfigure}
\begin{subfigure}{.245\textwidth}
  \centering
  \includegraphics[width=\linewidth]{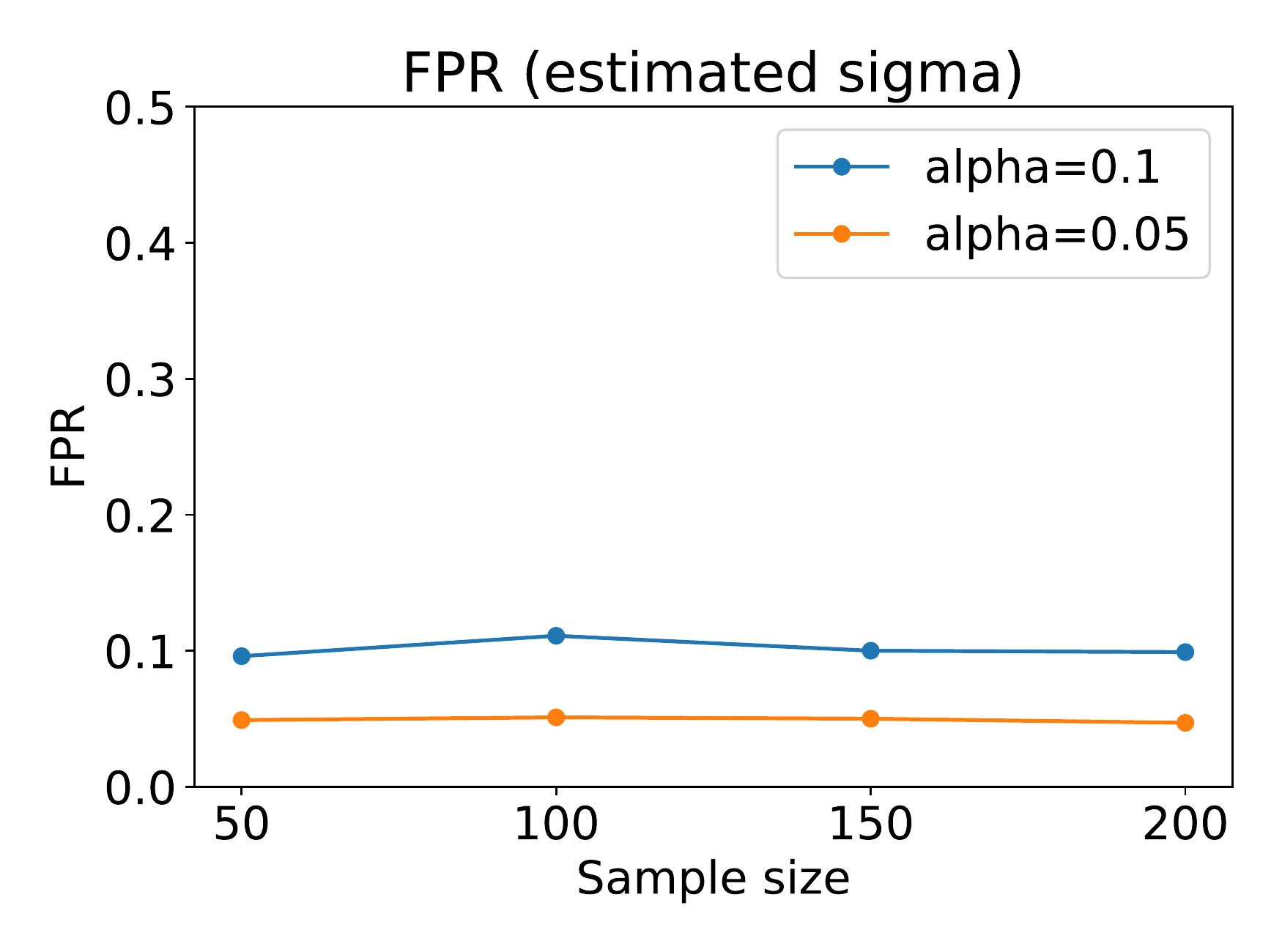}
\end{subfigure}
\caption{Robustness of proposed method for vanilla lasso.}
\label{fig:robust_lasso}

\vspace*{\floatsep}% 

\begin{subfigure}{.245\textwidth}
  \centering
  \includegraphics[width=\linewidth]{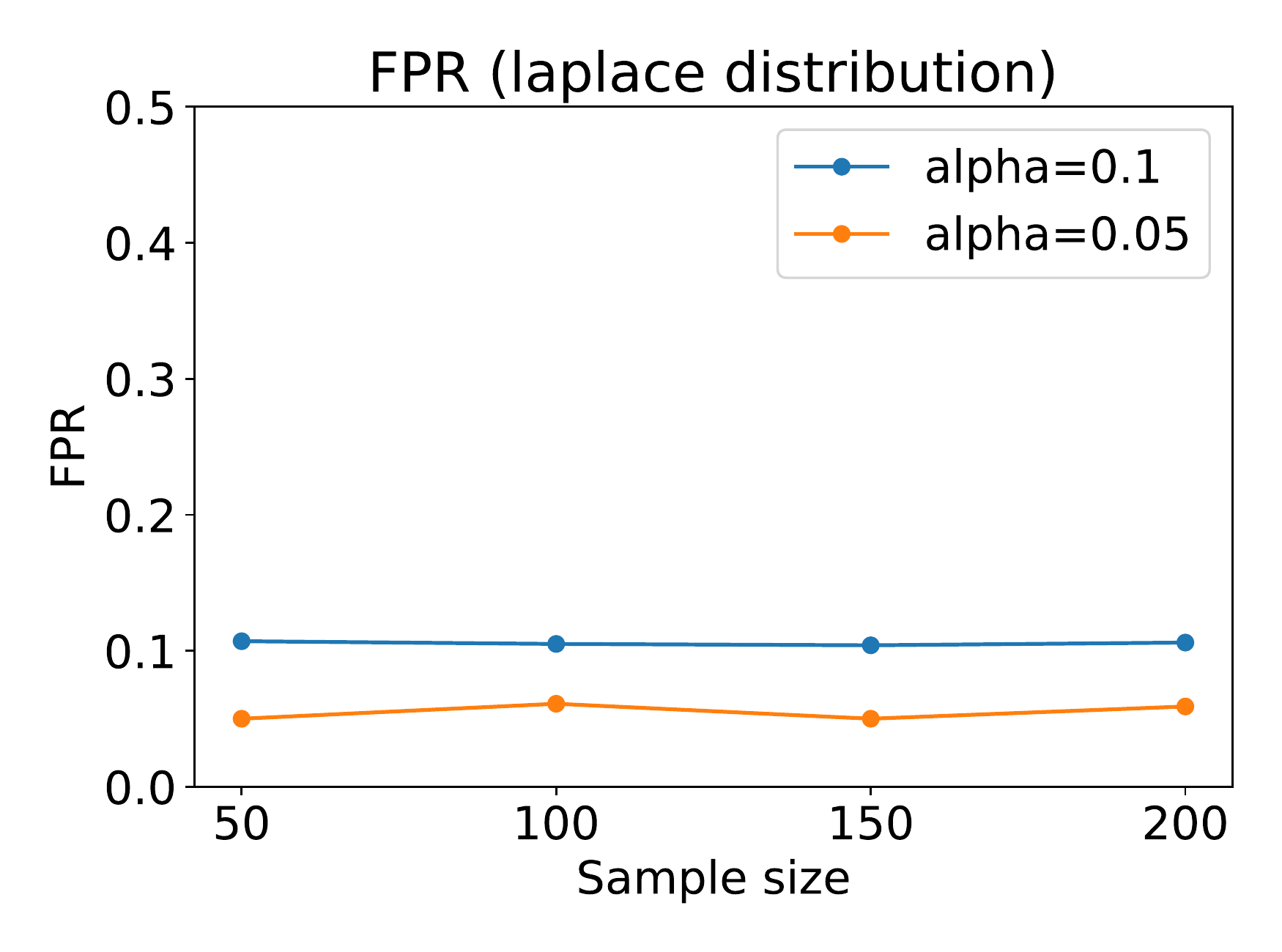}
\end{subfigure}
\begin{subfigure}{.245\textwidth}
  \centering
  \includegraphics[width=\linewidth]{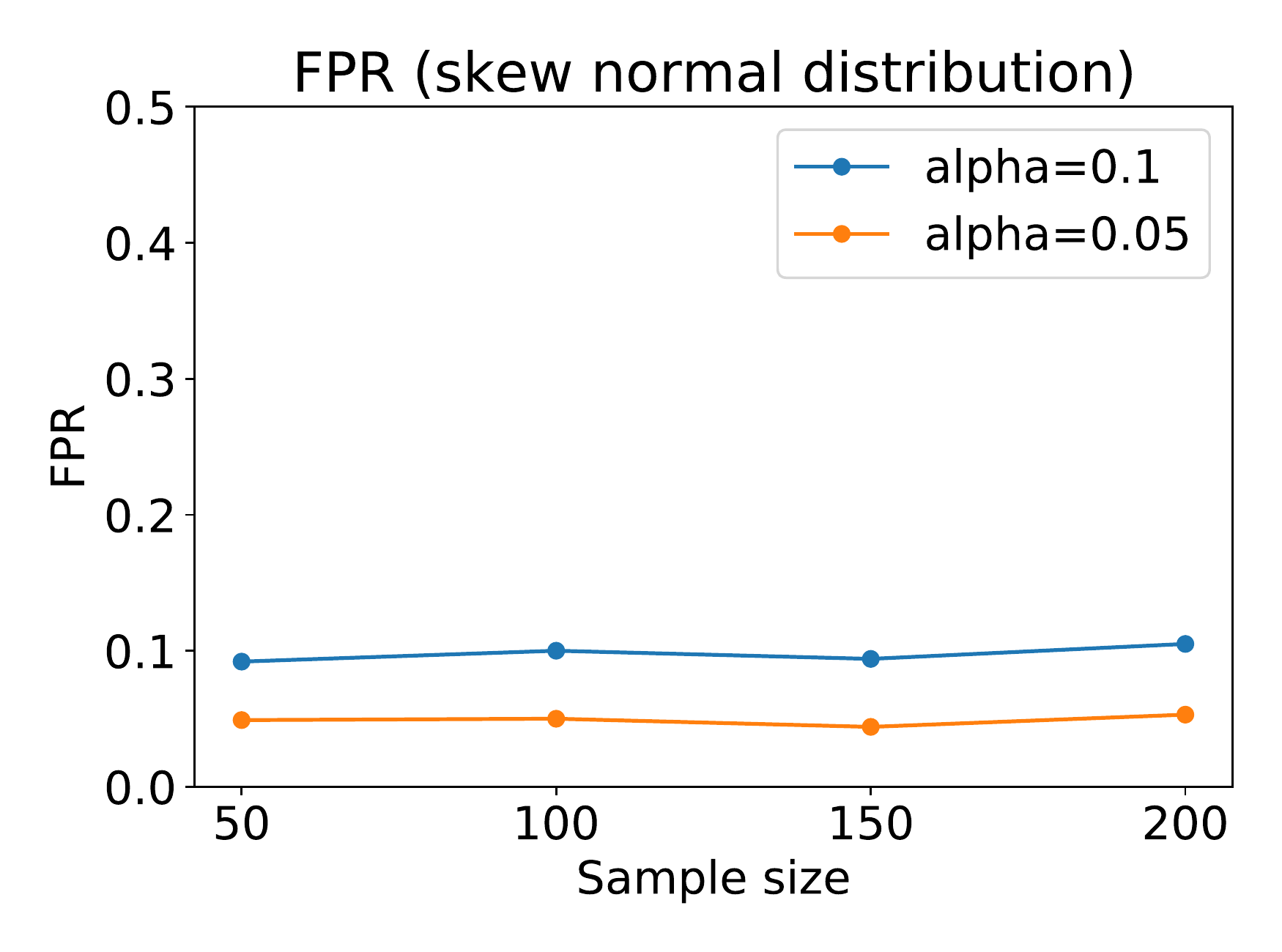}
\end{subfigure}
\begin{subfigure}{.245\textwidth}
  \centering
  \includegraphics[width=\linewidth]{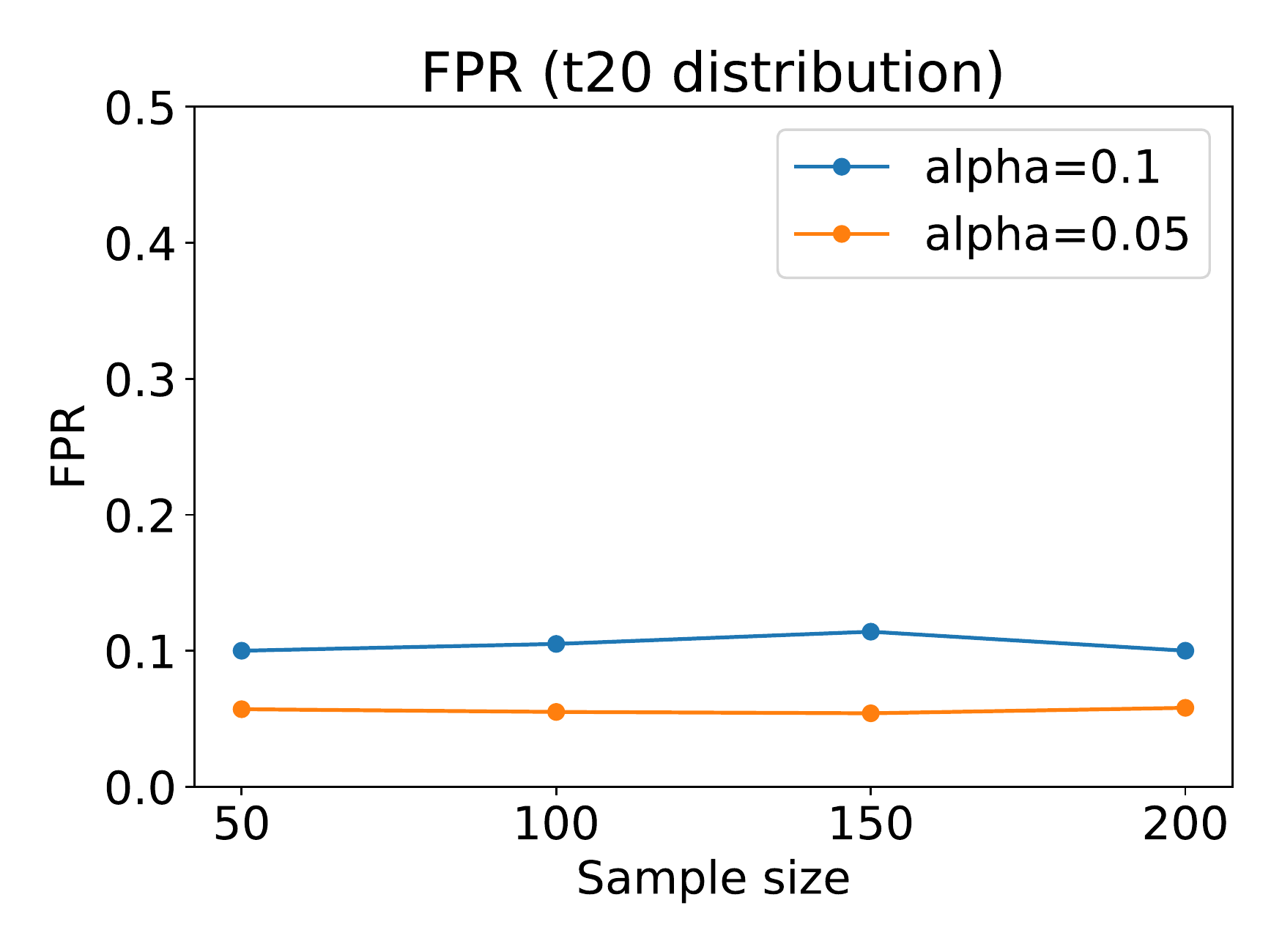}
\end{subfigure}
\begin{subfigure}{.245\textwidth}
  \centering
  \includegraphics[width=\linewidth]{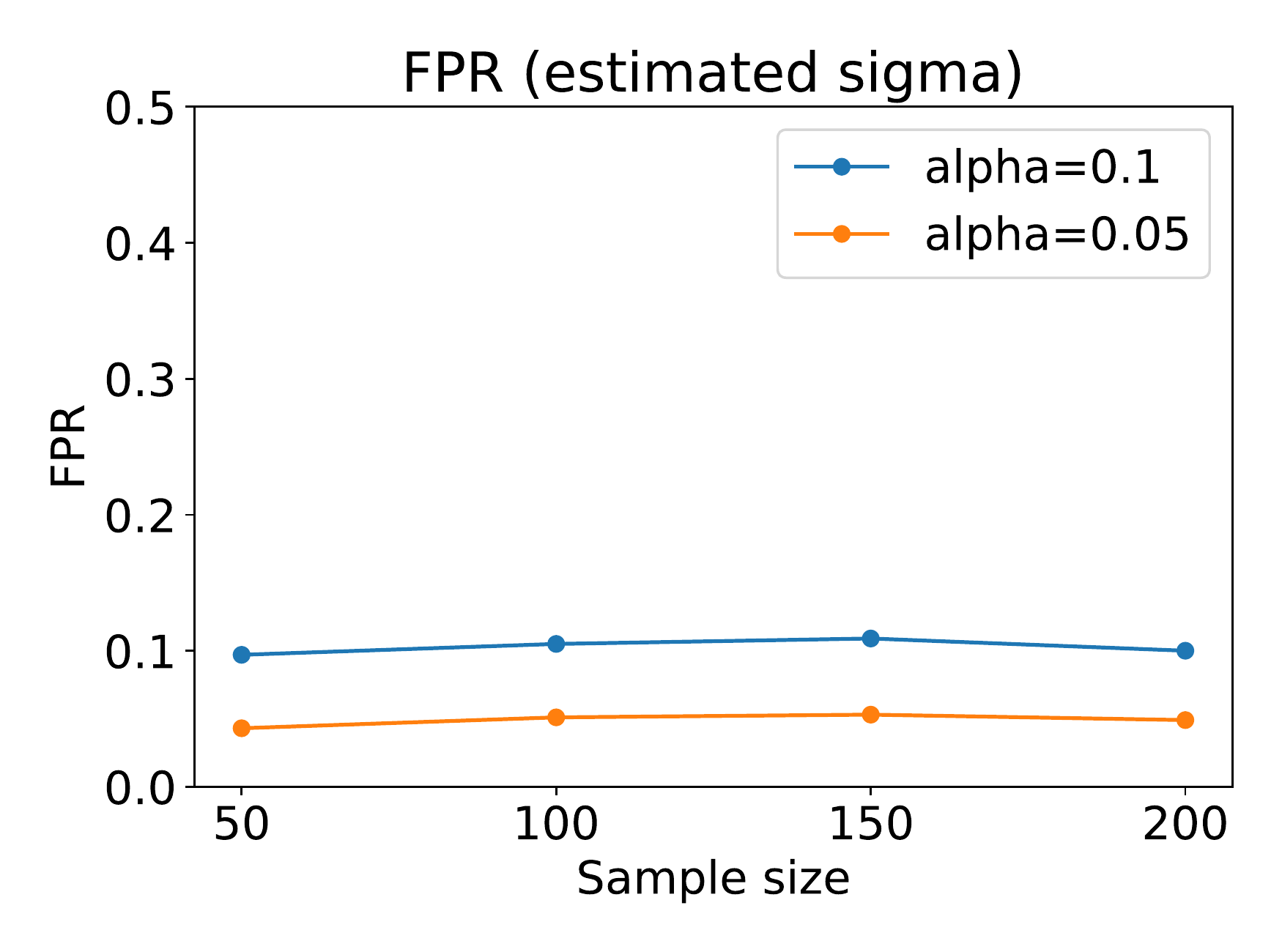}
\end{subfigure}
\caption{Robustness of proposed method for elastic net.}
\label{fig:robust_elastic_net}

\vspace*{\floatsep}% 

\begin{subfigure}{.245\textwidth}
  \centering
  \includegraphics[width=\linewidth]{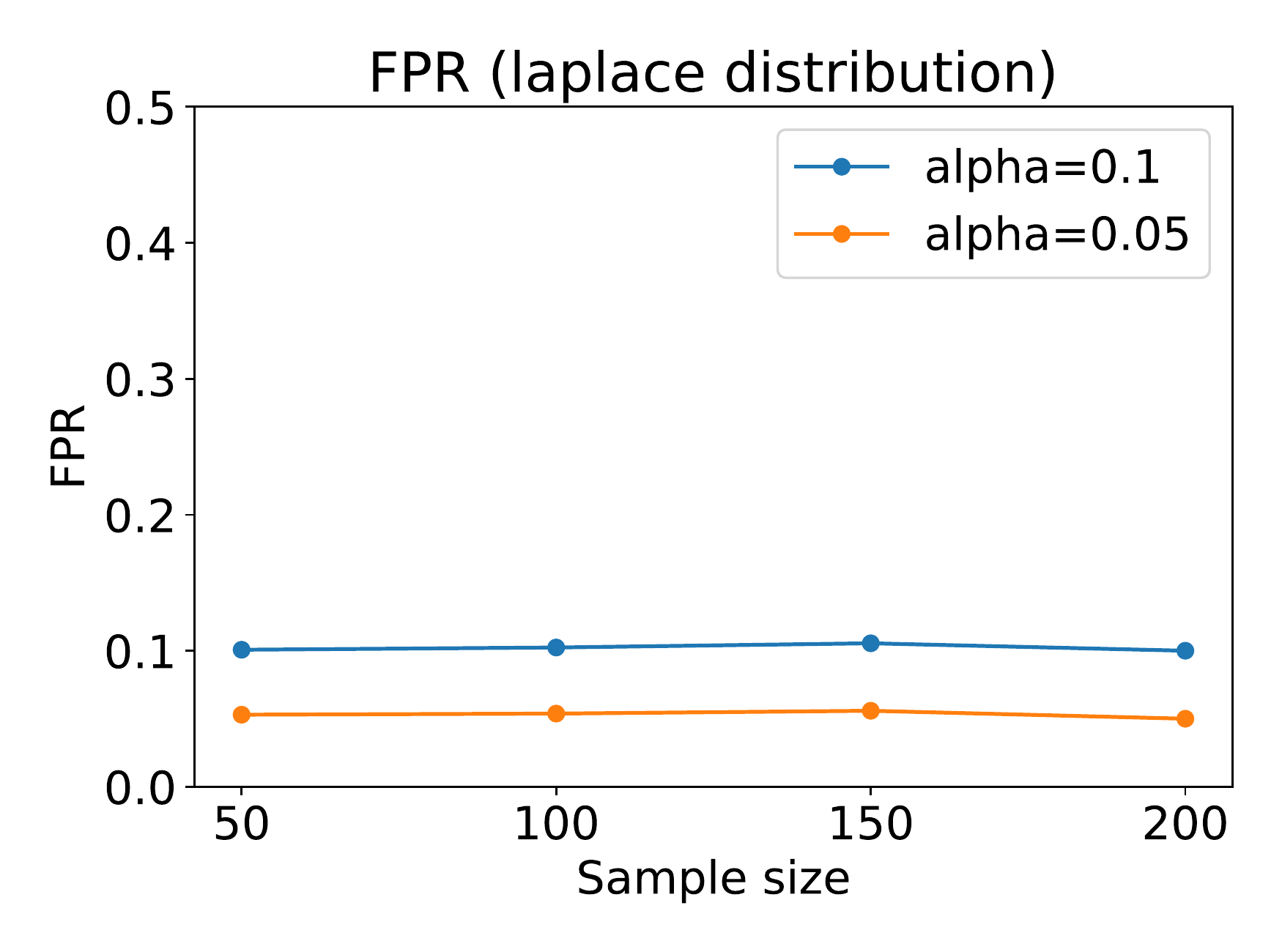}
\end{subfigure}
\begin{subfigure}{.245\textwidth}
  \centering
  \includegraphics[width=\linewidth]{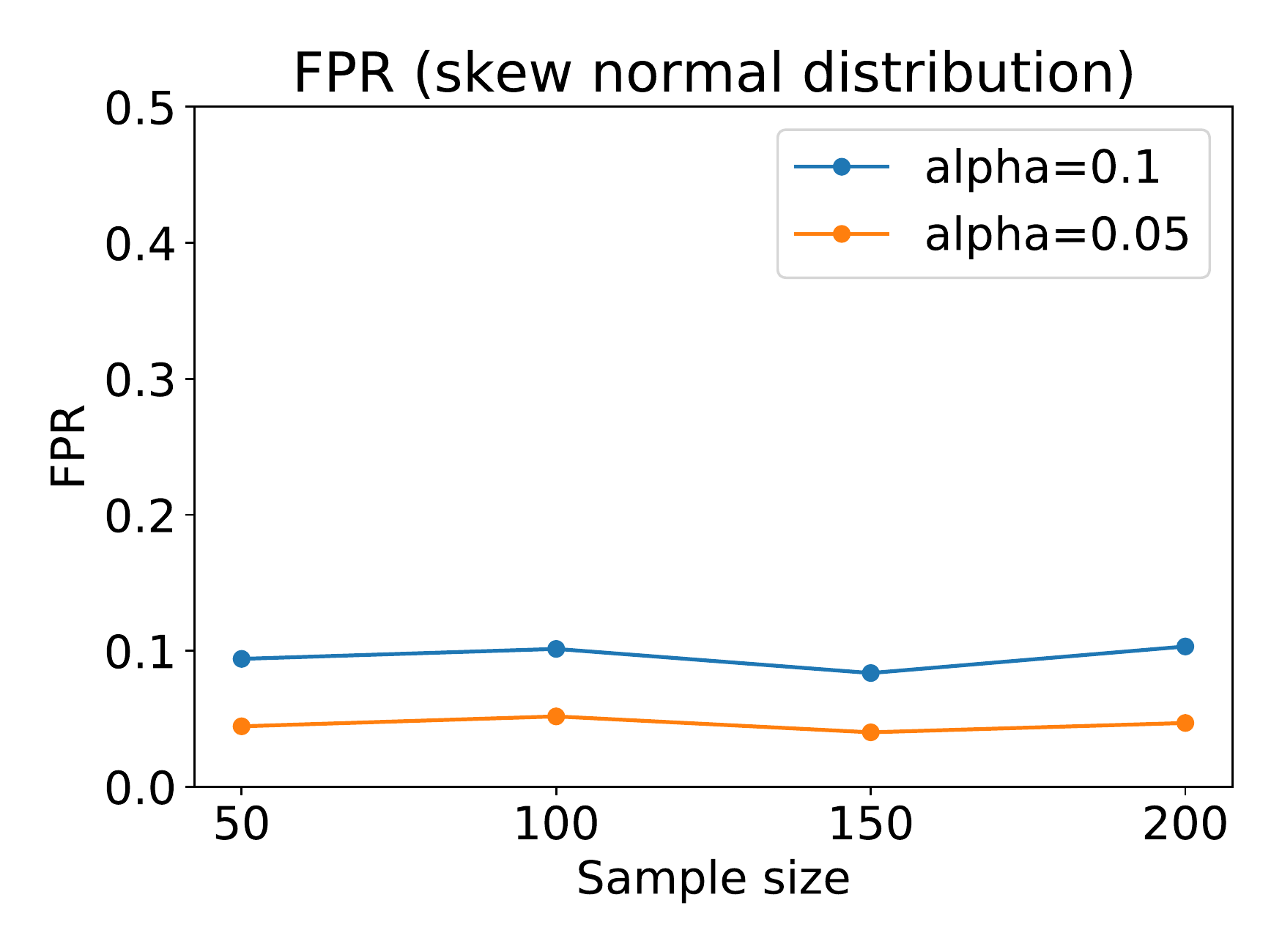}
\end{subfigure}
\begin{subfigure}{.245\textwidth}
  \centering
  \includegraphics[width=\linewidth]{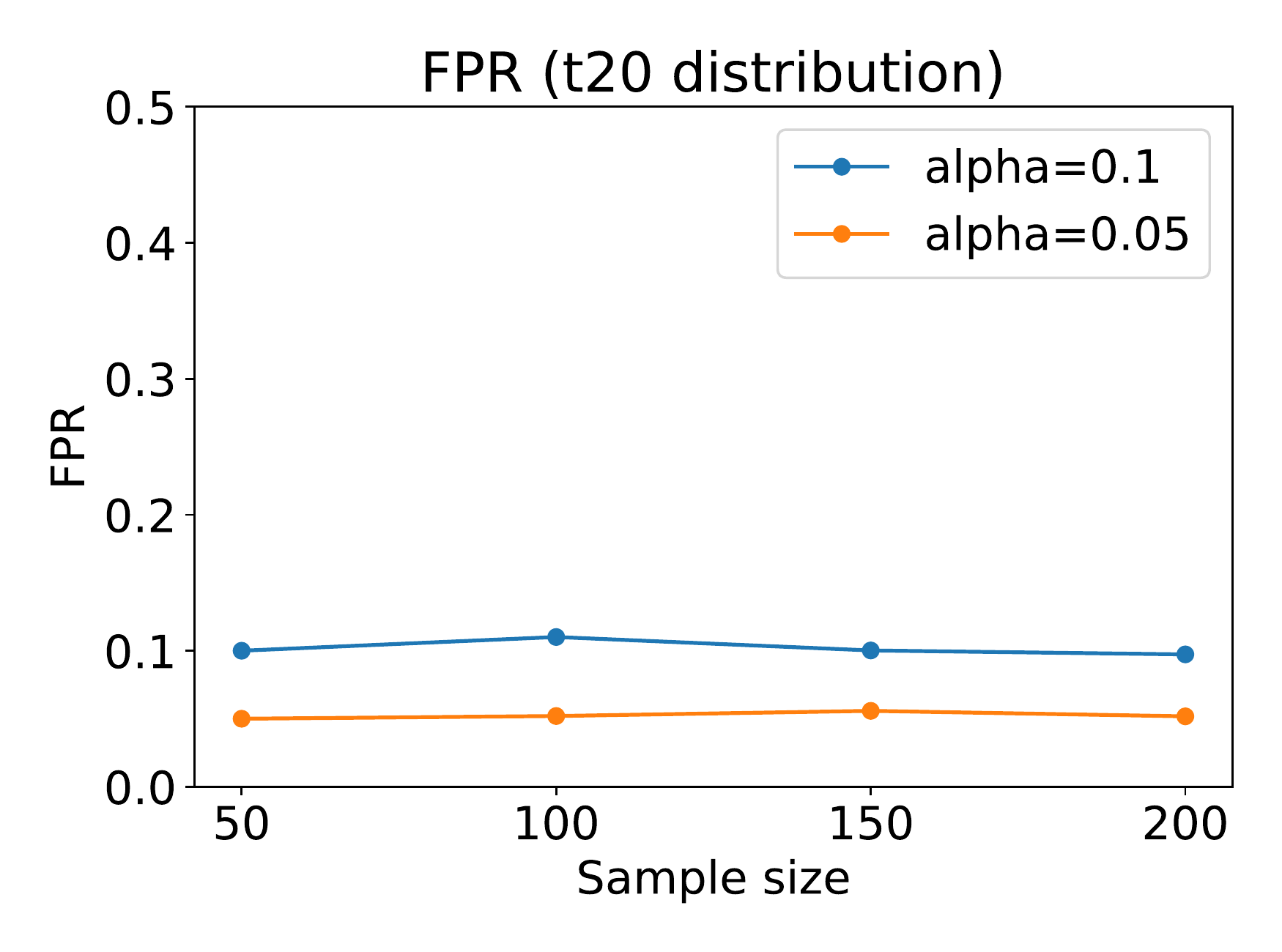}
\end{subfigure}
\begin{subfigure}{.245\textwidth}
  \centering
  \includegraphics[width=\linewidth]{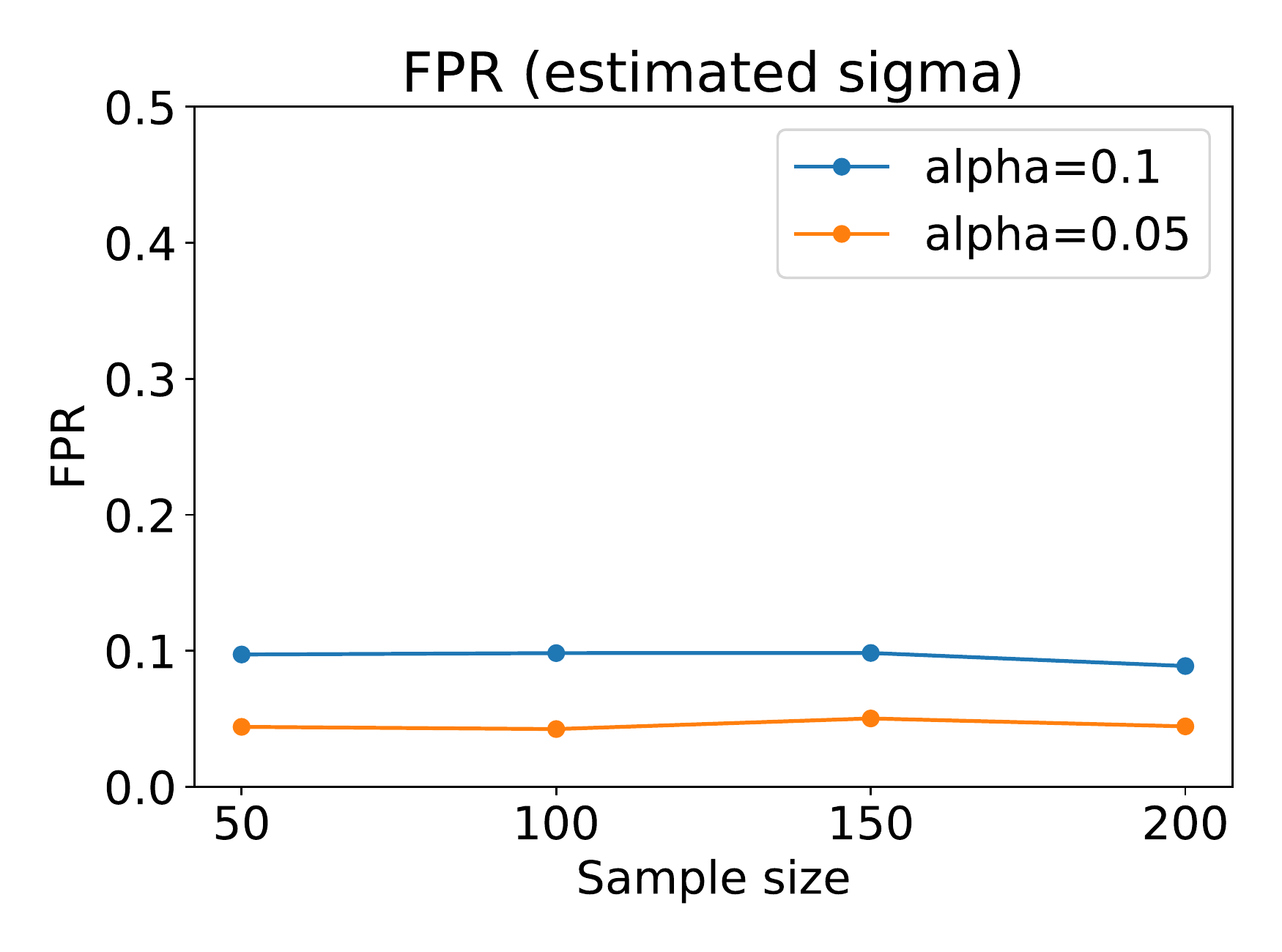}
\end{subfigure}
\caption{Robustness of proposed method for non-negative least squares.}
\label{fig:robust_non_negative}

\end{figure}

% ===

\begin{figure}[!t]

\begin{subfigure}{.245\textwidth}
  \centering
  \includegraphics[width=\linewidth]{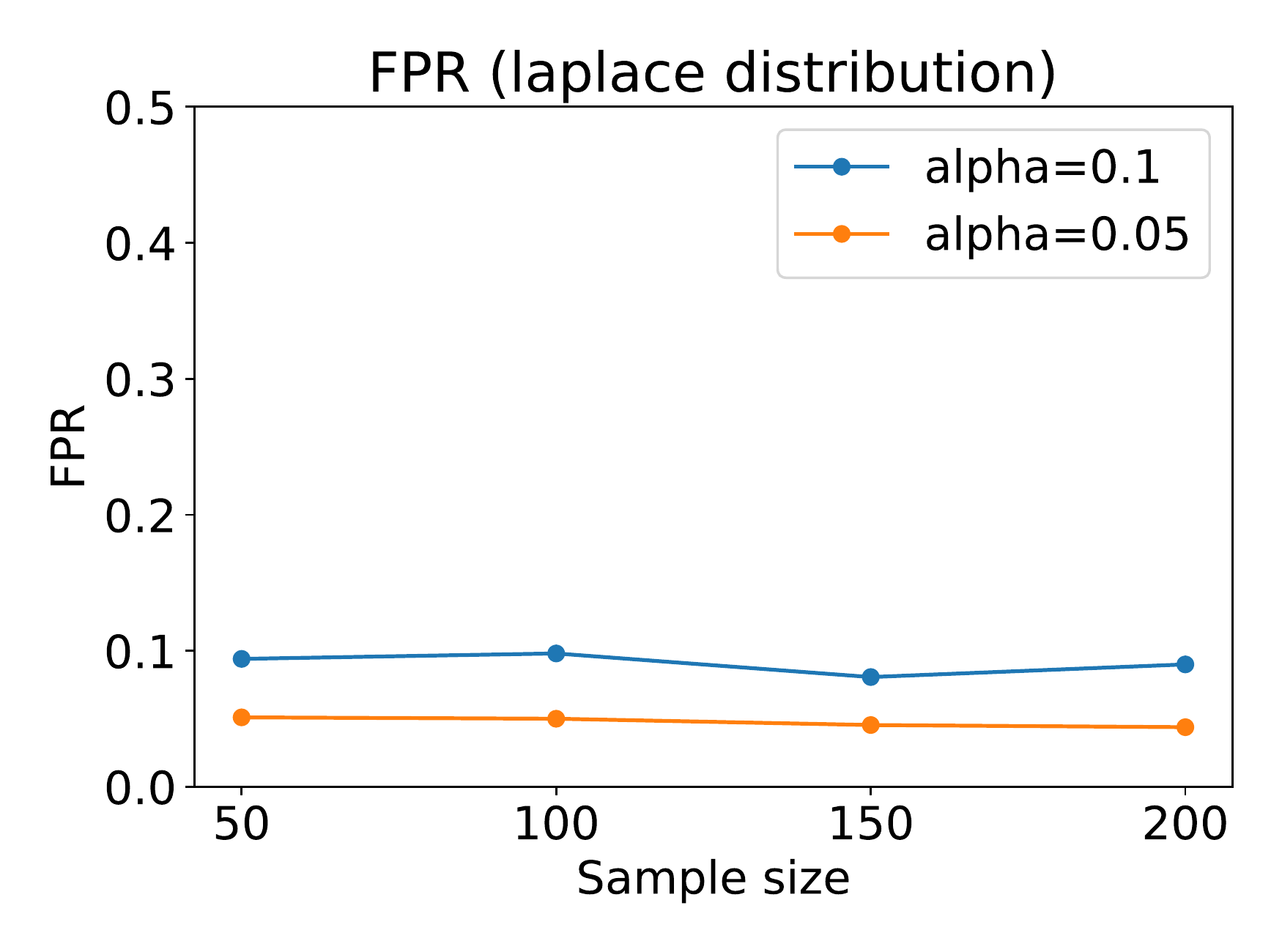}
\end{subfigure}
\begin{subfigure}{.245\textwidth}
  \centering
  \includegraphics[width=\linewidth]{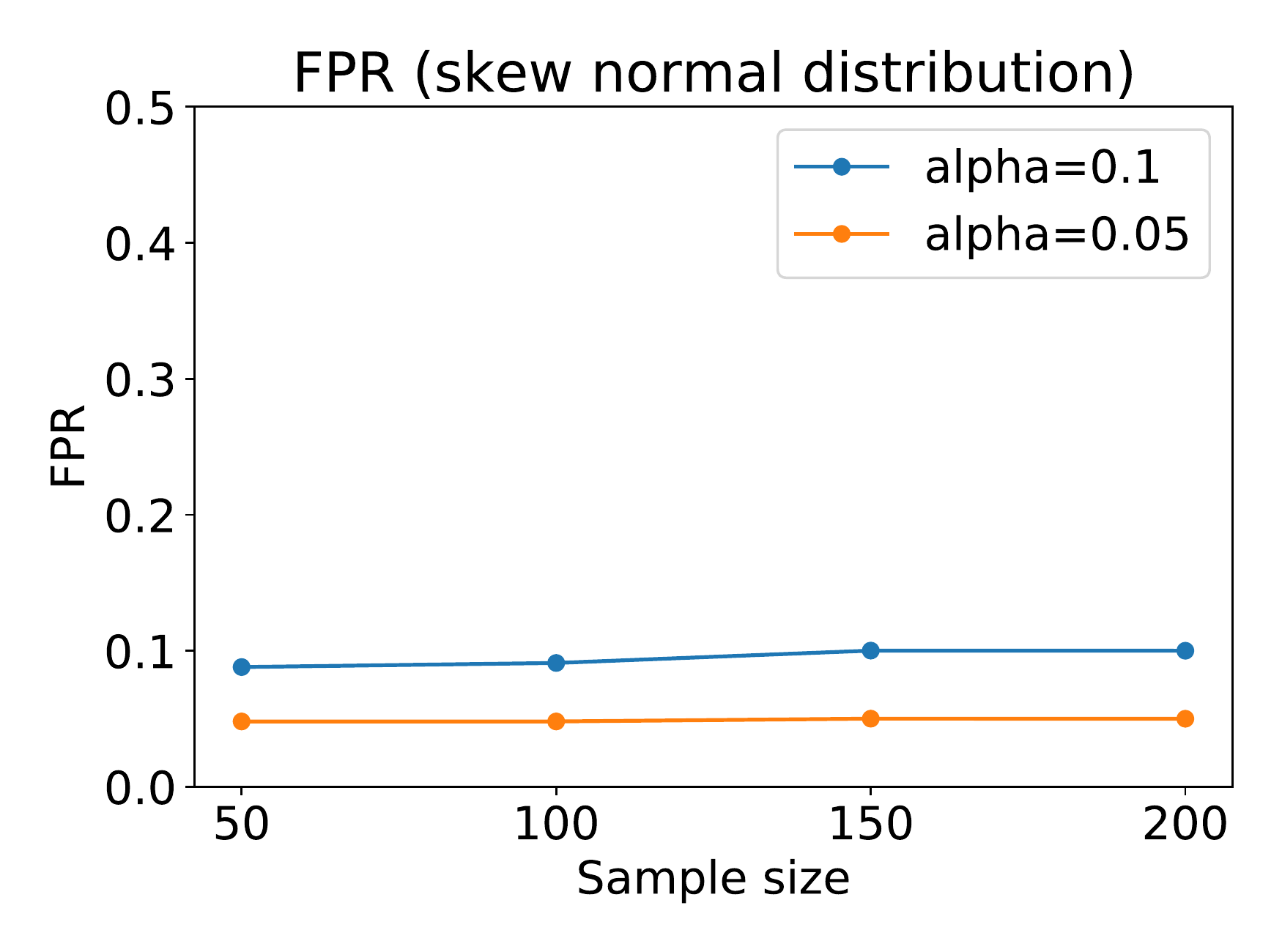}
\end{subfigure}
\begin{subfigure}{.245\textwidth}
  \centering
  \includegraphics[width=\linewidth]{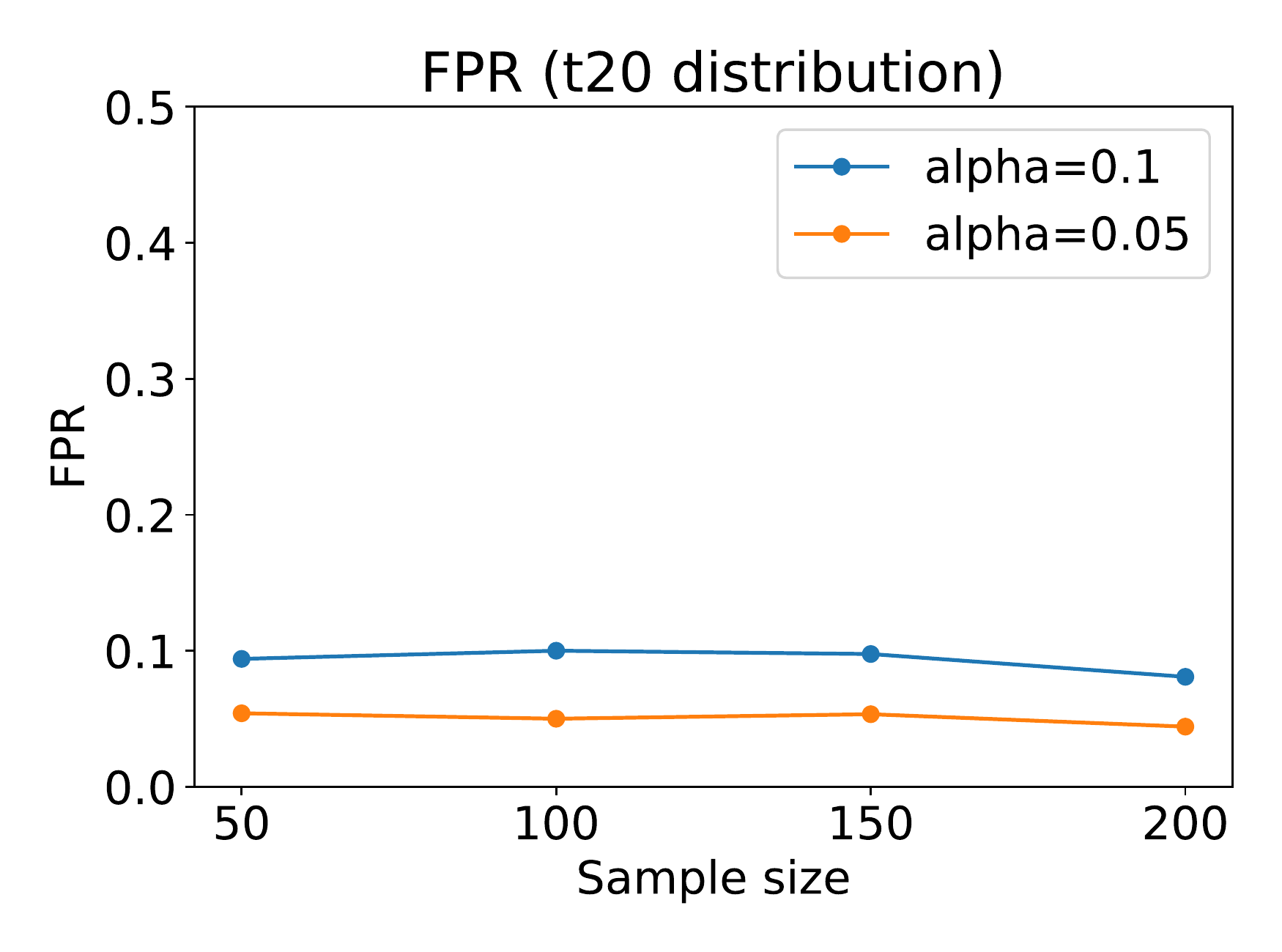}
\end{subfigure}
\begin{subfigure}{.245\textwidth}
  \centering
  \includegraphics[width=\linewidth]{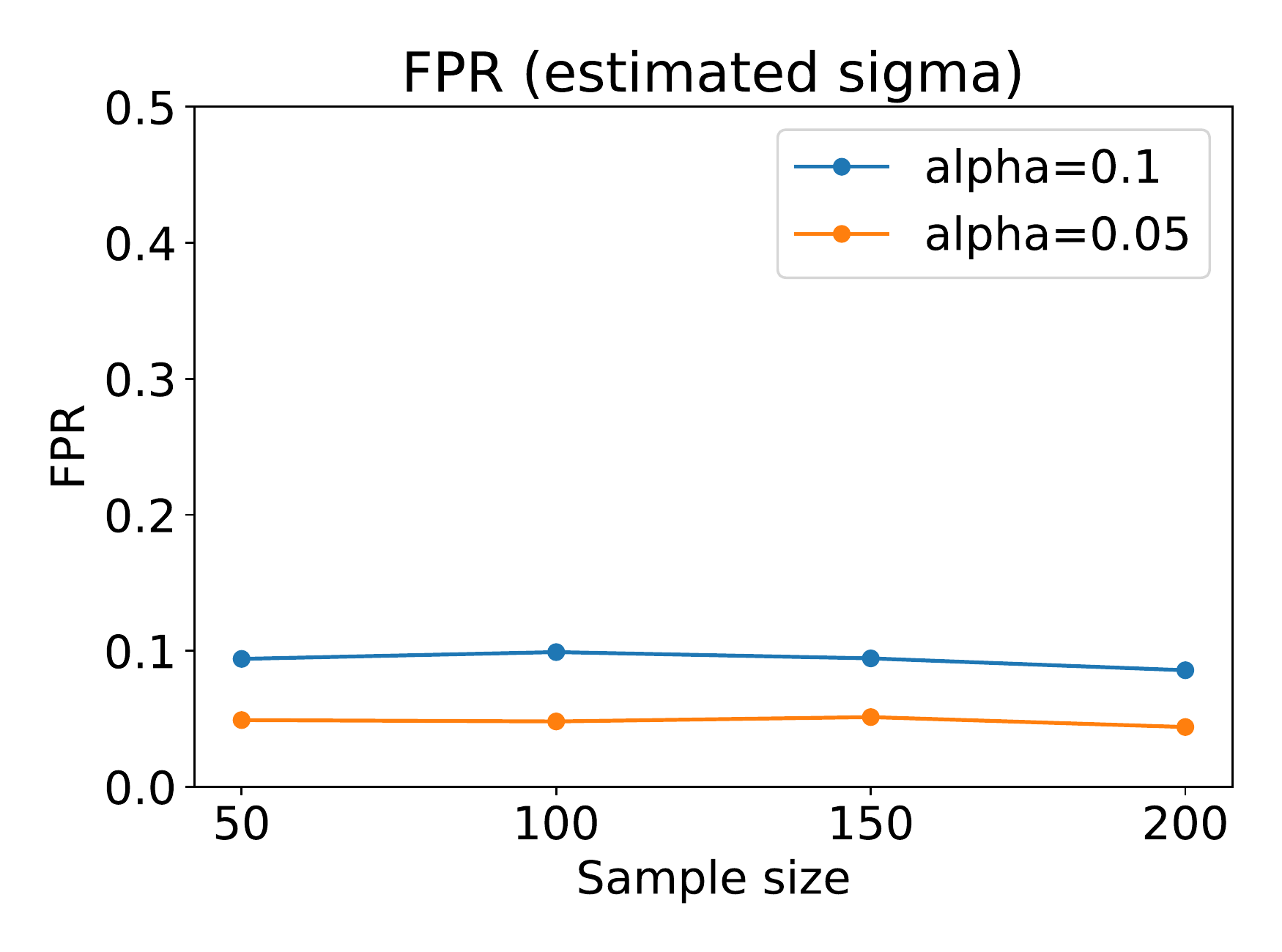}
\end{subfigure}
\caption{Robustness of proposed method for Huber regression with $\ell_1$ penalty.}
\label{fig:robust_huber_l1}
\end{figure}

% ==========================

\paragraph{Robustness of proposed method in terms of FPR control.} We demonstrated the robustness of our method in terms of FPR control by considering the following cases:

$\bullet$ Non-normal noise:  We considered the noise following the Laplace distribution, skew normal distribution (skewness coefficient: 10), and $t_{20}$ distribution.

$\bullet$ Unknown $\sigma^2$: We considered the case in which the variance was estimated from the data.

We generated $n$ outcomes: $y_i = \bm x_i^\top \bm \beta + \veps_i$, 
$i = 1, ..., n$, 
where 
$p = 5, \bm x_i \sim \NN(0, I_p)$, 
and $\veps_i$ follows a Laplace distribution, skew normal distribution, or $t_{20}$ distribution with a zero mean and the standard deviation set to 1.
In the case of the estimated $\sigma^2$, $\veps_i \sim \NN(0, 1)$.
We set all elements of $\bm \beta$ to 0 and set $\lambda = 0.5$.
For every case, we ran 1,200 trials for each $n \in \{50, 100, 150, 200\}$.
We confirmed that our method maintained high performance on FPR control. 
The results are presented in Figures \ref{fig:robust_fused_lasso}, \ref{fig:robust_lasso}, \ref{fig:robust_elastic_net}, \ref{fig:robust_non_negative}, and \ref{fig:robust_huber_l1}.

\paragraph{Results when accounting for CV selection event.} 
We conducted a comparison of the TPRs between the proposed method and the OC version that was proposed in \cite{loftus2015selective} when $\lambda$ was selected from $\Lambda_1 = \{2^{-1}, 2^{0}, 2^{1}\}$ or $\Lambda_2 = \{2^{-10}, 2^{-9}, ..., 2^{9}, 2^{10}\}$. 
The results are presented in Figure \ref{fig:fig_tpr_cv_para_oc}.
The existing method had lower power because additional conditioning on all intermediate models was required, which was also discussed in \citet{markovic2017unifying}. 
Our method exhibited higher power as we could characterize the minimum conditioning amount.

\begin{figure}[!t]
\begin{subfigure}{.495\linewidth}
  \centering
  \includegraphics[width=.7\linewidth]{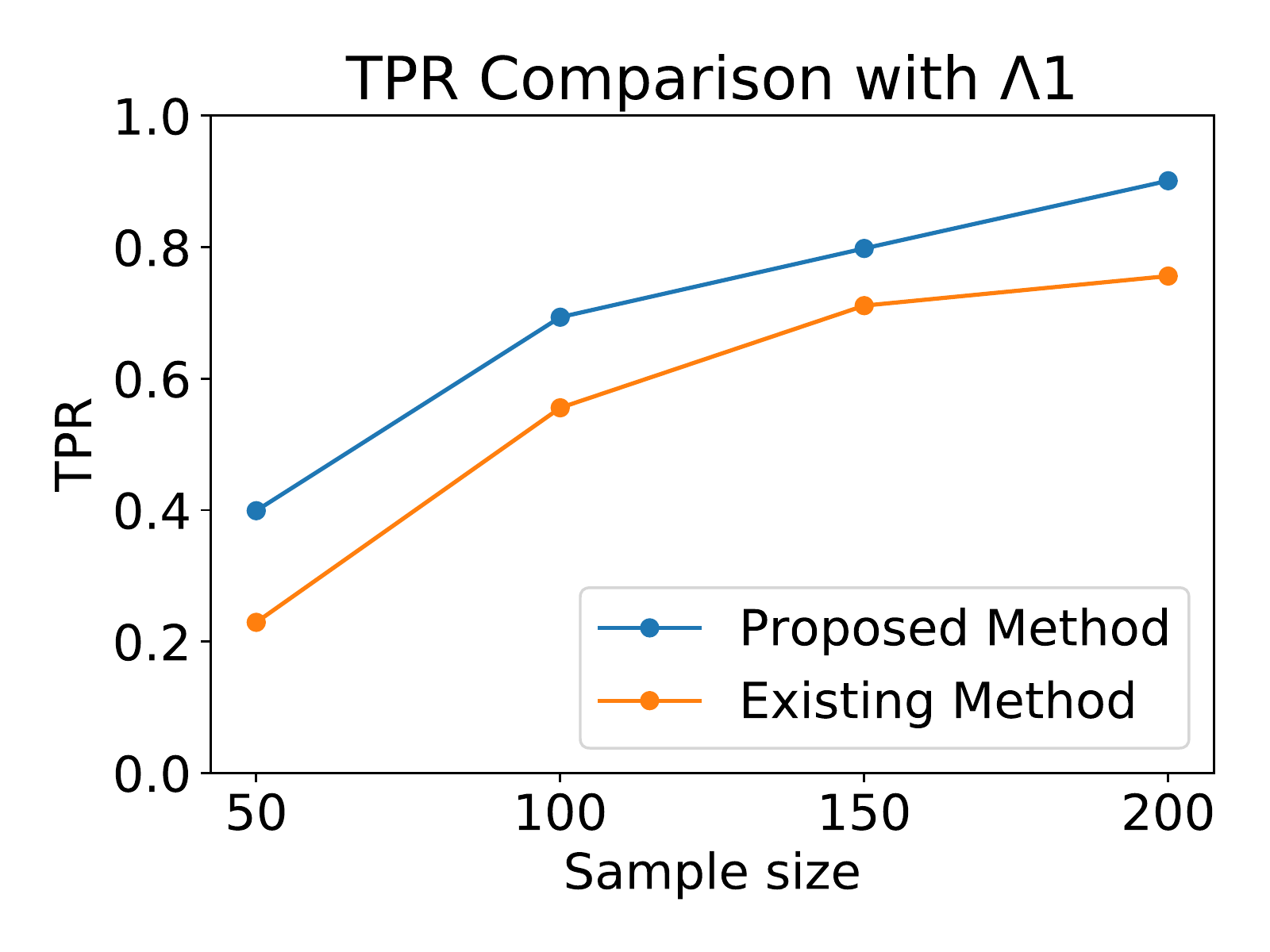}  
  \caption{$\Lambda_1 = \{2^{-1}, 2^0, 2^1\} $}
\end{subfigure}
\hspace{1pt}
\begin{subfigure}{.495\linewidth}
  \centering
  \includegraphics[width=.7\linewidth]{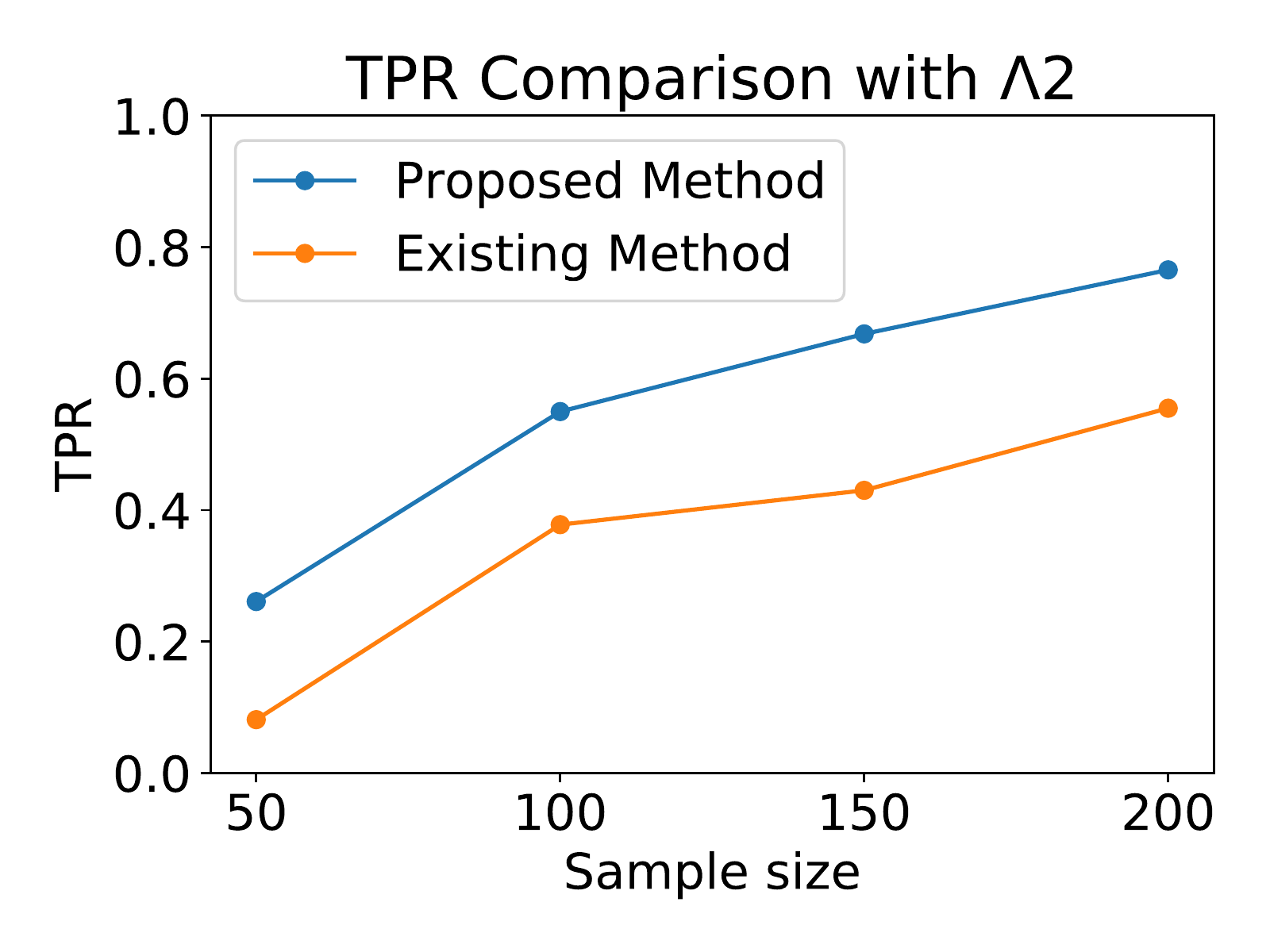}  
  \caption{$\Lambda_2 = \{2^{-10},..., 2^{10}\} $}
\end{subfigure}
\caption{TPR comparison with existing method \citep{loftus2015selective} when accounting for CV selection event.} 
\label{fig:fig_tpr_cv_para_oc}
\end{figure}

% ==================

\paragraph{Efficiency of proposed method.} 
Our main purpose was to demonstrate that the proposed method has not only high statistical power, but also reasonable computational costs.
We conducted experiments on feature selection by lasso.
The computational time of the proposed method was almost linear with respect to the number of active features, as illustrated in Figure \ref{fig:cc_1}.
Figure \ref{fig:cc_2} presents a boxplot of the actual number of intervals of z that were encountered on the line when constructing the truncation region $\cZ$.
Figure \ref{fig:cc_3} depicts the efficiency of our method compared to that of \citet{lee2016exact}, in which the authors mentioned the \emph{naive} method for removing sign conditioning by enumerating all possible combinations of signs $2^{|\cM_{\rm obs}|}$.

\begin{figure}[!t]
\begin{subfigure}{.32\linewidth}
\centering
\includegraphics[width=\textwidth]{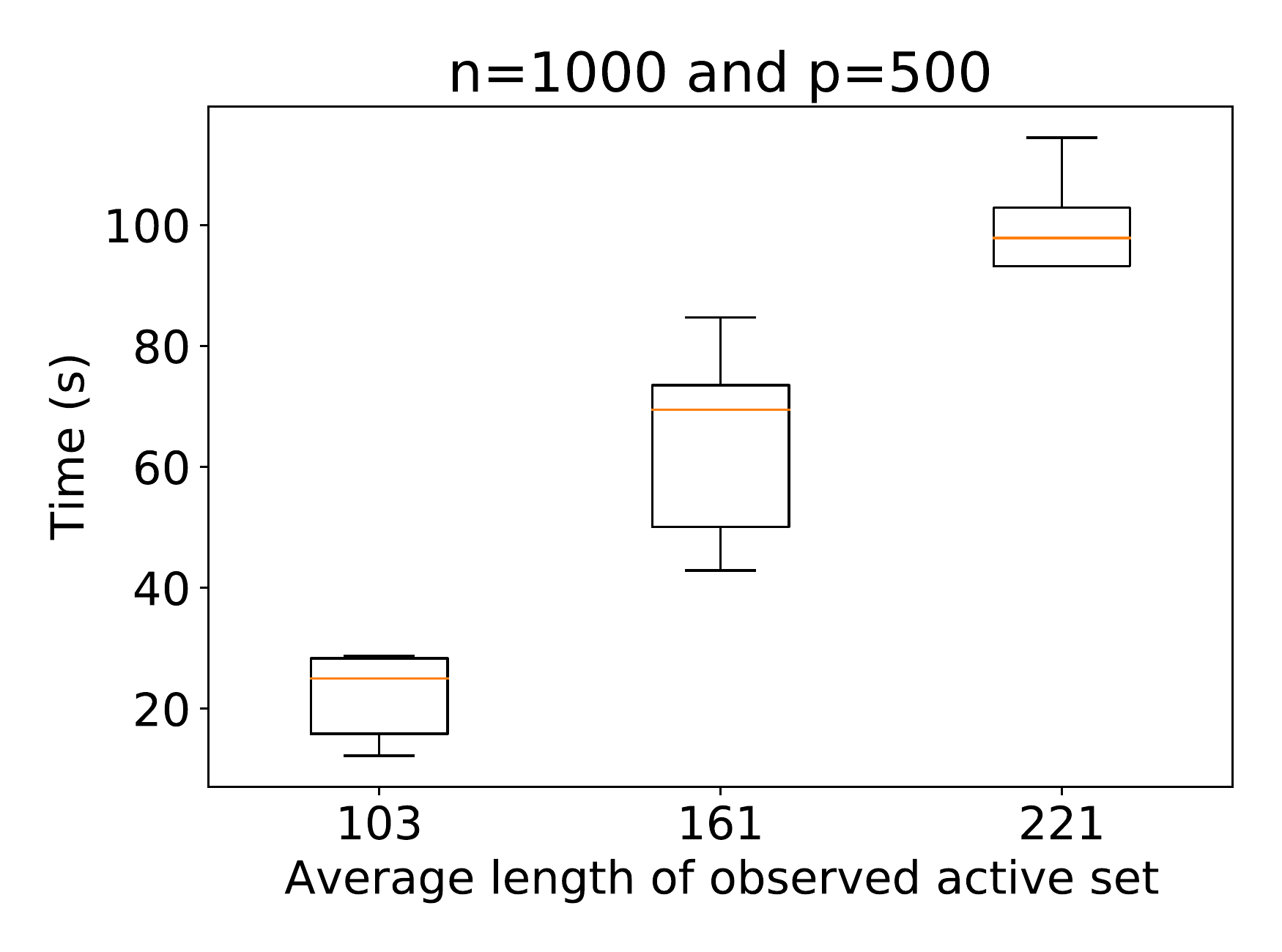}
\caption{Computational time of proposed method.}
\label{fig:cc_1}
\end{subfigure}
\hspace{1pt}
\begin{subfigure}{.32\linewidth}
\centering
\includegraphics[width=\textwidth]{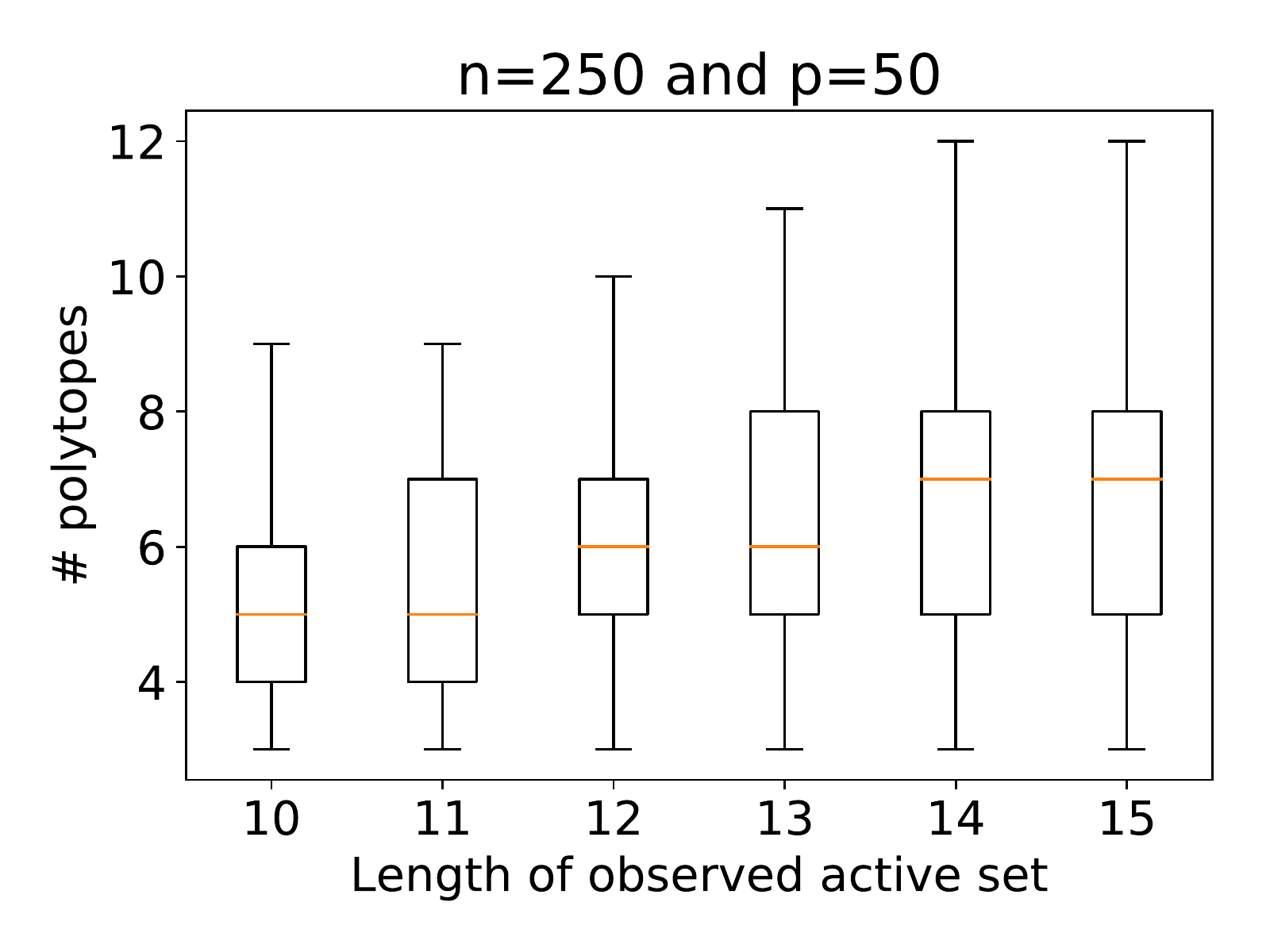}
\caption{Number of encountered intervals on line.}
\label{fig:cc_2}
\end{subfigure}
\hspace{1pt}
\begin{subfigure}{.31\linewidth}
\centering
\includegraphics[width=\textwidth]{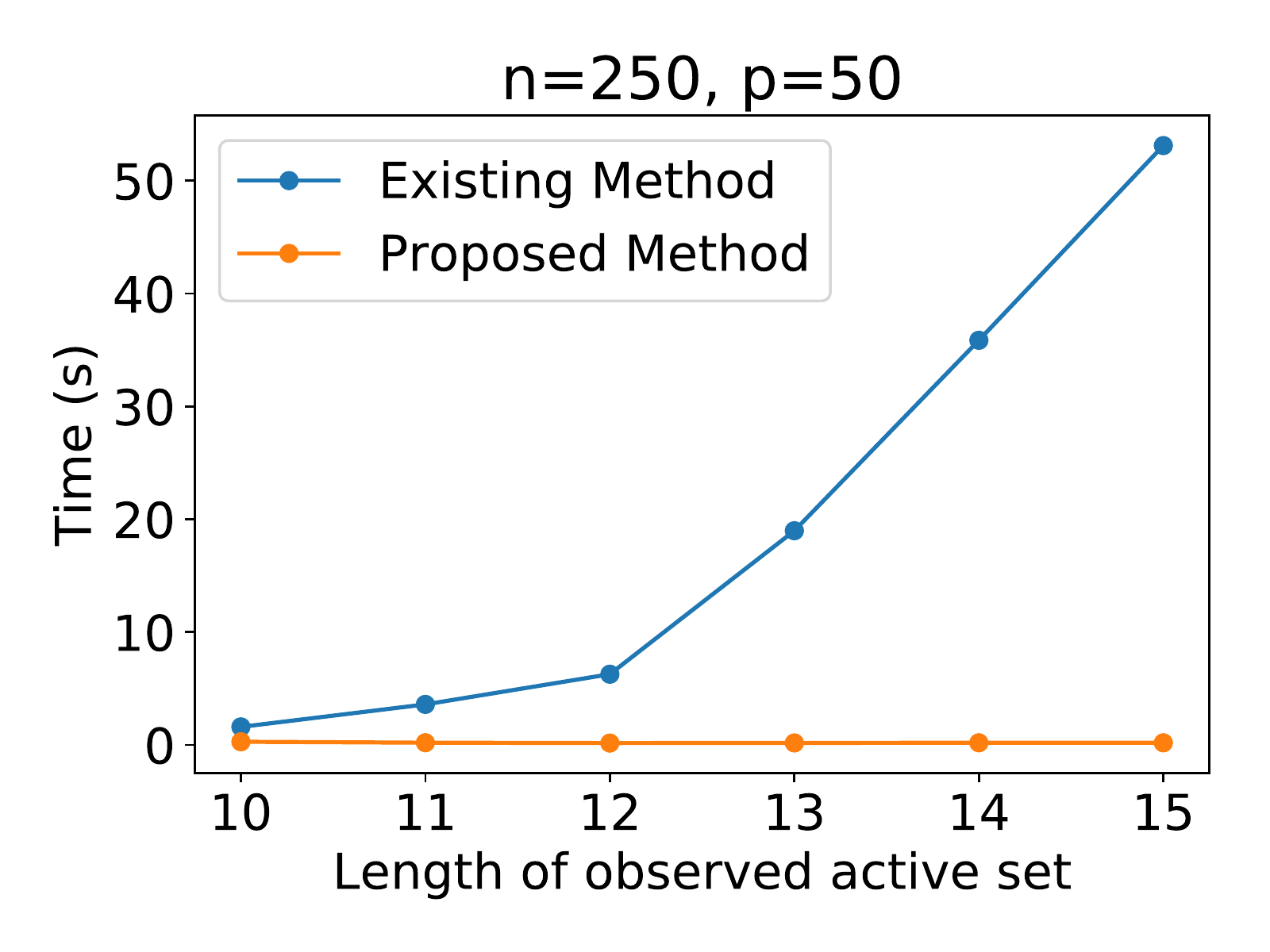}
\caption{Comparison with \citet{lee2016exact}.}
\label{fig:cc_3}
\end{subfigure}
\caption{Efficiency of proposed method.}
\end{figure}

\paragraph{Comparison with \citet{liu2018more}.} 
Furthermore, we demonstrated the efficiency of our method compared to the two methods proposed in \S3 (inference for partial regression targets) of \citet{liu2018more}.
In this work, only stable features were allowed to influence the formation of the test statistic.
Stable features are those with very strong signals and that cannot missed.
In the first method, which we refer to as TN-$\ell_1$, the stable features were selected by setting a higher value of $\lambda$.
In the second method, which we refer to as TN-Custom, the stable features were selected by setting a cutoff value.
The details of these two methods are presented in Appendix \ref{appendix:partial_target}.
In general, to perform SI with these two methods, all possible combinations of signs, which increase exponentially, still need to be naively enumerated.
The proposed method can be used to solve this computational bottleneck.
A comparison of the computational costs is illustrated in Figure \ref{fig:cc_compare_liu}.

\begin{figure}[!t]
\begin{subfigure}{.49\linewidth}
\centering
\includegraphics[width=.8\textwidth]{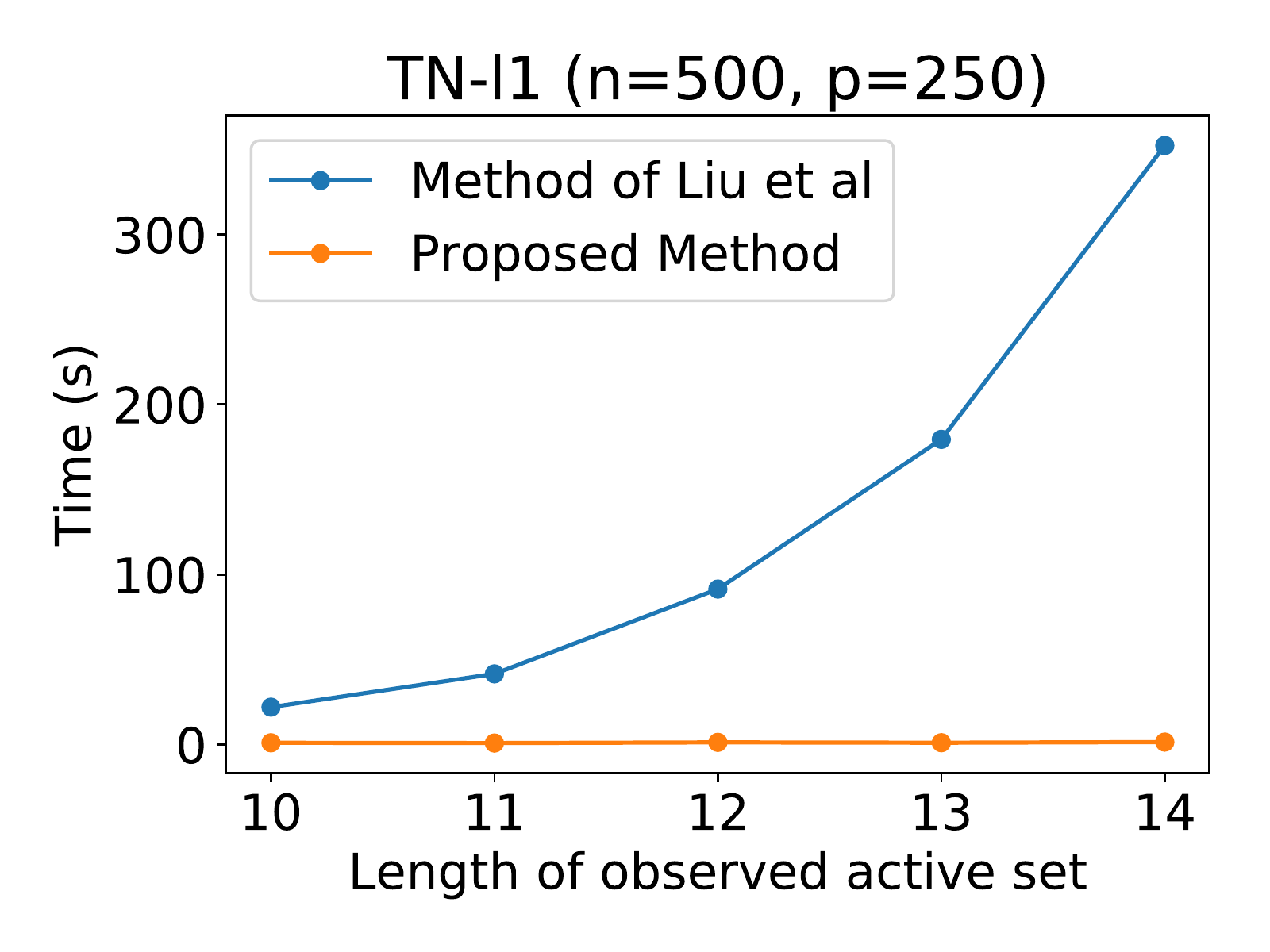}
\end{subfigure}
\begin{subfigure}{.49\linewidth}
\centering
\includegraphics[width=.8\textwidth]{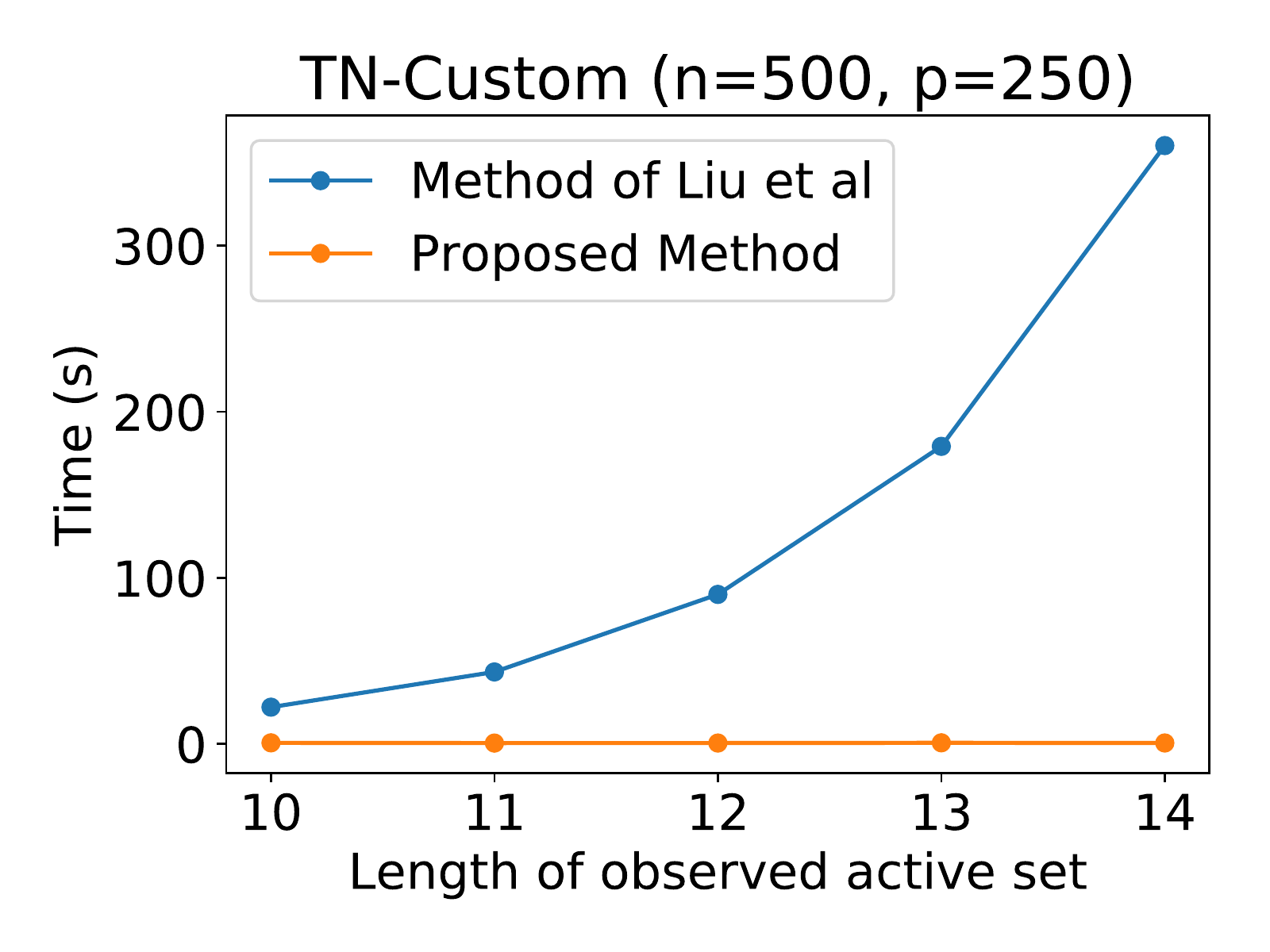}
\end{subfigure}
\caption{Comparison between proposed method and methods in \citet{liu2018more}, in which an exponentially increasing number of all possible sign combinations is still required.}
\label{fig:cc_compare_liu}
\end{figure}

% ==========================

\subsection{Results on Real-World Datasets} \label{subsec:exp_real_data_results}

\begin{figure}[!t]

\begin{subfigure}{.495\textwidth}
  \centering
  \includegraphics[width=.95\linewidth]{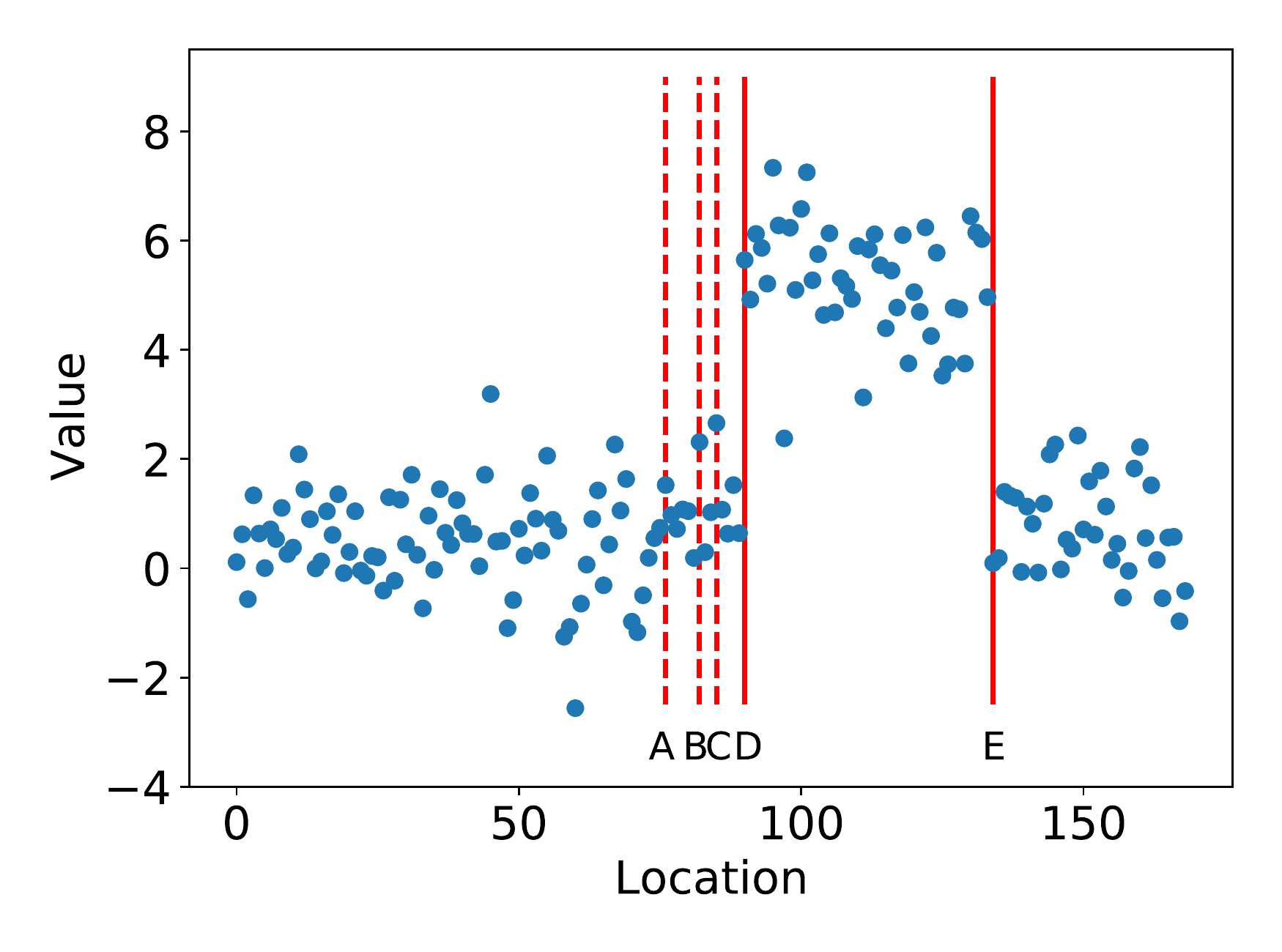}  
  {\footnotesize
  \def\arraystretch{1.2}
  \begin{tabular}{ | c | c | c | c | c | c |}
   \hline
    & A & B & C & D & E\\ 
   \hline
   Proposed & 0.9 & 0.6 & 0.3 & $1.4 \times 10^{-15}$ & 0.0 \\ 
   \hline
   OC & 0.9 & 0.3 & 0.3 & $6.0 \times 10^{-5}$ & $1.2 \times 10^{-11}$ \\  
   \hline
  \end{tabular}
  }
  \caption{Chromosomes 17, 18, and 19.}
\end{subfigure}
\begin{subfigure}{.495\textwidth}
  \centering
  % include second image
  \includegraphics[width=.95\linewidth]{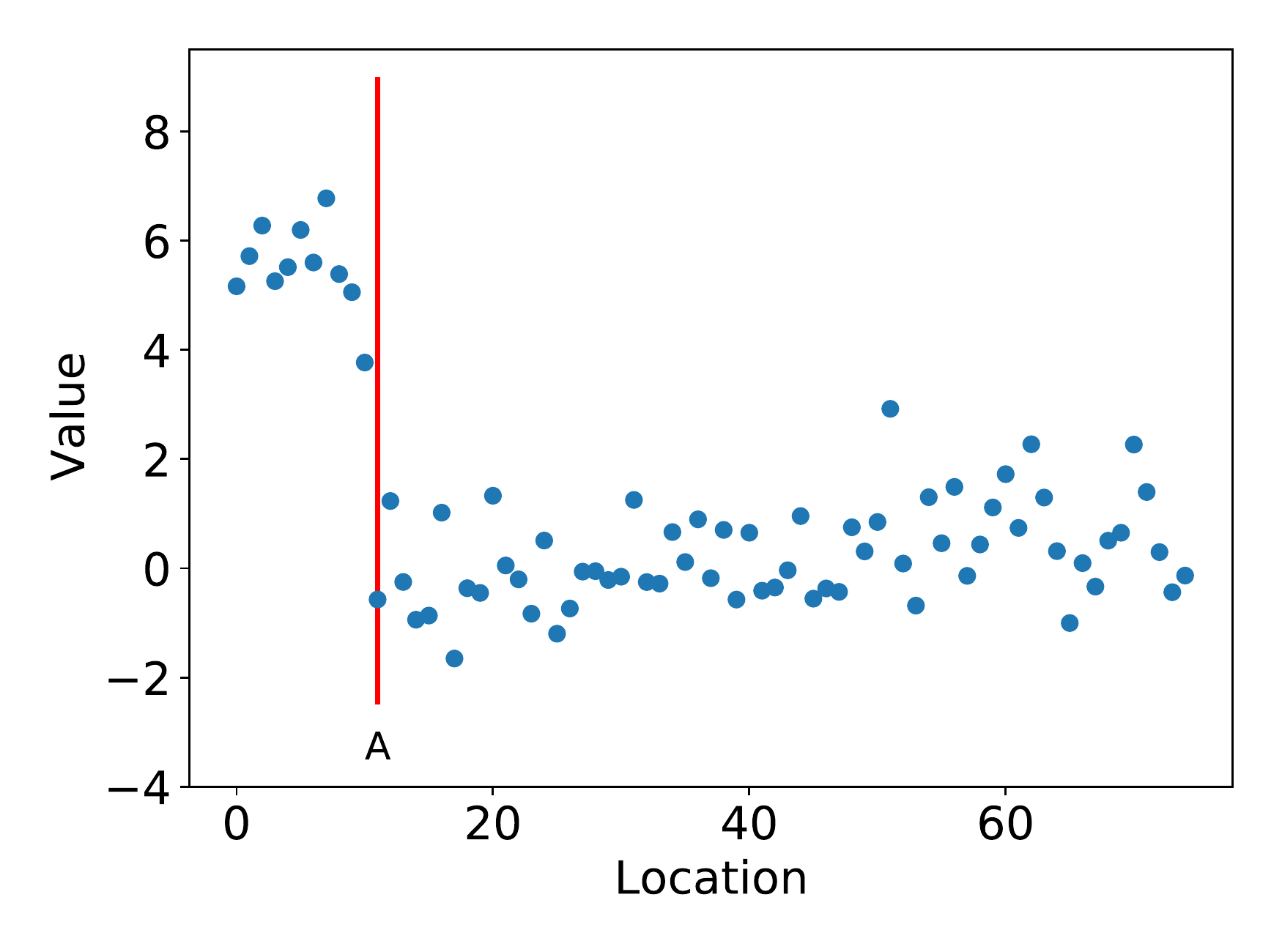}  
  {\footnotesize
  \def\arraystretch{1.2}
  \begin{tabular}{ | c | c |}
   \hline
    & A \\ 
   \hline
   Proposed &  $~~~~~~~~~0.0~~~~~~~~~$ \\ 
   \hline
   OC & 0.0  \\  
   \hline
  \end{tabular}
  }
  \caption{Chromosome 14.}
\end{subfigure}
\caption{Experimental results for cell lines GM00143 and GM01750.}
\label{fig:exp_cgh_GM00143_GM01750}

\vspace*{\floatsep}% 

\begin{subfigure}{.495\textwidth}
  \centering
  \includegraphics[width=.95\linewidth]{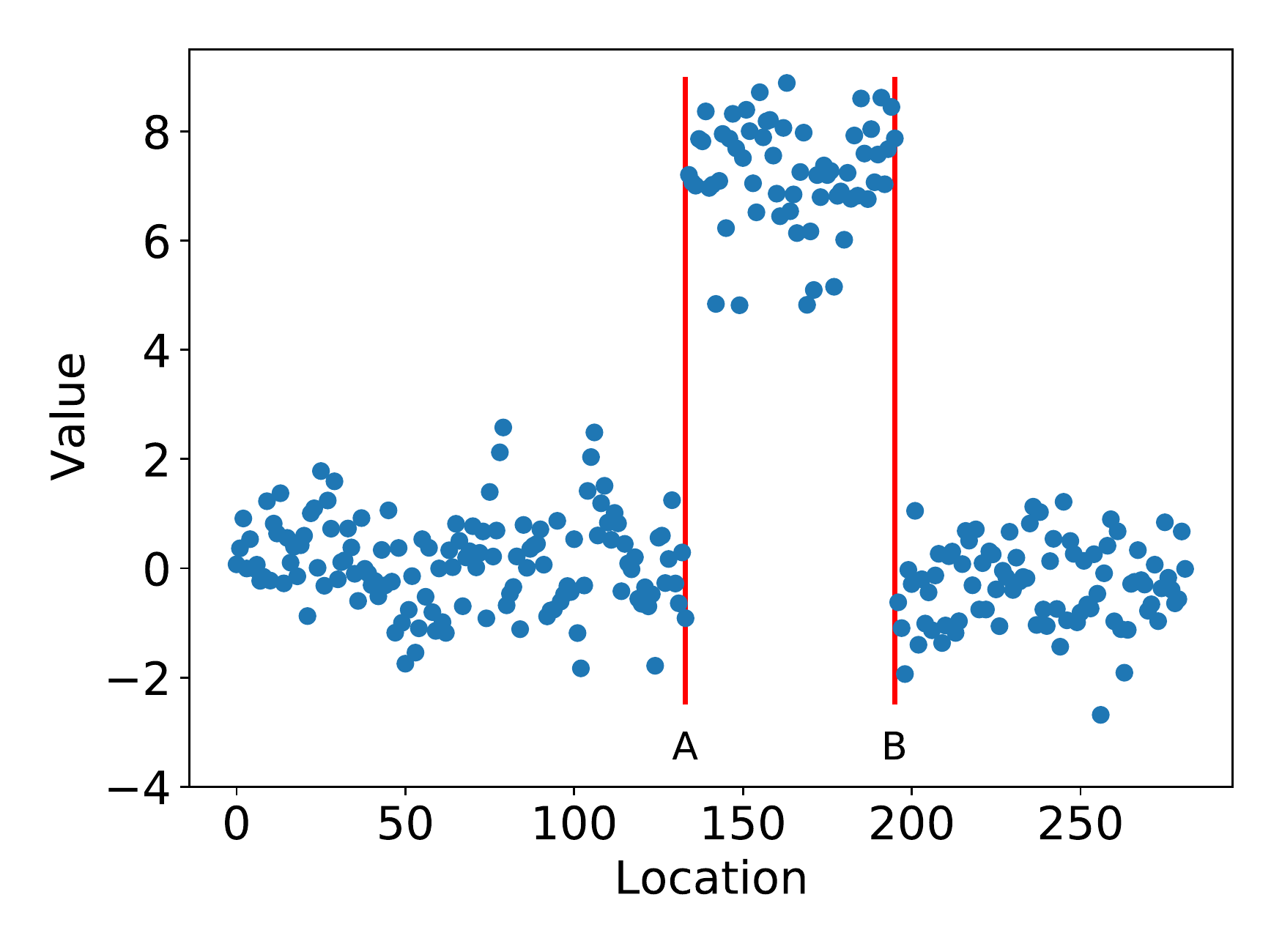}  
  {\footnotesize
  \def\arraystretch{1.2}
  \begin{tabular}{ | c | c | c |}
   \hline
    & A & B\\ 
   \hline
   Proposed & $~~ 2.7 \times 10^{-320} ~~$ & $~~~~~0.0~~~~~$  \\ 
   \hline
   OC & $5.3 \times 10^{-98}$ & 0.0 \\  
   \hline
  \end{tabular}
  }
  \caption{Chromosomes 1, 2, and 3.}
\end{subfigure}
\begin{subfigure}{.495\textwidth}
  \centering
  % include second image
  \includegraphics[width=.95\linewidth]{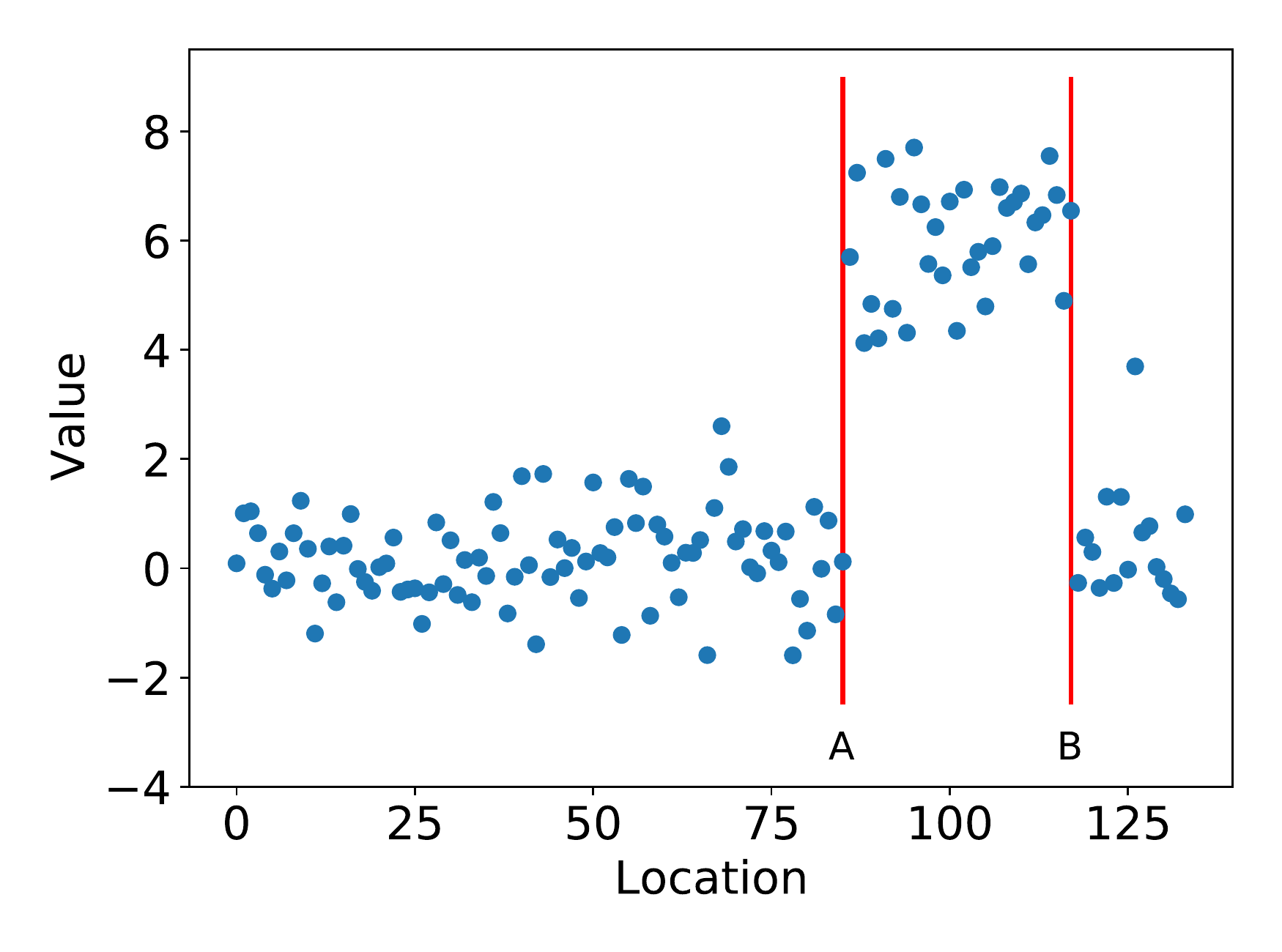}  
  {\footnotesize
  \def\arraystretch{1.2}
  \begin{tabular}{ | c | c | c |}
   \hline
    & A & B\\ 
   \hline
   Proposed & $~~ 1.1 \times 10^{-41}~~$ & $~~ 9.1 \times 10^{-14}~~$ \\ 
   \hline
   OC & $1.1 \times 10^{-41}$ & $9.1 \times 10^{-14}$ \\  
   \hline
  \end{tabular}
  }
  \caption{Chromosomes 20, 21, and 22.}
\end{subfigure}
\caption{ Experimental results for cell line GM03576.}
\label{fig:exp_cgh_GM03576}

\end{figure}

\paragraph{Array CGH data.} Array CGH analyses enable the detection of changes in copy numbers across the genome. 
We applied the proposed method and OC to the dataset with the ground truth provided in \citet{snijders2001assembly}. 
The results of the detected CPs and tables of $p$-values are presented in Figures \ref{fig:exp_cgh_GM00143_GM01750} and \ref{fig:exp_cgh_GM03576}. The solid red line denotes the significant CPs, which had a $p$-value that was smaller than the significance level following Bonferroni correction. All of the results were consistent with those of \citet{snijders2001assembly}.
Moreover, we compared the $p$-values of the proposed method and OC.
The $p$-values of the proposed method were smaller than or equal to those of OC for all true CPs, which indicates that the proposed method had higher power than OC.

The boxplots of the distribution of the $p$-values for the proposed method and OC on the real-world dataset are illustrated in Figure \ref{fig:box_plot_p_value}. 
We used the \emph{jointseg} package \citep{pierre2014performance} to generate realistic DNA copy number profiles of cancer samples with ``known'' truths.
Two datasets consisting of 1,000 profiles, each with a length of $n=60$, were created, as follows: 
\begin{itemize}
	\item  Dataset 1: Resampled from GSE11976 with tumor fraction $=$ 1.
	\item  Dataset 2: Resampled from GSE29172 with tumor fraction $=$ 1.
\end{itemize}

\begin{figure}[t]
\begin{subfigure}{.49\textwidth}
  \centering
  \includegraphics[width=.95\linewidth]{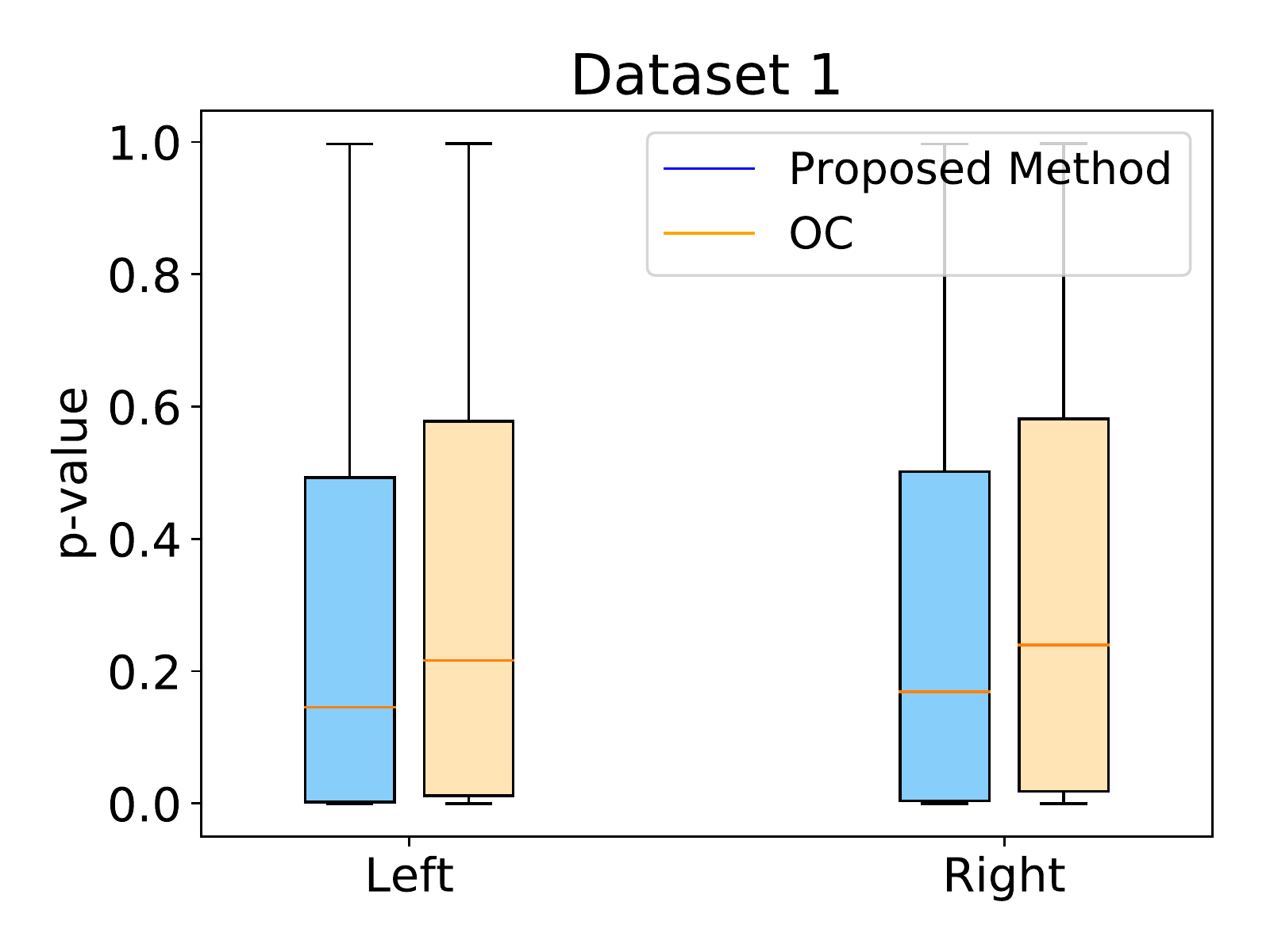}  
%  \vspace{-0.5cm}
%  \caption{The left plot shows the distributions of p-values for Dataset 1. 
%  %
%  The percentage the $p$-value of the proposed method is smaller than OC is 55.81\%.
%  %
%  The right plot shows how much the $p$-value of the proposed method decreases in the 55.81\%.
%  }
\end{subfigure}
\hspace{1pt}
\begin{subfigure}{.49\textwidth}
  \centering
  \includegraphics[width=.95\linewidth]{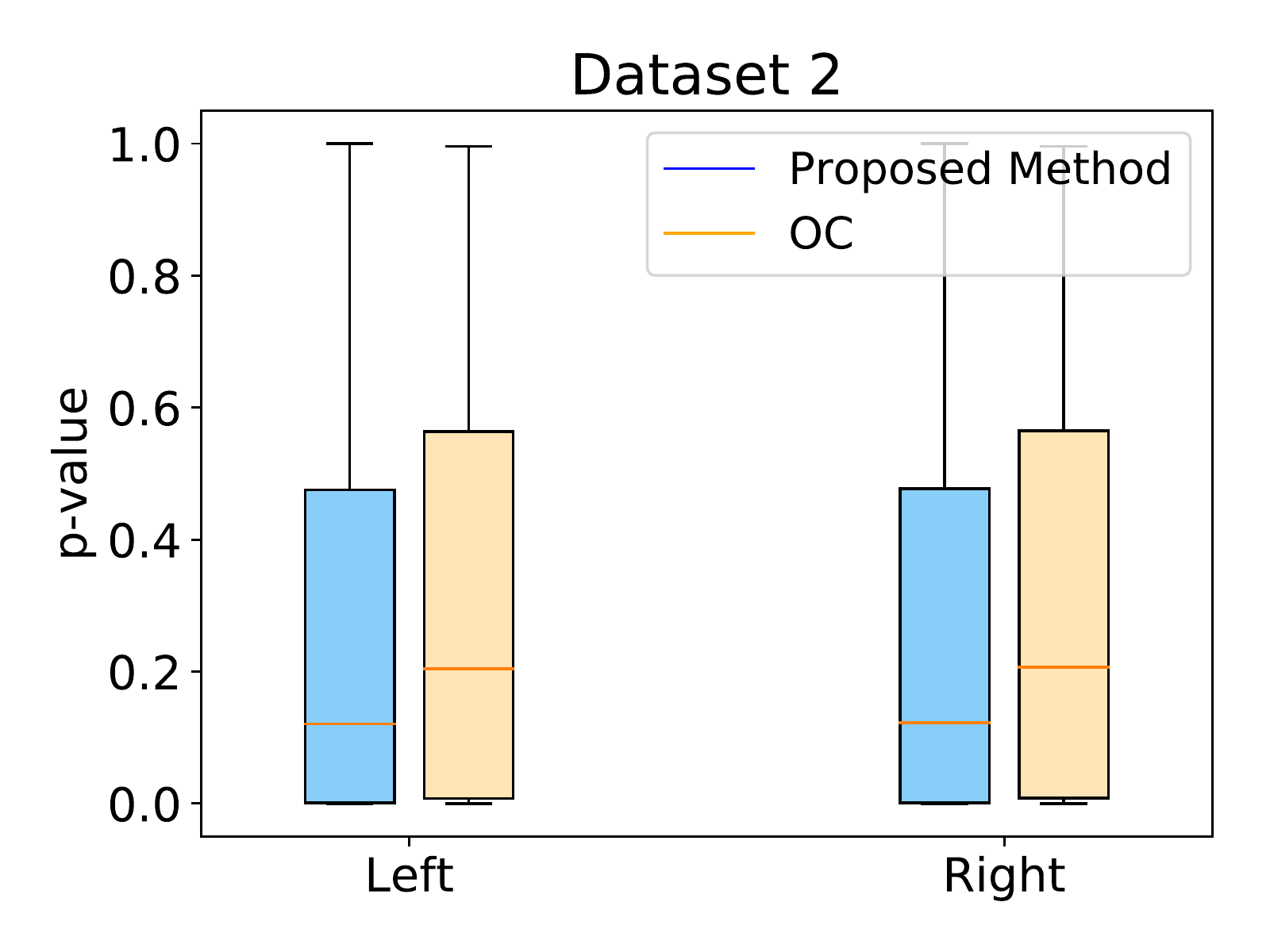}  
%  \vspace{-0.5cm}
%  \caption{The left plot shows the distributions of p-values for Dataset 2. 
%  %
%  The percentage the $p$-value of the proposed method is smaller than OC is 54.24\%.
%  %
%  The right plot show how much the $p$-value of the proposed method decreases in the 54.24\%.}
\end{subfigure}
\caption{Boxplots of $p$-values. 
The left plot in each figure depicts the distributions of the $p$-values, whereas
the right plot displays the distributions of the $p$-values for the cases in which the two $p$-values of the proposed method and OC differed.
In Dataset 1, the percentage of the $p$-value of the proposed method was 55.81\% smaller than that of OC.
In Dataset 2, the percentage of the $p$-value of the proposed method was 54.24\% smaller than that of OC.
In general, the $p$-value of the proposed method tended to be smaller than that of OC, which indicates that the proposed method had higher statistical power than OC. 
}
\label{fig:box_plot_p_value}
\end{figure}

\paragraph{Nile data.} 
These data contain the annual flow volume of the Nile River at Aswan from 1871 to 1970 (100 years).
In this case, the interest lies in unexpected events such as natural disasters.
According to Figure \ref{fig:exp_nile_data}, the proposed algorithm identified a CP at the $28^{\rm th}$ position, corresponding to the year 1899.
This result was consistent with that of \citet{jung2017bayesian}.

\begin{figure}[t]
\centering
\includegraphics[width=0.65\linewidth]{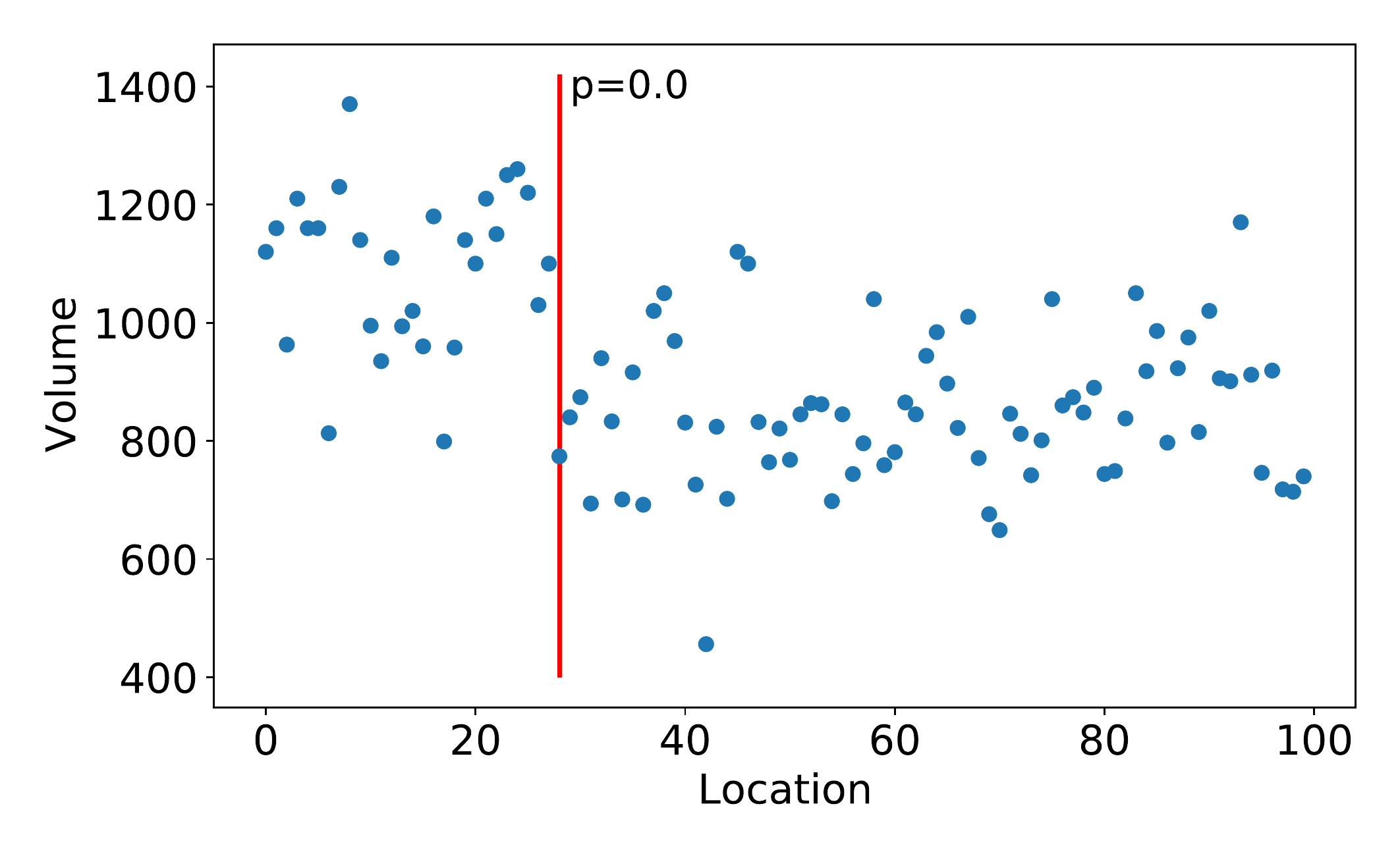}
\caption{Experimental result for Nile data. A CP was detected at the $28^{\rm th}$ position, which indicates that a change in the volume level occurred in 1899.}
\label{fig:exp_nile_data}
\end{figure}

\paragraph{Prostate data.}
We applied our proposed method for lasso to the prostate dataset from \citet{hastie2009elements}.
As $p < n$ for this dataset, we could estimate $\sigma^2$ using the residual sum of squares from the full regression model with all $p$ predictors. 
We set $\lambda = 5$.
Figure \ref{fig:exp_prostate_data} depicts the 95\% CIs for the features that were selected by both lasso and DS.

\begin{figure}[!t]
\centering
\includegraphics[width=0.55\linewidth]{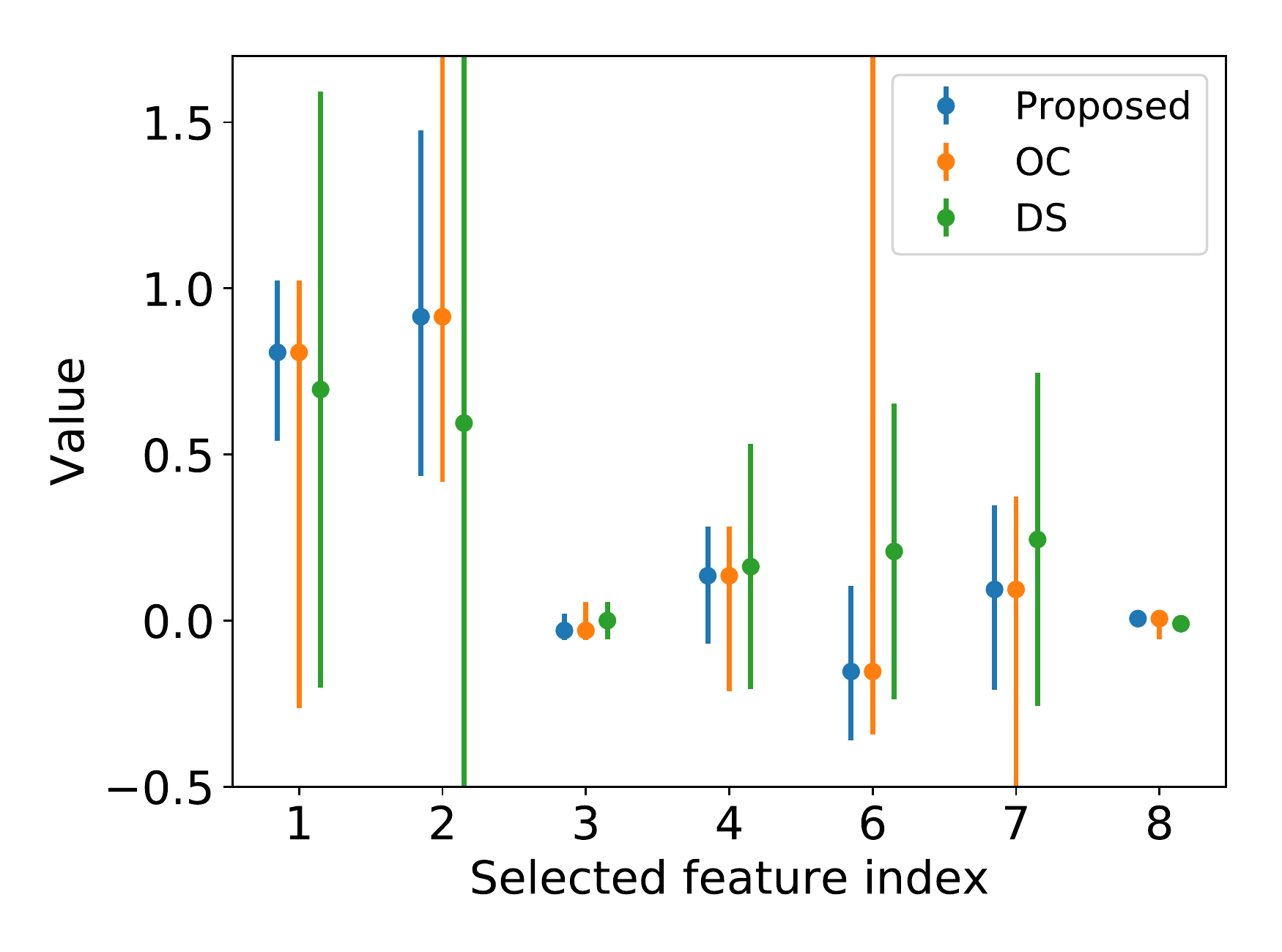}
\caption{Experimental results for prostate data. The indices of the features were 1: lcavol, 2: lweight, 3: age, 4: lbph, 6: lcp, 7: gleason, and 8: pgg45.}
\label{fig:exp_prostate_data}
\end{figure}

% --------------- Acknowledgement --------------------

\subsection*{Acknowledgements}

This work was partially supported by MEXT KAKENHI (20H00601, 16H06538), JST Moonshot R\&D (JPMJMS2033-05), NEDO (JPNP18002, JPNP20006), RIKEN Center for Advanced Intelligence Project, and RIKEN Junior Research Associate Program.

% --------------- Reference --------------------

\bibliographystyle{abbrvnat}
\bibliography{ref}

% --------------- Appendix --------------------

\appendix

\section{Full Target Case for Lasso in \citet{liu2018more}} \label{appendix:full_target}
In the full target case, as discussed in \cite{liu2018more}, the data are used to select the interesting features, but they are \emph{not} used to summarize the relation between the response and the selected features.
Therefore, \emph{all} features can be used to define the direction of interest. 
\[
	\bm \eta_j = X(X^\top X)^{-1} \bm e_j,
\]
where $\bm e_j \in \RR^p$ is a zero vector with 1 at the $j^{\rm th}$ coordinate.
The conditional inference is defined as
\begin{align} \label{eq:full_model_conditional_inference}
	\bm \eta_j^\top \bm {\bm Y} \mid \left \{ j \in \cA(\bm Y) , \bm q(\bm Y) = \bm q({\bm y}^{\rm obs}) \right \}.
\end{align}
In \cite{liu2018more}, the authors proposed a solution for conducting conditional inference for a specific case when $p < n$, and there was no solution for the case when $p > n$.
This problem can be solved with the proposed PP method.
First, we rewrite the conditional inference in (\ref{eq:full_model_conditional_inference}) as the problem of characterizing the sampling distribution of 
\begin{align} \label{eq:full_model_conditional_inference_parametric}
	Z \mid \{Z \in \cZ\} \text{ where } \cZ = \{z \in \RR \mid j \in \cA(\bm y(z))\}.
\end{align}
$\bm y(z)$ in (\ref{eq:full_model_conditional_inference_parametric}) is defined as in (\ref{eq:parametrized_data_space}).
Thereafter, to identify $\cZ$, only the path of the lasso solution $\hat{\bm{\beta}}(z)$ needs to be obtained, as proposed in \S3, and the intervals in which $j$ is an element of the active set corresponding to $\hat{\bm{\beta}}(z)$ simply need to be verified along the path.
Finally, after obtaining $\cZ$, we can easily compute the selective $p$-value or selective CI.

\section{Stable Partial Target Case for Lasso in \citet{liu2018more}}\label{appendix:partial_target}
In the stable partial target case, as discussed in \cite{liu2018more}, only stable features are allowed to influence the formation of the test statistic.
Stable features are those with very strong signals that we do not wish to omit.
We select a set $\cH_{\rm obs}$ of stable features.
Subsequently, for any $j \in \cH_{\rm obs}, j \in \cM_{\rm obs}$,
\[
	\bm \eta_j = X_{\cH_{\rm obs}} (X_{\cH_{\rm obs}}^\top  X_{\cH_{\rm obs}})^{-1} \bm e_j.
\]
For any $j \not \in \cH_{\rm obs}, j \in \cM_{\rm obs}$,
\[
	\bm \eta_j = X_{\cH_{\rm obs} \cup \{j\}} (X_{\cH_{\rm obs} \cup \{j\}}^\top  X_{\cH_{\rm obs} \cup \{j\}})^{-1} \bm e_j.
\]
Next, we demonstrate how to construct $\cH_{\rm obs}$ according to \cite{liu2018more}.

%==================

\paragraph{Stable target formation by setting higher value of $\lambda$ (TN-$\ell_1$).}
In this case, $\cH_{\rm obs}$ is the lasso active set, but with a higher value of $\lambda$ than that used to select $\cM_{\rm obs}$.
We denote $\cH_{\rm obs} =  \cH({\bm y}^{\rm obs})$, and subsequently, the conditional inference is defined as
\begin{align} \label{eq:stable_partial_model_l1_conditional_inference}
	\bm \eta_j^\top \bm {\bm Y} \mid \left \{ j \in \cA(\bm Y) ,  \cH (\bm Y) = \cH({\bm y}^{\rm obs}), \bm q(\bm Y) = \bm q({\bm y}^{\rm obs}) \right \}.
\end{align}
The main drawback of the method in \cite{liu2018more} is that all $2^{|\cH_{\rm obs}|}$ sign vectors must be considered, which requires substantial computation time when $|\cH_{\rm obs}|$ is large.
This limitation can easily be overcome using our piecewise linear homotopy computation.
First, we rewrite the conditional inference in (\ref{eq:stable_partial_model_l1_conditional_inference}) as the problem of characterizing the sampling distribution of 
\begin{align*} 
	Z \mid \{Z \in \cZ\} \text{ where } \cZ = \{z \in \RR \mid j \in \cA(\bm y(z)), \cH(\bm y(z)) = \cH({\bm y}^{\rm obs})\}.
\end{align*}
Thereafter, we can easily identify $\cZ = \cZ_1 \cap \cZ_2$, where 
$\cZ_1 = \{z \in \RR \mid j \in \cA(\bm y(z))\}$  
and 
$\cZ_2 = \{z \in \RR \mid \cH(\bm y(z)) = \cH({\bm y}^{\rm obs})\}$, 
which we can simply obtain using the method proposed in \S3 of the main paper.

%==================

\paragraph{Stable target formation by setting cutoff value $c$ (TN-Custom).}
In this case, we determine $\cH_{\rm obs}$ by setting a cutoff value $c$ to select $\beta_j$, such that $|\beta_j| \geq c$  
\footnote{Our formulation is slightly different from but more general than that in \cite{liu2018more}.}. 
The set $\cH_{\rm obs}$ is defined as 
\begin{align*}
	\cH_{\rm obs} = \left \{ j \in \cM_{\rm obs}, |\beta_j|  \geq c \right\},
\end{align*}
where $ \beta_j =  \bm e^\top_j (X_{\cM_{\rm obs}}^\top  X_{\cM_{\rm obs}})^{-1} X_{\cM_{\rm obs}}^\top  \bm y^{\rm obs} $.
We denote $\cH_{\rm obs} = \cH(\cM_{\rm obs}) \subset \cM_{\rm obs}$, and subsequently, the conditional inference is formulated as 
\begin{align}\label{eq:stable_partial_model_custom_conditional_inference}
	\bm \eta_j^\top \bm Y \mid \left \{\cH (\cA(\bm Y)) = \cH (\cM_{\rm obs}),  \cA(\bm Y) = \cM_{\rm obs} \right \}.
\end{align}
The main drawback of the method in \cite{liu2018more} is that it still requires conditioning on $\{\cA(\bm Y) = \cM_{\rm obs}\}$, which is computationally intractable when $|\cM_{\rm obs}|$ is large, because the enumeration of $2^{|\cM_{\rm obs}|}$ sign vectors is required.
This limitation can easily be overcome using our proposed method.

\clearpage

\end{document}